\documentclass{article} 
\usepackage{iclr2023_conference,times}

\usepackage{hyperref}       
\usepackage{url}            
\usepackage{booktabs}       
\usepackage{amsfonts}       
\usepackage{nicefrac}       
\usepackage{microtype}      

\usepackage{graphicx}
\usepackage{subcaption}
\usepackage{booktabs} 
\usepackage{amsmath}
\usepackage{amsthm}
\usepackage{amssymb}
\usepackage{mathtools}
\usepackage{algorithm}
\usepackage{algorithmic}
\usepackage{thmtools}
\usepackage{thm-restate}
\usepackage{multirow}

\usepackage{hyperref}
\usepackage{xcolor}

\newcommand{\bmu}{\text{\boldmath{$\mu$}}}

\newcommand{\bB}{\text{\boldmath{$B$}}}

\newcommand{\btheta}{\text{\boldmath{$\theta$}}}

\newcommand{\bz}{\boldsymbol{z}}
\newcommand{\bx}{\boldsymbol{x}}

\newcommand{\bF}{\boldsymbol{F}}

\newcommand{\bxi}{\boldsymbol{\xi}}

\newcommand{\bu}{\boldsymbol{u}}

\newcommand{\bI}{\boldsymbol{I}}

\newcommand{\bA}{\boldsymbol{A}}

\newcommand{\bbR}{\mathbb{R}}

\newcommand{\diag}{\mathsf{diag}}

\DeclareMathOperator*{\argmin}{arg\,min}

\newtheorem{assumption}{\textbf{Assumption}}\newtheorem{definition}{\textbf{Definition}}\newtheorem{lemma}{\textbf{Lemma}}\newtheorem{theorem}{\textbf{Theorem}}\newtheorem{proposition}{\textbf{Proposition}}\newtheorem{remark}{\textbf{Remark}}\newtheorem{example}{\textbf{Example}}

\newcommand{\mE}{\mathbb{E}}

\newcommand{\cE}{\mathcal{E}}

\newcommand{\cX}{\mathcal{X}}
\newcommand{\cY}{\mathcal{Y}}
\newcommand{\cZ}{\mathcal{Z}}
\newcommand{\cL}{\mathcal{L}}

\newcommand{\cS}{\mathcal{S}}

\newcommand{\cN}{\mathcal{N}}

\newcommand{\cP}{\mathcal{P}}

\newcommand{\cF}{\mathcal{F}}

\newcommand{\cO}{\mathcal{O}}

\newcommand{\ty}{\tilde{y}}

\makeatother

\usepackage{hyperref}
\usepackage{url}

\title{Breaking Correlation Shift via Conditional Invariant Regularizer}
\author{Mingyang Yi$^{1,2,3}$, Ruoyu Wang$^{1,2}$, Jiacheng Sun$^{3}$, Zhenguo Li$^{3}$, Zhi-Ming Ma$^{1,2}$\\
	$^{1}$University of Chinese Academy of Sciences\\
	\texttt{\{yimingyang17,wangruoyu17\}@mails.ucas.edu.cn} \\
	$^{2}$Academy of Mathematics and Systems Science, Chinese Academy of Sciences\\
	\texttt{mazm@amt.ac.cn}\\
	$^{3}$Huawei Noah’s Ark Lab\\
	\texttt{\{sunjiacheng1,li.zhenguo\}@huawei.com}
}
\iclrfinalcopy
\begin{document}
	\date{}
	\maketitle	
	
	\begin{abstract}
		Recently, generalization on out-of-distribution (OOD) data with correlation shift has attracted great attentions. The correlation shift is caused by the spurious attributes that correlate to the class label, as the correlation between them may vary in training and test data. For such a problem, we show that given the class label, the models that are conditionally independent of spurious attributes are OOD generalizable. Based on this, a metric Conditional Spurious Variation (CSV) which controls the OOD generalization error, is proposed to measure such conditional independence. To improve the OOD generalization, we regularize the training process with the proposed CSV. Under mild assumptions, our training objective can be formulated as a nonconvex-concave mini-max problem. An algorithm with a provable convergence rate is proposed to solve the problem. Extensive empirical results verify our algorithm's efficacy in improving OOD generalization.  
	\end{abstract}
	\section{Introduction}
	The success of standard learning algorithms rely heavily on the identically distributed assumption of training and test data. However, in real-world, such assumption is often violated due to the varying circumstances, selection bias, and other reasons \citep{meinshausen2015maximin}. Thus, learning a model that generalizes on out-of-distribution (OOD) data has attracted great attentions. The OOD data \citep{ye2021ood} can be categorized into data with \emph{diversity shift} or \emph{correlation shift}. Roughly speaking, there is a mismatch of the spectrum and a spurious correlation between training and test distributions under the two shifts, respectively. Compared with diversity shift, correlation shift is less explored \citep{ye2021ood}, while the misleading spurious correlation works for training data may deteriorate model's performance on test data \citep{beery2018recognition}. 
	\par
	The correlation shift says, for the spurious attributes in data, there exists variation of (spurious) correlation between class label and such spurious attributes from training to test data (Figure \ref{fig:waterbirds}). Based on a theoretical characterization of it, we show that given the class label, the model which is conditionally independent of spurious attributes has stable performance across training and OOD test data. Then, a metric \emph{Conditional Spurious Variation} (CSV, Definition \ref{def:csv}) is proposed to measure such conditional independence. Notably, in contrast to the existing metrics related to OOD generalization \citep{hu2020domain,mahajan2021domain}, our CSV can control the OOD generalization error. 
	\par
	\begin{figure}[t]\centering
		\vspace{-0.5in}
		\includegraphics[width=1\textwidth]{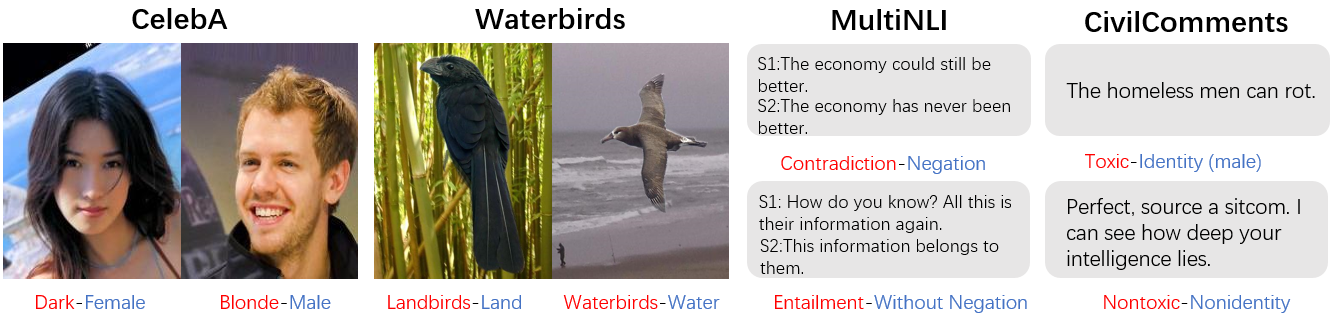}
		\vspace{-0.22in}
		\caption{Examples of \texttt{CelebA} \citep{liu2015deep}, \texttt{Waterbirds} \citep{sagawa2019distributionally}, \texttt{MultiNLI} \citep{williams2018broad}, and \texttt{CivilComments} \citep{borkan2019nuanced} involved in this paper. The class labels and spurious attributes are respectively colored with red and blue. Their correlation may vary from training set to test set. More details are shown in Section \ref{sec:experiments}.}
		\label{fig:waterbirds}
		\vspace{-0.3in}
	\end{figure} 
	\par
	To improve OOD generalization, we regularize the training process with estimated CSV. With observable spurious attributes, we propose an estimator to CSV. However, such observable condition may be violated. In this case, we propose another estimator, which approximates a sharp upper bound of CSV. We regularize the training process with one of them, depending on whether the spurious attributes are observable. Our method improves the observable condition in \citep{sagawa2019distributionally}.
	\par
	Under mild smoothness assumptions, the regularized training objective can be formulated as a specific non-convex concave minimax problem. A stochastic gradient descent based algorithm with a provable convergence rate of order $\cO(T^{-2/5})$ is proposed to solve it, where $T$ is the number of iterations. 
	\par
	Finally, extensive experiments are conducted to empirically verify the effectiveness of our methods on the OOD data with spurious correlation. Concretely, we conduct experiments on benchmark classification datasets \texttt{CelebA} \citep{liu2015deep}, \texttt{Waterbirds} \citep{sagawa2019distributionally}, \texttt{MultiNLI} \citep{williams2018broad}, and \texttt{CivilComments} \citep{borkan2019nuanced}. Empirical results show that our algorithm consistently improves the model's generalization on OOD data with correlation shifts. 
	
	\section{Related Works and Preliminaries}
	\subsection{Related Works} 
	\paragraph{OOD Generalization.} The appearance of OOD data \citep{hendrycks2018benchmarking} has been widely observed in machine learning community \citep{recht2019imagenet,schneider2020improving,salman2020unadversarial,tu2020empirical,lohn2020estimating}. To tackle this, researchers have proposed various algorithms from different perspectives, e.g., distributional robust optimization \citep{sinha2018certifying,volpi2018generalizing,sagawa2019distributionally,yi2021improved,levy2020large} or causal inference \citep{arjovsky2019invariant,he2021towards,liu2021heterogeneous,mahajan2021domain,wang2022out,ye2021towards}. \cite{ye2021ood} points out that the OOD data can be categorized into data with \emph{diversity shift} (e.g., \texttt{PACS} \citep{li2018deep}) and \emph{correlation shift} (e.g., \texttt{Waterbirds} \citep{sagawa2019distributionally}), and we focus on the latter in this paper, as we have clarified that it deteriorates the performance of the model on OOD test data \citep{geirhos2018imagenet,beery2018recognition,xie2020in,wald2021calibration}.  
	
	\paragraph{Domain Generalization.} To goal of domain generalization is extrapolating model to test data from unseen domains to capture OOD generalization. The problem we explored can be treated by domain generalization methods as data with different spurious attributes can be regarded as from different domains. The core idea in domain generalization is to learn a domain-invariant model. To this end, \cite{arjovsky2019invariant,hu2020domain,li2018deep,mahajan2021domain,heinze2021conditional,krueger2021out,wald2021calibration,seo2022information} propose plenty of invariant metrics as training regularizer. However, unlike our CSV, none of these metrics controls the OOD generalization error. Moreover, none of these methods capture the invariance corresponds to the correlation shift we discussed (see Section \ref{sec:guaranteed OOD generalization error}). This motivates us to reconsider the effectiveness of these methods. Finally, in contrast to ours, these methods require observable domain labels, and it is usually impractical. The techniques in \citep{liu2021heterogeneous,arpit2019predicting,sohoni2020no,creager2021environment} are also applicable without domain information, but they are built on strong assumptions (mixture Gaussian data \citep{liu2021heterogeneous} and linear model \citep{arpit2019predicting}) or require a high-quality spurious attribute classifier \citep{sohoni2020no,creager2021environment}.
	
	\paragraph{Distributional Robustness.} The distributional robustness \citep{ben2013robust} based methods minimize the worst-case loss over different groups of data \citep{sagawa2019distributionally,liu2021just,zhou2022model}. The groups are decided via certain rules, e.g., data with same spurious attributes \citep{sagawa2019distributionally} or annotated via validation sets with observable spurious attributes \citep{liu2021just,zhou2022model}. However, \cite{sagawa2019distributionally} finds that directly minimizing the worst-group loss results in unstable training processes. In contrast, our method has stable training process as it balances the objectives of accuracy and robustness over spurious attributes (see Section \ref{sec:regularize training with cv}).
	
	
	
	\subsection{Problem Setup}\label{sec:notions}
	We collect the notations in this paper. $\|\cdot\|$ is the $\ell_{2}$-norm of vectors. $\cO(\cdot)$ is the order of a number. The sample $(X, Y)\in\cX\times\cY$, where $X$ and $Y$ are respectively input data and its label. The integer set from $1$ to $K$ is $[K]$. The cardinal of a set $A$ is $|A|$. The loss function is $\cL(\cdot, \cdot):\bbR^{K}\times\cY\rightarrow\bbR^{+}$ with $0 \leq \cL(\cdot, \cdot) \leq M$ for positive $K, M$. For any distribution $P$, let $R_{pop}(P, f) = \mE_{P}[\cL(f(X), Y)]$ and $R_{emp}(P, f) = n^{-1}\sum_{i=1}^{n}\cL(f(\bx_{i}), y_{i})$ respectively be the population risk under $P$ and its empirical counterpart. Here $\{(\bx_{i}, y_{i})\}_{i=1}^{n}$ are $n$ i.i.d. samples from distribution $P$, and $f(\cdot): \cX\rightarrow\bbR^{K}$ is the model potentially with parameter space $\Theta\subset\bbR^{d}$ ($f(\cdot)$ becomes $f_{\btheta}(\cdot)$). For random variables $V_{1}, V_{2}$ with joint distribution $P_{V_{1}, V_{2}}$, $P_{V_{1}}$ and $P_{V_{2}\mid V_{1}}$ are the marginal distribution of $V_{1}$ and the conditional distribution of $V_{2}$ under $V_{1}$. $P_{V_{1}}(v_{1})$ and $P_{V_{2}\mid V_{1}}(v_{1}, v_{2})$ are their probability measures.
	\par
	Suppose the training and OOD test data are respectively from distributions $P_{X, Y, Z}$ and $Q_{X, Y, Z}$. We may neglect the subscript if there is no obfuscation. There usually exists similarities between $P$ and $Q$ that guarantee the possibility of OOD generalization \citep{kpotufe2021marginal}.
	The similarity we explored is that there only exists a correlation shift in the OOD test data formulated as follows. For each input $X$, there exists spurious attributes $Z\in\cZ$ that are not causal to predict class label, but $Z$ is potentially related to class label $Y$. The $Z$ can be some features of $X$ e.g., gender of celebrity in Figure \ref{fig:waterbirds}. The correlation between $Z$ and $Y$ (i.e., spurious correlation) can vary from training to test data, and the one in training distribution may become a misleading signal for the model to make predictions on test data. For example, in the celebrity's face in Figure \ref{fig:waterbirds}, if most males in the training set have dark hair, the model may overfit such spurious correlation and mispredict the male with blond hair. Thus we should learn a model that is robust to correlation shift defined as follows.  
	\par
	\begin{definition}\label{def:correlation shift}
		Given training distribution $P$, the test distribution $Q\in \cP$ has correlation shift, where 
		\begin{equation}\label{eq:data structual}
			\small
			\begin{aligned}
				\cP = \{Q_{X, Y, Z}: Q_{Y} = P_{Y}, Q_{X\mid Y, Z} = P_{X\mid Y, Z}\}.
			\end{aligned}
		\end{equation}
	\end{definition}
	Our definition characterizes the distributions with correlation shift. The first equality in \eqref{eq:data structual} obviates the mismatching caused by label shift (i.e., $P_{Y}\neq Q_{Y}$) which is unrelated to spurious correlation. More discussions of it are in Appendix \ref{subsec: label shift}. The second equality in \eqref{eq:data structual} states the invariance of conditional distribution of data, given the class label and spurious attributes, which is reasonable as the unstable spurious correlations are decided by the joint distribution of $Y$ and $Z$. Finally, since 
	\begin{equation}
		\small
		\begin{aligned}
			Q_{X, Y}(\bx, y) & = Q_{X, Y}(\bx\mid y)Q_{Y}(y) = Q_{Y}(y)\int_{\cZ} Q_{X\mid Y, Z}(\bx\mid y,  z)dQ_{Z\mid Y}(z\mid y) \\
			& = P_{Y}(y)\int_{\cZ} P_{X\mid Y, Z}(\bx\mid y,  z)dQ_{Z\mid Y}(z\mid y),
		\end{aligned}
	\end{equation}
	the two constraints in \eqref{eq:data structual} together implies the correlation shift of $Q\in\cP$ is from the variety of conditional distributions $Q_{Z \mid Y}$, which is consistent with intuition. Our definition is different from the ones in \citep{mahajan2021domain,makar2022fairness}, as they rely on a causal directed acyclic graph and the existence of a sufficient statistic such that $Y$ only affects $X$ through it. 
	
	\section{Generalizing on OOD Data}\label{sec:Generalizing on OOD Data Under Spurious Correlation}
	In this section, we show that misleading spurious correlation can hurt the OOD generalization. Then we give a condition under which the model is OOD generalizable. 
	\subsection{Spurious Correlation Misleads Models}
	The common way to train a model is empirical risk minimization \citep[ERM][]{vapnik1999nature}, i.e., approximating the minimizer of $R_{pop}(P,f)$ which generalizes well on in-distribution samples via minimizing its empirical counterpart $R_{emp}(P, f)$. However, the following proposition shows that the minimizer of $R_{pop}(P,f)$ may not generalize well on the OOD data from the other distributions in $\cP$. 
	\begin{restatable}{proposition}{counterexample}
		\label{pro:counter example}
		There exists a population risk $R_{pop}(P, f)$ whose minimizer has nearly perfect performance on the data from $P$, while it fails to generalize to OOD data drawn from another $Q\in\cP$. 
	\end{restatable}
	Similar results also appear in \citep{xie2020in,krueger2021out}, while they are not obtained on the minimizer of population risk. The proof of this proposition is in Appendix \ref{app:proofs in sec generalizing} which indicates that the spurious correlation in training data can become a misleading supervision signal that deteriorates the model's performance on OOD data. Hence, it is crucial to learn a model that is independent of such spurious correlation, even if it sometimes can be helpful in the training set \citep{xie2020in}.  
	\subsection{Conditional Independence Enables OOD Generalization}
	Next, we give a sufficient condition proved in Appendix \ref{app:proofs in sec generalizing} to make the model OOD generalizable. The condition is also necessary under some specific data generating structures \citep{veitch2021counterfactual}.   
	\begin{restatable}{theorem}{thminvariantmodel}\label{thm:conditional independence}
		For model $f(\cdot)$ satisfying $f(X)\perp Z\!\mid \!Y$, the conditional distribution $Y\mid f(X)$ and population risk $\mE_{Q}[\cL(f(X), Y)]$ are invariant with $(X, Y)\sim Q_{X, Y}$ such that $Q\in \cP$.  
	\end{restatable}
	Here $f(X)\!\!\perp \! Z \! \mid \!Y$ means given $Y$, $f(X)$ is conditionally independent of $Z$. Theorem \ref{thm:conditional independence} shows our conditional independence obviates the impact of correlation shift, as the prediction error (gap between $Y$ and $f(X)$, decided by $Y\!\mid \!f(X)$) and population risk of model are invariant over test distributions $Q\in\cP$. Thus we propose to obtain a model that is conditional independent of spurious attributes.
	
	
	
	\begin{remark}
		If the spurious attributes are domain labels, the conditional independence in Theorem \ref{thm:conditional independence} becomes the ones in \citep{liu2015deep,hu2020domain,mahajan2021domain}, while they do not explore its correlation with the OOD generalization. Besides, the counterexample in \cite{mahajan2021domain} violates our conditional invariant assumption in \eqref{eq:data structual} and hence is not contrary to our theorem.
	\end{remark}
	\section{Learning OOD Generalizable Model}\label{sec:learning OOD generalizable model}
	In Theorem \ref{thm:conditional independence} we propose a independence condition to break correlation shift. In this section, a metric Conditional Spurious Variation (CSV) is proposed to quantitatively measure the independence. As our CSV can control the OOD generalization error, smaller CSV leads to improved OOD generalization. Finally, two estimators of CSV are proposed, depending on whether spurious attributes are observable. 
	
	\subsection{Guaranteed OOD Generalization Error}\label{sec:guaranteed OOD generalization error}
	Theorem \ref{thm:conditional independence} shows that the conditional independence between the model and spurious attributes guarantees the OOD generalization. We propose the following metric to measure such independence.
	\begin{definition}[Conditional Spurious Variation]\label{def:csv}
		The conditional spurious variation of model $f(\cdot)$ is
		\begin{equation}
			\small
			\begin{aligned}
				\mathrm{CSV}(f) = \mE_{Y}\left[\sup_{z_{1}, z_{2}}\left(\mE_{X}[\cL(f(X),Y)\mid Y, Z = z_{1}] - \mE_{X}[\cL(f(X), Y)\mid Y, Z = z_{2}]\right)\right].
			\end{aligned}
		\end{equation}
	\end{definition}
	\par
	As can be seen, CSV is a functional of $f(\cdot)$ which measures the intra-class conditional variation of the model over spurious attributes, given the class label $Y$. It can be computed via training distribution and is invariant across $Q\in\cP$ due to \eqref{eq:data structual}. It is worth noting that the model satisfies the conditional independence in Theorem \ref{thm:conditional independence} has zero CSV but not vice versa. \footnote{Conditional independence is a strong sufficient condition to make model OOD generalizable. However, the proof of Theorem \ref{thm:conditional independence} shows the model that is invariant with spurious correlation $Q(Z\mid Y)$ is sufficient to be OOD generalizable, while the invariance can be characterized by both zero CSV and conditional independence.} However, the following theorem proved in Appendix \ref{app:proofs in sec guaranteed} shows that $\mathrm{CSV}(f)$ is sufficient to control the OOD generalization error. 
	\begin{restatable}{theorem}{boundedgap}\label{thm:bounded gap}
		For any $Q\in \cP$, we have 
		\begin{equation}\label{eq:ood gen bound}
			\small
			\sup_{Q\in\cP}|R_{emp}(f, P) - R_{pop}(f, Q)| \leq \left|R_{emp}(f, P) - R_{pop}(f, P)\right| + \mathrm{CSV}(f)
		\end{equation}
	\end{restatable}
	The $|R_{emp}(f,\! P) \!-\! R_{pop}(f,\! P)|$ is in-distribution generalization error, which is well explored \citep{vershynin2018}. Thus, we upper bound the OOD generalization error via the in-distributional one and CSV. The OOD generalization error is also connected to many other metrics e.g., \citep{hu2020domain,mahajan2021domain,ben2007analysis,ben2010theory,muandet2013domain,ganin2016domain}, but none of them directly control the OOD generalization error. 
	Besides, these metrics are proposed to obtain the invariance over $Z$ as a condition, i.e., invariant $P_{f(X), Y \mid Z}$ or $P_{f(X)\mid Z}$, while the invariances can not handle correlation shift (Definition. \ref{def:correlation shift}). As, 1): invariant $P_{f(X), Y \mid Z}$ implies invariant $P_{Y\mid Z}$ which is incompatible with correlation shift, 2): invariant $P_{f(X)\mid Z}=\int_{\cY} P_{f(X)\mid Z, Y}(f(x)\mid z, y)dP_{Y\mid Z}(y\mid z)$ does not imply invariant $P_{f(X)\mid Y, Z}$ which guarantees OOD generalization.    
	\par
	As our bound \eqref{eq:ood gen bound} involves both CSV and in-distribution generalization error, it motivates us to explore whether the conditional independence is contradicted by the in-distribution generalization. The following information-theoretic bound proved in Appendix \ref{app:proofs in sec guaranteed} presents a positive answer. Let $I(V_{1}, V_{2})$ be the mutual information between variables $V_{1}, V_{2}$, we have the following result. 
	\begin{restatable}{theorem}{gengap}\label{thm:gen gap}
		Let model $f_{\btheta}(\cdot)$ parameterized by $\btheta\in \Theta\subset\bbR^{d}$, and is trained on $\cS=\{(\bx_{i}, y_{i})\}_{i=1}^{n}$ from distribution $P$, with the spurious attributes of $\bx_{i}$ is $z_{i}$. If the learned model $f_{\btheta_{\cS}}(\cdot)\perp \cS_{\bz}\mid \cS_{y}$\footnote{$\btheta_{\cS}$ is the learned parameters depends on training set $\cS$. $f_{\btheta_{\cS}}(\cdot)$ is a random element that takes values in a functional space (i.e., model space), details can be referred to \citep{shiryaev2016probability}.} 
		\begin{equation}\label{eq:gen bound}
			\small
			\cE_{\text{gen}}(f_{\btheta_{\cS}}, P) \leq \inf_{g}\sqrt{\frac{M^{2}}{4n}\left(I(\cS_{\bx - g(z)}, \cS_{y}; f_{\btheta_{\cS}}\mid \cS_{y}, \cS_{g(z)}) + I(\cS_{y}; f_{\btheta_{\cS}})\right)},
		\end{equation}
		where $\cE_{\text{gen}}(f_{\btheta_{\cS}}, P) = |\mE[R_{emp}(f_{\btheta_{\cS}}, P)] - R_{pop}(f_{\btheta_{\cS}}, P)|$, $g(\cdot)$ is any measurable function, $\cS_{\bx - g(z)} = \{\bx_{i} - g(z_{i})\}_{i=1}^{n}$, $\cS_{y} = \{y_{i}\}_{i=1}^{n}$. 
	\end{restatable}
	Our bound improves the classical result $\cE_{gen}(f_{\btheta_{\cS}}, P)\leq \sqrt{M^{2}I(\cS, \btheta_{\cS})/4n}$ without conditional independence \citep{steinke2020reasoning}, because taking $g(\cdot)$ as a constant function, then applying data processing inequality \citep{xu2017information} in r.h.s. of \eqref{eq:gen bound}, it becomes $\sqrt{M^{2}I(\cS, \btheta_{\cS})/4n}$. The bound indicates the model $f_{\btheta_{\cS}}(\cdot)$ that is conditional independent of spurious attribute (we aim to capture) is not in contradiction with in-distribution generalization. Thus, due to Theorem \ref{thm:bounded gap}, less conditional independence with spurious attributes of model improves the OOD generalization bound.
	
	\subsection{Estimating CSV with Observable Spurious Attributes}\label{sec:estimators of CSV}
	As smaller CSV enables improved OOD generalization, we propose to regularize the training process with it. Suppose we have $n$ i.i.d. $\{(\bx_{i}, y_{i})\}_{i=1}^{n}$ training samples with spurious attributes $\{z_{i}\}_{i=1}^{n}$ from $P$. Before our analysis, we need the following two mild assumptions in the sequel. 
	\begin{assumption}\label{ass:feature space}
		The class label and spurious attributes are discrete, i.e., the $\cY = [K_{y}]$ and $\cZ = [K_{z}]$ for some positive integers $K_{y}, K_{z}$. Besides that, the number of observations $A_{kz} = \{i: y_{i}=k, z_{i}=z\}$ from each pair of $(k, j)\in[K_{y}]\times [K_{z}]$ is $|A_{kz}| = n_{kz} >0$.
	\end{assumption}
	\begin{assumption}\label{ass:Lip}
		The model is parameterized by $\btheta\in\Theta\subset\bbR^{d}$. The loss function $\cL(f_{\btheta}(\bx), y)$ is Lipschitz continuous and smooth w.r.t. $\btheta$ with coefficient $L_{0}$ and $L_{1}$, i.e., for any $(\bx, y)\in\cX\times\cY$, and $\btheta_{1}, \btheta_{2}\in\Theta$,
		\begin{equation}
			\small
			\begin{aligned}
				\left|\cL(f_{\btheta_{1}}(\bx), y) - \cL(f_{\btheta_{2}}(\bx), y)\right| \leq L_{0}\|\btheta_{1} - \btheta_{2}\|;\\
				\left\|\nabla_{\btheta}\cL(f_{\btheta_{1}}(\bx), y) - \nabla_{\btheta}\cL(f_{\btheta_{2}}(\bx), y)\right\| \leq L_{1}\|\btheta_{1} - \btheta_{2}\|.
			\end{aligned}
		\end{equation}
	\end{assumption}
	In Assumption \ref{ass:feature space}, we require the spurious attributes space is finite. This is explained as $Z$ is a ``label" of spurious attributes, e.g., the gender label ``male" or ``female" in \texttt{CelebA} dataset (Figure \ref{fig:waterbirds}) when classifying hair color. Assumption \ref{ass:feature space} also requires the data with all the possible combinations of the label and spurious attributes are collected in the training set. This is a mild condition since we do not restrict the magnitude of $n_{kz}$. For example, to satisfy this, we can synthetic some of the missing data by generative models as in \citep{wang2022out,zhu2017unpaired}. 
	\par
	Let $\cL_{kz}(f_{\btheta}) = \mE[\cL(f_{\btheta}(X), k)\mid Y=k, Z=z]$, $\hat{\cL}_{kz}(f_{\btheta})$ $ = (1 / n_{kz})\sum_{i\in A_{kz}}\cL(f_{\btheta}(\bx_{i}), k)$, and $\hat{p}_{k} = n_{k} / n$ with $n_{k}=\sum_{z\in[K_{z}]}n_{kz}$. Then the following empirical counterpart of CSV 
	\begin{equation}\label{eq:empirical csv}
		\small
		\widehat{\rm{CSV}}(f_{\btheta}) = \sum\limits_{k = 1}^{K_{y}}\sup_{z_{1}, z_{2}\in[K_{z}]}\left(\hat{\cL}_{kz_{1}}(f_{\btheta}) - \hat{\cL}_{kz_{2}}(f_{\btheta})\right)\hat{p}_{k}
	\end{equation}
	is a natural estimator to CSV. The following theorem quantify its approximation error. 
	\begin{restatable}{theorem}{genbound}\label{thm:gen bound}
		Under Assumption \ref{ass:feature space} and \ref{ass:Lip}, if  $\inf_{k\in[K_{y}], z\in[K_{z}]} n_{kz} / n_{k} = \cO(1)$, then
		\begin{equation}\label{eq:gen bound order}
			\small
			\mathrm{CSV}(f_{\btheta}) \leq \widehat{\rm CSV}(f_{\btheta}) + \cO\left(\frac{\log{(1/\delta)}}{\sqrt{n}}\right) 
		\end{equation}
		holds with probability at least $1 - \delta$ for any $\btheta\in\Theta, \delta > 0$.
	\end{restatable}
	This theorem implies that CSV is upper bounded by $\widehat{\rm CSV}(f_{\btheta})$. As shown in the proof in Appendix \ref{app: proofs in sec:estimators of CSV}, we hide a factor related to covering number \citep{vershynin2018} of the hypothesis space in the numerator of $\cO\left(\log{(1/\delta)}/\sqrt{n}\right)$ in \eqref{eq:gen bound order}. The hidden factor is of order $\sqrt{d}$ \citep{wainwright2019}. Thus, more samples are required to estimate CSV in high-dimensional space. Besides that, if the condition $\inf_{k,z}n_{kz} / n_{k} = \cO(1)$ does not hold, the order of error is $\cO(1 / \sqrt{\min_{k, z}n_{kz}})$ (see Appendix \ref{app: proofs in sec:estimators of CSV} for details). 
	
	\subsection{Estimating CSV with Unobservable Spurious Attributes}\label{sec:estimating without spurious attributes}
	Computing the empirical CSV \eqref{eq:empirical csv} requires observable spurious attributes which may not be available in practice \citep{liu2021just}. Thus, we need to estimate CSV in the absence of spurious attributes.
	\par
	Let $P_{kz} = P_{X\mid Y = k, Z = z}$ be the conditional distributions of $X$ with $Y, Z$ given. The core difficulty of estimating CSV with unobservable spurious attributes is to estimate $\sup_{z}{\mE}_{P_{kz}}[\cL(f_{\btheta}(X), k)] - \inf_{z}{\mE}_{P_{kz}}[\cL(f_{\btheta}(X), k)]$
	via $n_{k}$ independent samples $\{(x_{i}, y_{i})\}_{i\in A_{k}}$ drawn from a mixture distribution $P_{k} = \sum_{z\in[K_{z}]}\pi_{kz}P_{kz}$. Here  $A_{k}=\bigcup_{z\in[K_{z}]}A_{kz}$ for $k\in[K_{y}]$, $P_{k} = P_{X\mid Y=k}$, $\pi_{kz} = P_{Z\mid Y}(Z = z \mid Y = k)$, and we can not specify the data in $A_{k}$ is from which of $A_{kz}$. To proceed, suppose $\pi_{kz} \geq c > 0$, which is a necessary condition for Assumption \ref{ass:feature space} to hold. We show in Appendix \ref{app:proof of sec:estimating without spurious attributes} that the quantile conditional expectation 
	\begin{equation}\label{eq:quntile estimator}
		\small
		\begin{aligned}
			\mE_{P_{k}}[\cL(f_{\btheta}(X), k)\mid \cL(f_{\btheta}(X), k)\geq q_{P_{k}}(1 - c)] - \mE_{P_{k}}[\cL(f_{\btheta}(X), k)\mid \cL(f_{\btheta}(X), k)\leq q_{P_{k}}(c)]
		\end{aligned}
	\end{equation}
	is an upper bound (which is sharp for $K \geq 3$) of $\sup_{z}{\mE}_{P_{kz}}[\cL(f_{\btheta}(X), k)] - \inf_{z}{\mE}_{P_{kz}}[\cL(f_{\btheta}(X), k)]$. Here $q_{P_{k}}(\cdot)$ is the quantile function defined as $q_{P_{k}}(s) = \inf\{p: P_{k}(\cL(f_{\btheta}(X), k) \leq p) \geq s\}$. For large $n_{k}$, we must have $\pi_{kz}\geq 1/n_{k} = c$ for each $z\in[K_{z}]$. Thus by substituting the expectation on $P_{k}$ in \eqref{eq:quntile estimator} with its empirical counterpart for $c = 1/n_{k}$, we get the the following estimator
	\begin{equation}\label{eq:estimator without spurious attribute}
		\small
		\widehat{\rm{CSV}}_{\rm U}(f_{\btheta}) \!= \!\sum_{k=1}^{K_{y}}\!\left(\max_{i\in A_{k}}\!\cL(f_{\btheta}(\bx_{i}), k) \! - \!\min_{i\in A_{k}}\!\cL(f_{\btheta}(\bx_{i}), k)\right)\hat{p}_{k}. \!
	\end{equation}
	The subscript ``$\rm U$'' means ``unobservable spurious attributes''. Besides that, the $\widehat{\rm{CSV}}_{\rm U}(f_{\btheta})$ is an upper bound to the estimator 
	$\widehat{\rm{CSV}}(f_{\btheta})$, which is another straightfoward way to obtain it.
	
	\section{Regularizing Training with CSV (RCSV)}\label{sec:regularize training with cv}
	\begin{algorithm}[t!]

		\caption{Regularize training with CSV.}
		\label{alg:sgd}
		\textbf{Input:} Training set $\{(\bx_{i}, y_{i})\}_{i=1}^{n}$, number of labels $K_{y}$ and spurious attributes $K_{z}$, 
		training steps $T$, model $f_{\btheta}(\cdot)$ parameterized by $\btheta$. Initialized $\btheta_{0}, \{\bF^{k}_{0}\}$. Positive regularization constant $\lambda$, surrogate constant $\rho$, and correction constant $\gamma$. Estimators $\hat{R}_{emp}(f_{\btheta}, P)$ to $R_{emp}(f_{\btheta}, P)$, $\hat{\bF}^{k}(\btheta)$ to $\bF^{k}(\btheta)$.
		
		\begin{algorithmic}[1]
			\FOR    {$t=0, \cdots ,T$}
			\STATE  {\textbf{Solve the maximization problem:}}
			\FOR    {$k=1, \cdots, K_{y}$}
			\STATE  {$\bF^{k}_{t + 1} = (1 - \gamma)\bF^{k}_{t} + \gamma\hat{\bF}^{k}(\btheta(t))$;}
			\STATE  {$\bu_{k}(t + 1) = \mathrm{Softmax}(\bF^{k}_{t + 1} / \rho)$.}
			\ENDFOR 
			\STATE  {\textbf{Minimization step via SGD:}}
			\STATE  {$\btheta(t + 1) = \btheta(t) -  \eta_{\btheta}\sum\limits_{k = 1}^{K_{y}}\hat{p}_{k}\nabla_{\btheta}(\hat{R}_{emp}(f_{\btheta(t)}, P) + \lambda\bu_{k}(t + 1)^{\top}\bF_{t + 1}^{k})$.}
			\ENDFOR
		\end{algorithmic}
	\end{algorithm}
	
	The previous results have claimed that the model with small CSV generalizes well on OOD data. On the other hand, Theorem \ref{thm:gen bound} and discussion in Section \ref{sec:estimating without spurious attributes} have approximated the CSV via $\widehat{\rm{CSV}}(f_{\btheta})$ and $\widehat{\rm{CSV}}_{\rm U}(f_{\btheta})$, respectively. Thus we can regularize the training process with one or the other to improve the OOD generalization, depending on whether the spurious attributes are observable. 
	\par
	It is notable that both of the regularized training objectives can be formulated as the following minimax problem for positive constants $m$ and  $\lambda$
	\begin{equation}\label{eq:minimax problem}
		\small
		\min_{\btheta\in\Theta}\sum_{k=1}^{K_{y}}\hat{p}_{k}\left(R_{emp}(f_{\btheta}, P) \!+\! \lambda \max_{\bu\in\Delta_{m}}\bu^{\top}\bF^{k}(\btheta)\right).
	\end{equation}
	Here $\Delta_{m} = \{\bu = (u_{1},\dots, u_{m})\in\bbR^{m}_{+}: \sum_{i}u_{i} = 1\}$, $\bF^{k}(\btheta)\in\bbR^{m}$, and each dimension of $\bF^{k}(\btheta)$ is Lipschitz continuous function with Lipschitz gradient. Under Assumption \ref{ass:feature space} and \ref{ass:Lip}, the training process of empirical risk minimization regularized with $\widehat{\rm{CSV}}(f_{\btheta})$ or $\widehat{\rm{CSV}}_{\rm U}(f_{\btheta})$ can be formulated as the above problem by respectively setting $\bF^{k}(\btheta)$ as the vectorization of the two following matrices. The $m$ of $\widehat{\rm{CSV}}(f_{\btheta})$ and  $\widehat{\rm{CSV}}_{\rm U}(f_{\btheta})$ are respectively $K_{z}^{2}$ and $|A_{k}|^{2}$. 
	\begin{equation}\label{eq:csv for f}
		\small
		(\hat{\cL}_{kz_{1}}(f_{\btheta}) - \hat{\cL}_{kz_{2}}(f_{\btheta}))_{z_{1}, z_{2}\in[K_{z}]}, \quad (\cL(f_{\btheta}(\bx_{i_{1}}), k) - \cL(f_{\btheta}(\bx_{i_{2}}), k))_{i_{1},i_{2} \in A_{k}}.
	\end{equation}
	\par
	Before solving \eqref{eq:minimax problem}, we clarify the difference between regularizing training with CSV and distributional robustness optimization (DRO) based methods which minimize the worst-case expected loss over data with same spurious attributes, e.g., GroupDRO \citep{sagawa2019distributionally} minimizes $\max_{k,z}\mE_{P_{kz}}[\cL(f_{\btheta}(X), k)]$. Theoretically, the OOD generalizable model has perfect in-distribution test accuracy and robustness over different spurious attributes as in \eqref{eq:ood gen bound}. Our regularized training objective split such two goals, while DRO based methods mix them into one objective. Though both objectives theoretically upper bound the loss on OOD data, we empirically observe that the two goals are in contradiction with each other (see Section \ref{sec:experiments}). We also observe that splitting the two goals (our objective) enables us easily take a balance between them, which guarantees a stable training process. In contrast, the mixed training objective can be easily dominated by one of the two goals, which results in an unstable training process.  Similar phenomena are also observed in \citep{sagawa2019distributionally}. This motivates the early stopping or large weight decay regularizer used in GroupDRO.
	\subsection{Solving the Minimax Problem}
	Let $\phi^{k}(\btheta, \bu) = R_{emp}(f_{\btheta}, P) + \lambda\bu^{\top}\bF^{k}(\btheta)$,  $\Phi^{k}(\btheta) = \max_{\bu\in\Delta_{m}}\phi^{k}(\btheta, \bu)$. Under Assumption \ref{ass:feature space} and \ref{ass:Lip}, \eqref{eq:minimax problem} is a nonconvex-concave minimax problem. As explored in \citep{lin2020gradient}, the nonconvex-strongly concave minimax problem is much easier than the nonconvex-concave one. Thus we consider the surrogate of $\phi^{k}(\btheta, \bu)$ defined as $\phi_{\rho}^{k}(\btheta, \bu) = \phi^{k}(\btheta, \bu) - \lambda\rho\sum_{j=1}^{m}\bu(j)\log{\left(m\bu(j)\right)}$ for $\lambda_{\rho} > 0$, 
	which is strongly concave w.r.t. $\bu$, and $\phi^{k}(\btheta, \bu)$ is well approximated by it for small $\rho$. Next, we consider the following nonconvex-strongly concave problem 
	\begin{equation}\label{eq:objective}
		\small
		\min_{\btheta\in\Theta}\sum_{k=1}^{K_{y}}\hat{p}_{k}\max_{\bu\in\Delta_{m}}\phi_{\rho}^{k}(\btheta, \bu) = \min_{\btheta\in\Theta}\sum_{k=1}^{K_{y}}\hat{p}_{k}\Phi_{\rho}^{k}(\btheta),
	\end{equation} 
	instead of \eqref{eq:minimax problem}, where $\Phi^{k}_{\rho}(\btheta) = \max_{\bu\in\Delta_{m}}\phi^{k}_{\rho}(\btheta, \bu)$. To solve \eqref{eq:objective}, we propose the Algorithm \ref{alg:sgd}.
	
	\par
	In Algorithm \ref{alg:sgd}, lines 3-6 solve the maximization problem in \eqref{eq:objective}, which has close-formed solution $\bu_{k}^{*}(t + 1) = \mathrm{Softmax}(\bF^{k}(\btheta(t)) / \rho)$ \citep{yi2021reweighting}, where $\mathrm{Softmax}(\cdot)$ is the softmax function \citep{epasto2020optimal}. As the estimator $\hat{\bF}^{k}(\btheta)$ may have large variance, substituting the $\bF^{k}(\btheta)$ in $\bu_{k}^{*}(t + 1)$ with it in Line 8 (the minimization step) will induce a large deviation. Thus we use the moving average correction $\bF_{t + 1}^{k}$ (Line 4) to estimate $\bF^{k}(\btheta(t))$, which guarantees our convergence result in Theorem \ref{thm:convergence rate}. The convergence rate of Algorithm \ref{alg:sgd} is evaluated via approximating first-order stationary point, which is standard in non-convex problems \citep{ghadimi2013stochastic,lin2020gradient}.
	\begin{restatable}{theorem}{convergencerate}
		\label{thm:convergence rate}
		Under Assumption \ref{ass:feature space} and \ref{ass:Lip}, if $\hat{R}_{emp}(f_{\btheta}, P)$ and $\hat{\bF}^{k}(\btheta)$ are all unbiased estimators with bounded variance, $\btheta(t)$ is updated by Algorithm \ref{alg:sgd} with $\eta_{\btheta} = \cO\left(T^{-\frac{3}{5}}\right)$ and $\gamma = T^{-\frac{2}{5}}$, then 
		\begin{equation}
			\small
			\min_{1\leq t\leq T}\mE\left[\left\|\sum\limits_{k=1}^{K_{y}}\hat{p}_{k}\nabla\Phi_{\rho}^{k}(\btheta(t))\right\|^{2}\right] \leq \cO\left(T^{-\frac{2}{5}}\right). 
		\end{equation}
		Besides that, for any $\btheta(t)$ and $\rho$, we have $|\sum_{k=1}^{K_{y}}\hat{p}_{k}(\Phi_{\rho}^{k}(\btheta(t)) - \Phi^{k}(\btheta(t)))| \leq \lambda\rho(1 / me + 2\log{m})$.
	\end{restatable}
	The theorem is proved in Appendix \ref{app: solving minimax problem}, and it says the first-order stationary point of the surrogate loss  $\Phi_{\rho}^{k}(\cdot)$ is approximated by $\btheta(t)$ in Algorithm \ref{alg:sgd}, in the order of $\cO(T^{-2/5})$ (can be improved to $\cO(T^{-1/2})$ when $\sigma^{2} = \cO(T^{-1/2})$). As the gap between $\Phi^{k}(\cdot)$ and $\Phi_{\rho}^{k}(\cdot)$ is $\cO(\rho)$, taking small $\rho$ yields small $\Phi^{k}(\btheta(T))$. The unbiased estimators in our theorem are constructed in the next section. 
	\section{Experiments}\label{sec:experiments}
	In this section, we empirically evaluate the efficacy of the proposed Algorithm \ref{alg:sgd} in terms of breaking the spurious correlation. More experiments are shown in Appendix \ref{app:experiments}. 
	
	\paragraph{Implementation.} The Algorithm \ref{alg:sgd} can be applied to the regularized training process with either $\widehat{\rm CSV}(f_{\btheta})$ (RCSV) or $\widehat{\rm CSV_{U}}(f_{\btheta})$ (RCSV$_{\rm U}$) depending on whether spurious attributes are observable. The implementations is clear after estimating $R_{emp}(f_{\btheta(t)}, P)$ and $F^{k}(\btheta)$ in \eqref{eq:csv for f}.
	\par
	For RCSV, in each step $t$, we let $\hat{R}_{emp}(f_{\btheta(t)}, P)$ be the empirical risk over a uniformly drawn batch (size $S$) of data. Then we randomly sample another batch (size $S$) of data with replacement. The probability of each data with class label $k$ and spurious attribute $z$ be sampled is $1 / (K_{y}K_{z}n_{kz})$. Then the $\hat{\cL}_{kz}(f_{\btheta(t)})$ in \eqref{eq:csv for f} is estimated as the conditional expected risk over this batch of data. 
	\par
	For RCSV$_{\rm U}$, in each step $t$, the $\hat{R}_{emp}(f_{\btheta(t)}, P)$ is as in RCSV. We also randomly sample another mini-batch (batch size $S$) of data with replacement but the probability of data with label $k$ be sampled is $1 / (K_{y}n_{k})$. We estimate $F^{k}(\btheta)$ in \eqref{eq:csv for f} via its empirical counterpart over these sampled data. 
	\par
	As can be seen, all of the resulting $\hat{R}_{emp}(f_{\btheta(t)}, P)$ and $\hat{\bF}^{k}(\btheta(t))$ are unbiased estimators with variance of order $\cO(1/S)$ as in Theorem \ref{thm:convergence rate}. Besides that, our RCSV (resp. RCSV$_{\rm U}$) can be implemented with (resp. without) observable spurious attributes. The complete implementation of RCSV and RCSV$_{\rm U}$ are shown in Appendix \ref{app:algorithms}. When estimating CSV, the data are sampled with weights that are inversely proportional to $n_{kz}$ or $n_{k}$. The sampling strategy also appears in \citep{sagawa2019distributionally,idrissi2021simple,arjovsky2019invariant} which significantly improves the OOD generalization according to our ablation study in Appendix \ref{app:experiments without reweighting}.
	\par
	\begin{table}[t]
		\vspace{-0.2in}
		\caption{Test accuracy (\%) of ResNet50 on each group of \texttt{CelebA} and \texttt{Waterbirds}.}
		\vspace{-0.1in}
		\label{tbl:celeba}
		\centering
		\scalebox{0.8}{
			{
				\begin{tabular}{l|c|*{8}{c}}
					\hline
					Dataset & Method / Group & D-F & D-M & B-F & B-M & Avg & Total & Worst & SA\\
					\hline
					\multirow{9}{*}{\texttt{CelebA}} & RCSV & 92.1 & 94.0 & 92.3 & 91.8 & \textbf{92.6} & 92.9 & \textbf{91.8} &$\surd$ \\
					& IRM  & 92.8 & 93.1 & 88.9 & 89.3 &  91.0 & 92.3 & 88.9 & $\surd$ \\
					& GroupDRO & 94.4 & 94.6 & 88.9 & 88.3 & 91.6 & \textbf{93.7} & 88.3 & $\surd$ \\
					& ERMRS$_{\rm YZ}$ & 88.8 & 94.7 & 95.8 & 85.6 & 91.2 & 91.9 & 85.6 & $\surd$ \\
					& RCSV$_{\rm U}$ & 88.3 & 97.8 & 96.1 & 76.9 & 89.8 & 93.3  &76.9  & $\times$ \\
					& Correlation	& 87.6 & 96.3 & 96.8 & 69.4 & 87.5 & 91.9 & 69.4 & $\times$ \\
					& ERMRS$_{\rm Y}$ & 91.3 & 97.5 & 91.0 & 62.2 & 85.5 & 93.3 & 62.2 & $\times$ \\
					\hline
					Dataset & Method / Group & L-L & L-W & W-L & W-W & Avg & Total & Worst & SA\\
					\hline
					\multirow{10}{*}{\texttt{Waterbirds}} & RCSV & 93.4 & 90.8 & 88.6 & 89.4 & \textbf{90.5} & \textbf{91.4} & \textbf{88.6} & $\surd$ \\
					& IRM  & 93.1 & 87.7 & 87.4 & 89.7 & 89.5 & 90.4 & 87.4 & $\surd$ \\
					& GroupDRO & 92.0 & 87.8 & 87.9 & 90.3 & 89.5 & 89.7 & 87.8 & $\surd$ \\
					& ERMRS$_{\rm YZ}$ & 93.2 & 88.5 & 87.9 & 91.3 & 90.2 & 90.6 & 87.9 & $\surd$ \\
					& RCSV$_{\rm U}$ & 98.5 & 81.2 & 81.3 & 95.5 & 89.1 & 89.5 & 81.2 & $\times$ \\
					& Correlation	& 98.8 & 74.4 & 70.9 & 93.0 & 84.3 & 85.5 & 70.9 & $\times$ \\
					& ERMRS$_{\rm Y}$ & 99.3 & 76.7 & 68.8 & 94.9 & 84.9 & 86.6 & 68.8 & $\times$ \\
					\hline
		\end{tabular}}}
	\end{table}
	\begin{table}[t]
		\vspace{-0.1in}
		\caption{Test accuracy (\%) of BERT on each group of \texttt{MultiNLI}.}
		\vspace{-0.1in}
		\label{tbl:textual}
		\centering
		\scalebox{0.8}{
			{
				\begin{tabular}{l|c|*{10}{c}}
					\hline
					Dataset & Method / Group & C-WN & C-N & E-WN & E-N & N-WN & N-N & Avg & Total & Worst & SA\\
					\hline
					\multirow{9}{*}{\texttt{MultiNLI}} & RCSV & 77.1	& 95.3	&  83.5	&  79.5	&  81.3	&  78.6	&  \textbf{82.6} &  \textbf{81.5}	&  \textbf{78.6} &$\surd$ \\
					& IRM  & 79.7 &	96.9 &	80.4 &	71.7 &	77.2 &	71.7 &	79.6 &	79.9 & 	71.7 & $\surd$ \\
					& GroupDRO & 77.4 &	93.5 &	82.5  &	79.7  &	81.5  &	77.6  &	82.0	 &  81.3  & 77.6
					& $\surd$ \\
					& ERMRS$_{\rm YZ}$ & 74.7 &	89.5  &	79.1  &	72.3  &	82.0  &	71.8  &	78.2  &	79.3  &	71.8
					& $\surd$ \\ 
					& RCSV$_{\rm U}$ & 79.3	 & 95.5  &	82.0	 & 74.3	 &  78.1 &	70.8 &	80.0 &	80.6  &	70.8
					& $\times$ \\
					& Correlation	& 76.2	& 94.7  &	76.7 &	67.9  &	75.5 &	67.9 &	76.5 &	77.0	& 67.9
					& $\times$ \\
					& ERMRS$_{\rm Y}$ & 82.7	& 95.5	& 79.4	& 72.6	& 79.7 	& 67.3	& 79.5	& 81.1	& 67.3
					& $\times$ \\
					\hline
		\end{tabular}}}
		\vspace{-0.2in}
	\end{table}
	\begin{table}[t]
		\vspace{-0.1in}
		\caption{Test accuracy (\%) of BERT on each group of \texttt{CivilComments}.}
		\vspace{-0.1in}
		\label{tbl:civil}
		\centering
		\scalebox{0.8}{
			{
				\begin{tabular}{l|c|*{8}{c}}
					\hline
					Dateset & Method / Group & N-N & N-I & T-N & T-I & Avg & Total & Worst & SA \\
					\hline
					\multirow{9}{*}{\texttt{CivilComments}} & RCSV & 93.1	& 87.7	& 82.4	& 71.7	&  \textbf{83.7} &  89.3	& \textbf{71.7}
					& $\surd$ \\
					& IRM  & 96.2	& 88.5	& 68.0	& 67.5	& 80.1	& 90.3 & 67.5 & $\surd$ \\
					& GroupDRO & 94.5	& 88.7	& 76.3	& 69.4	& 82.2	& 90.0	& 69.4
					& $\surd$ \\
					& ERMRS$_{\rm YZ}$ & 94.3	& 88.8	& 79.1	& 70.0	& 83.1	& 90.0	& 70.0 			
					& $\surd$ \\
					& RCSV$_{\rm U}$ & 96.2	& 89.5	& 72.6	& 68.7	& 81.7	& 90.9	& 68.7
					& $\times$ \\
					& Correlation	& 94.1	& 89.2	& 85.2	& 65.5	& 83.5	& 89.6	& 65.5 & $\times$ \\
					& ERMRS$_{\rm Y}$ & 98.0	& 94.4	& 61.0	& 57.2	& 77.7	& \textbf{92.3}	& 57.2 
					& $\times$ \\  
					\hline
				\end{tabular}
				\vspace{-0.3in}
		}}
	\end{table}
	
	\paragraph{Data.} We use the following benchmark datasets with correlation shift (see details in Appendix \ref{app:dataset}).
	
	\textbf{\texttt{CelebA} \citep{liu2015deep}.} An image classification task to recognize a celebrity's hair color (``dark'' or ``blond''), which is spuriously correlated with the celebrity's gender (``male'' or ``female''). The data are categorized as 4 groups via the combination of hair color and gender, e.g., ``dark-female'' (D-M). 
	\par
	\textbf{\texttt{Waterbirds} \citep{sagawa2019distributionally}.} An image classification task to recognize a bird as ``waterbird'' or ``landbird'', while the bird is spuriously correlated with background ``land'' or ``water''. The data are categorized into 4 groups, e.g., ``landbird-water'' (L-W). 
	\par
	\textbf{\texttt{MultiNLI} \citep{williams2018broad}.} Given a sentence-pair, the task aims to recognize the relationship between the two sentences, i.e., ``entailment'', ``neutrality'', ``contradiction''. The relationship is spuriously correlated with the presence of negation words. The data are categorized into 6 groups, e.g., ``entailment-without negation'' (E-W), ``contradiction-negation'' (C-N). 
	\par
	\textbf{\texttt{CivilComments} \citep{borkan2019nuanced}.} A textual classification task to check whether a sentence is toxic or not with the label spuriously correlated with whether any of 8 certain demographic identities are mentioned. The data have 4 groups, e.g., ``nontoxic-identity'' (N-I), ``toxic-nonidentity'' (T-N). 
	\paragraph{Setup.} We compare our methods RCSV and RCSV$_{\rm U}$ with four baseline methods (see Appendix \ref{app:benchmark algorithms} for details) i.e., ERM with reweighted sampling (ERMRS) \citep{idrissi2021simple}, IRM \citep{arjovsky2019invariant}, GroupDRO \citep{sagawa2019distributionally}, and Correlation \citep{arpit2019predicting}. 
	\par
	The GroupDRO and IRM use the reweighted sampling strategy as in RCSV, while the Correlation uses same one with RCSV$_{\rm U}$. As these sampling strategies improve the OOD generalization \citep{idrissi2021simple}, to make a fair comparison, we also conduct ERMRS with the two sampling strategies. The two ERMRS are respectively denoted as ERMRS$_{\rm Y}$ and ERMRS$_{\rm YZ}$. The involved 7 methods are categorized as 2 groups, i.e., conducted with observable spurious attributes (RCSV, IRM, GroupDRO, ERMRS$_{\rm YZ}$) and with unobservable spurious attributes (RCSV$_{\rm U}$, Correlation, ERMRS$_{\rm Y}$, Correlation). 
	\par
	\par
	The backbone models of image (\texttt{CelebA}, \texttt{Waterbirds}) and textual datasets (\texttt{MultiNLI}, \texttt{CivilComments}) are respectively ResNet-50 \citep{he2016deep} pre-trained on \texttt{ImageNet} \citep{deng2009imagenet} and pre-trained BERT \citep{devlin2019bert}. The hyperparameters are in Appendix \ref{app:hyperparemeters}.
	\paragraph{Main Results.} Our goal is to verify whether all these methods can break the spurious correlation in data. Thus for each dataset, we report the test accuracies on each group of it, as the groups are divided via the combination of class label and spurious attribute. We also report the averaged test accuracies over groups (``Avg''), the test accuracies on the whole test set (``Total'', which is in-distribution test accuracy expected for \texttt{Waterbirds}), and the worst test accuracies over groups (``Worst''). The results are in Table \ref{tbl:celeba}, \ref{tbl:textual}, \ref{tbl:civil}. The ``$\surd$'' and ``$\times$'' for SA (spurious attributes) respectively mean whether the method requires observable spurious attributes. The observations from these tables are as follows.
	\par
	To check the OOD generalization, a direct way is comparing the column of ``Avg'' and ``Worst'' to summarize the results in each group. As can be seen, in terms of the two criteria, the proposed RCSV (resp. RCSV$_{\rm U}$) consistently achieves better performances, compared to baseline methods with observable (resp. unobservable) spurious attributes. This verifies our methods can improve the OOD generalization. On the other hand, leveraging the observable spurious attributes benefits the OOD generalization since the methods with them consistently exhibits better performances than the ones without them. For example, the discussion in Section \ref{sec:estimating without spurious attributes} shows that the estimator of CSV in RCSV with observable spurious attributes is more accurate than the one in RCSV$_{\rm U}$. 
	\par
	There is a trade-off between the robustness of the model over spurious attributes and the test accuracies on the groups with the same class label, especially for \texttt{CelebA} and \texttt{Waterbirds}, see ``D-F'' v.s. ``D-M'' in \texttt{CelebA} for example.
	The phenomenon illustrates that some spurious correlations are captured for all methods. However, compared to the other methods, our methods have better averaged test accuracies and a smaller gap between the test accuracies over groups with the same spurious attributes. The robustness and test accuracies here respectively correspond to the goals of ``robustness'' and ``in-distribution test accuracy'' in Section \ref{sec:regularize training with cv}, the improvements support our discussion in Section \ref{sec:regularize training with cv} that splitting the goals of accuracy and robustness enables us easily take a balance between them.
	\section{Conclusion}
	In this paper, we explore the OOD generalization for data with correlation shift. After a formal characterization, we give a sufficient condition to make the model OOD generalizable. The condition is the conditional independence of the model, given the class label. Conditional Spurious Variation, which controls the OOD generalization error, is proposed to measure such independence. Based on this metric, we propose an algorithm with a provable convergence rate to regularize the training process with two estimators of CSV (i.e., RCSV and RCSV$_{\rm U}$), depending on whether the spurious attributes are observable. Finally, the experiments conducted on the datasets \texttt{CelebA}, \texttt{Waterbirds}, \texttt{MultiNLI}, \texttt{CivilComments} verify the efficacy of our methods. 

	\newpage
	\bibliography{reference}

\begin{thebibliography}{70}
\providecommand{\natexlab}[1]{#1}
\providecommand{\url}[1]{\texttt{#1}}
\expandafter\ifx\csname urlstyle\endcsname\relax
  \providecommand{\doi}[1]{doi: #1}\else
  \providecommand{\doi}{doi: \begingroup \urlstyle{rm}\Url}\fi

\bibitem[Arjovsky et~al.(2019)Arjovsky, Bottou, Gulrajani, and
  Lopez-Paz]{arjovsky2019invariant}
Martin Arjovsky, L{\'e}on Bottou, Ishaan Gulrajani, and David Lopez-Paz.
\newblock Invariant risk minimization.
\newblock Preprint arXiv:1907.02893, 2019.

\bibitem[Beery et~al.(2018)Beery, Van~Horn, and Perona]{beery2018recognition}
Sara Beery, Grant Van~Horn, and Pietro Perona.
\newblock Recognition in terra incognita.
\newblock In \emph{European Conference on Computer Vision}, 2018.

\bibitem[Ben-David et~al.(2007)Ben-David, Blitzer, Crammer, Pereira,
  et~al.]{ben2007analysis}
Shai Ben-David, John Blitzer, Koby Crammer, Fernando Pereira, et~al.
\newblock Analysis of representations for domain adaptation.
\newblock \emph{Advances in Neural Information Processing Systems}, 2007.

\bibitem[Ben-David et~al.(2010)Ben-David, Blitzer, Crammer, Kulesza, Pereira,
  and Vaughan]{ben2010theory}
Shai Ben-David, John Blitzer, Koby Crammer, Alex Kulesza, Fernando Pereira, and
  Jennifer~Wortman Vaughan.
\newblock A theory of learning from different domains.
\newblock \emph{Machine learning}, 79\penalty0 (1):\penalty0 151--175, 2010.

\bibitem[Ben-Tal et~al.(2013)Ben-Tal, Den~Hertog, De~Waegenaere, Melenberg, and
  Rennen]{ben2013robust}
Aharon Ben-Tal, Dick Den~Hertog, Anja De~Waegenaere, Bertrand Melenberg, and
  Gijs Rennen.
\newblock Robust solutions of optimization problems affected by uncertain
  probabilities.
\newblock \emph{Management Science}, 59\penalty0 (2):\penalty0 341--357, 2013.

\bibitem[Bernhard \& Rapaport(1995)Bernhard and Rapaport]{bernhard1995theorem}
Pierre Bernhard and Alain Rapaport.
\newblock On a theorem of danskin with an application to a theorem of von
  neumann-sion.
\newblock \emph{Nonlinear Analysis: Theory, Methods \& Applications},
  24\penalty0 (8):\penalty0 1163--1181, 1995.

\bibitem[Borkan et~al.(2019)Borkan, Dixon, Sorensen, Thain, and
  Vasserman]{borkan2019nuanced}
Daniel Borkan, Lucas Dixon, Jeffrey Sorensen, Nithum Thain, and Lucy Vasserman.
\newblock Nuanced metrics for measuring unintended bias with real data for text
  classification.
\newblock In \emph{World Wide Web conference}, 2019.

\bibitem[Creager et~al.(2021)Creager, Jacobsen, and
  Zemel]{creager2021environment}
Elliot Creager, J{\"o}rn-Henrik Jacobsen, and Richard Zemel.
\newblock Environment inference for invariant learning.
\newblock In \emph{International Conference on Machine Learning}, 2021.

\bibitem[Deng et~al.(2009)Deng, Dong, Socher, Li, Li, and
  Fei-Fei]{deng2009imagenet}
Jia Deng, Wei Dong, Richard Socher, Li-Jia Li, Kai Li, and Li~Fei-Fei.
\newblock Imagenet: A large-scale hierarchical image database.
\newblock In \emph{2009 IEEE Conference on Computer Vision and Pattern
  Recognition}, 2009.

\bibitem[Devansh~Arpit(2019)]{arpit2019predicting}
Richard~Socher Devansh~Arpit, Caiming~Xiong.
\newblock Predicting with high correlation features.
\newblock Preprint arXiv:1910.00164, 2019.

\bibitem[Devlin et~al.(2019)Devlin, Chang, Lee, and Toutanova]{devlin2019bert}
Jacob Devlin, Ming-Wei Chang, Kenton Lee, and Kristina Toutanova.
\newblock Bert: Pre-training of deep bidirectional transformers for language
  understanding.
\newblock In \emph{Conference of the North American Chapter of the Association
  for Computational Linguistics}, 2019.

\bibitem[Epasto et~al.(2020)Epasto, Mahdian, Mirrokni, and
  Zampetakis]{epasto2020optimal}
Alessandro Epasto, Mohammad Mahdian, Vahab Mirrokni, and Emmanouil Zampetakis.
\newblock Optimal approximation-smoothness tradeoffs for soft-max functions.
\newblock \emph{Advances in Neural Information Processing Systems}, 2020.

\bibitem[Ganin et~al.(2016)Ganin, Ustinova, Ajakan, Germain, Larochelle,
  Laviolette, Marchand, and Lempitsky]{ganin2016domain}
Yaroslav Ganin, Evgeniya Ustinova, Hana Ajakan, Pascal Germain, Hugo
  Larochelle, Fran{\c{c}}ois Laviolette, Mario Marchand, and Victor Lempitsky.
\newblock Domain-adversarial training of neural networks.
\newblock \emph{Journal of Machine Learning Research}, 17\penalty0
  (1):\penalty0 2096--2030, 2016.

\bibitem[Geirhos et~al.(2018)Geirhos, Rubisch, Michaelis, Bethge, Wichmann, and
  Brendel]{geirhos2018imagenet}
Robert Geirhos, Patricia Rubisch, Claudio Michaelis, Matthias Bethge, Felix~A
  Wichmann, and Wieland Brendel.
\newblock Imagenet-trained cnns are biased towards texture; increasing shape
  bias improves accuracy and robustness.
\newblock In \emph{International Conference on Learning Representations}, 2018.

\bibitem[Ghadimi \& Lan(2013)Ghadimi and Lan]{ghadimi2013stochastic}
Saeed Ghadimi and Guanghui Lan.
\newblock Stochastic first-and zeroth-order methods for nonconvex stochastic
  programming.
\newblock \emph{SIAM Journal on Optimization}, 23\penalty0 (4):\penalty0
  2341--2368, 2013.

\bibitem[Gulrajani \& Lopez-Paz(2020)Gulrajani and
  Lopez-Paz]{gulrajani2020search}
Ishaan Gulrajani and David Lopez-Paz.
\newblock In search of lost domain generalization.
\newblock In \emph{International Conference on Learning Representations}, 2020.

\bibitem[Gururangan et~al.(2018)Gururangan, Swayamdipta, Levy, Schwartz,
  Bowman, and Smith]{gururangan2018annotation}
Suchin Gururangan, Swabha Swayamdipta, Omer Levy, Roy Schwartz, Samuel Bowman,
  and Noah~A Smith.
\newblock Annotation artifacts in natural language inference data.
\newblock In \emph{North American Chapter of the Association for Computational
  Linguistics}, 2018.

\bibitem[He et~al.(2016)He, Zhang, Ren, and Sun]{he2016deep}
Kaiming He, Xiangyu Zhang, Shaoqing Ren, and Jian Sun.
\newblock Deep residual learning for image recognition.
\newblock In \emph{Conference on Computer Vision and Pattern Recognition},
  2016.

\bibitem[He et~al.(2021)He, Shen, and Cui]{he2021towards}
Yue He, Zheyan Shen, and Peng Cui.
\newblock Towards non-iid image classification: A dataset and baselines.
\newblock \emph{Pattern Recognition}, 110:\penalty0 107383, 2021.

\bibitem[Heinze-Deml \& Meinshausen(2021)Heinze-Deml and
  Meinshausen]{heinze2021conditional}
Christina Heinze-Deml and Nicolai Meinshausen.
\newblock Conditional variance penalties and domain shift robustness.
\newblock \emph{Machine Learning}, 110\penalty0 (2):\penalty0 303--348, 2021.

\bibitem[Hendrycks \& Dietterich(2018)Hendrycks and
  Dietterich]{hendrycks2018benchmarking}
Dan Hendrycks and Thomas Dietterich.
\newblock Benchmarking neural network robustness to common corruptions and
  perturbations.
\newblock In \emph{International Conference on Learning Representations}, 2018.

\bibitem[Hu et~al.(2020)Hu, Zhang, Chen, and Chan]{hu2020domain}
Shoubo Hu, Kun Zhang, Zhitang Chen, and Laiwan Chan.
\newblock Domain generalization via multidomain discriminant analysis.
\newblock In \emph{Uncertainty in Artificial Intelligence}, 2020.

\bibitem[Idrissi et~al.(2021)Idrissi, Arjovsky, Pezeshki, and
  Lopez-Paz]{idrissi2021simple}
Badr~Youbi Idrissi, Martin Arjovsky, Mohammad Pezeshki, and David Lopez-Paz.
\newblock Simple data balancing achieves competitive worst-group-accuracy.
\newblock Preprint, 2021.

\bibitem[Kpotufe \& Martinet(2021)Kpotufe and Martinet]{kpotufe2021marginal}
Samory Kpotufe and Guillaume Martinet.
\newblock Marginal singularity and the benefits of labels in covariate-shift.
\newblock \emph{The Annals of Statistics}, 49\penalty0 (6):\penalty0
  3299--3323, 2021.

\bibitem[Krueger et~al.(2021)Krueger, Caballero, Jacobsen, Zhang, Binas, Zhang,
  Le~Priol, and Courville]{krueger2021out}
David Krueger, Ethan Caballero, Joern-Henrik Jacobsen, Amy Zhang, Jonathan
  Binas, Dinghuai Zhang, Remi Le~Priol, and Aaron Courville.
\newblock Out-of-distribution generalization via risk extrapolation (rex).
\newblock In \emph{International Conference on Machine Learning}, 2021.

\bibitem[LeCun et~al.(1998)LeCun, Bottou, Bengio, and
  Haffner]{lecun1998gradient}
Yann LeCun, L{\'e}on Bottou, Yoshua Bengio, and Patrick Haffner.
\newblock Gradient-based learning applied to document recognition.
\newblock \emph{Proceedings of the IEEE}, 86\penalty0 (11):\penalty0
  2278--2324, 1998.

\bibitem[Levy et~al.(2020)Levy, Carmon, Duchi, and Sidford]{levy2020large}
Daniel Levy, Yair Carmon, John~C Duchi, and Aaron Sidford.
\newblock Large-scale methods for distributionally robust optimization.
\newblock \emph{Advances in Neural Information Processing Systems}, 2020.

\bibitem[Li et~al.(2017)Li, Yang, Song, and Hospedales]{li2017deeper}
Da~Li, Yongxin Yang, Yi-Zhe Song, and Timothy~M Hospedales.
\newblock Deeper, broader and artier domain generalization.
\newblock In \emph{IEEE International Conference on Computer Vision}, 2017.

\bibitem[Li et~al.(2018)Li, Tian, Gong, Liu, Liu, Zhang, and Tao]{li2018deep}
Ya~Li, Xinmei Tian, Mingming Gong, Yajing Liu, Tongliang Liu, Kun Zhang, and
  Dacheng Tao.
\newblock Deep domain generalization via conditional invariant adversarial
  networks.
\newblock In \emph{European Conference on Computer Vision}, 2018.

\bibitem[Lin et~al.(2020)Lin, Jin, and Jordan]{lin2020gradient}
Tianyi Lin, Chi Jin, and Michael Jordan.
\newblock On gradient descent ascent for nonconvex-concave minimax problems.
\newblock In \emph{International Conference on Machine Learning}, 2020.

\bibitem[Liu et~al.(2021{\natexlab{a}})Liu, Haghgoo, Chen, Raghunathan, Koh,
  Sagawa, Liang, and Finn]{liu2021just}
Evan~Z Liu, Behzad Haghgoo, Annie~S Chen, Aditi Raghunathan, Pang~Wei Koh,
  Shiori Sagawa, Percy Liang, and Chelsea Finn.
\newblock Just train twice: Improving group robustness without training group
  information.
\newblock In \emph{International Conference on Machine Learning},
  2021{\natexlab{a}}.

\bibitem[Liu et~al.(2021{\natexlab{b}})Liu, Hu, Cui, Li, and
  Shen]{liu2021heterogeneous}
Jiashuo Liu, Zheyuan Hu, Peng Cui, Bo~Li, and Zheyan Shen.
\newblock Heterogeneous risk minimization.
\newblock In \emph{International Conference on Machine Learning},
  2021{\natexlab{b}}.

\bibitem[Liu et~al.(2015)Liu, Luo, Wang, and Tang]{liu2015deep}
Ziwei Liu, Ping Luo, Xiaogang Wang, and Xiaoou Tang.
\newblock Deep learning face attributes in the wild.
\newblock In \emph{IEEE International Conference on Computer Vision}, 2015.

\bibitem[Lohn(2020)]{lohn2020estimating}
Andrew~J Lohn.
\newblock Estimating the brittleness of ai: safety integrity levels and the
  need for testing out-of-distribution performance.
\newblock Preprint arXiv:2009.00802, 2020.

\bibitem[Long et~al.(2018)Long, Cao, Wang, and Jordan]{long2018conditional}
Mingsheng Long, Zhangjie Cao, Jianmin Wang, and Michael~I Jordan.
\newblock Conditional adversarial domain adaptation.
\newblock \emph{Advances in Neural Information Processing Systems}, 2018.

\bibitem[Loshchilov \& Hutter(2018)Loshchilov and
  Hutter]{loshchilov2018decoupled}
Ilya Loshchilov and Frank Hutter.
\newblock Decoupled weight decay regularization.
\newblock In \emph{International Conference on Learning Representations}, 2018.

\bibitem[Mahajan et~al.(2021)Mahajan, Tople, and Sharma]{mahajan2021domain}
Divyat Mahajan, Shruti Tople, and Amit Sharma.
\newblock Domain generalization using causal matching.
\newblock In \emph{International Conference on Machine Learning}, 2021.

\bibitem[Makar \& D'Amour(2022)Makar and D'Amour]{makar2022fairness}
Maggie Makar and Alexander D'Amour.
\newblock Fairness and robustness in anti-causal prediction.
\newblock Preprint, 2022.

\bibitem[Meinshausen \& B{\"u}hlmann(2015)Meinshausen and
  B{\"u}hlmann]{meinshausen2015maximin}
Nicolai Meinshausen and Peter B{\"u}hlmann.
\newblock Maximin effects in inhomogeneous large-scale data.
\newblock \emph{The Annals of Statistics}, 43\penalty0 (4):\penalty0
  1801--1830, 2015.

\bibitem[Muandet et~al.(2013)Muandet, Balduzzi, and
  Sch{\"o}lkopf]{muandet2013domain}
Krikamol Muandet, David Balduzzi, and Bernhard Sch{\"o}lkopf.
\newblock Domain generalization via invariant feature representation.
\newblock In \emph{International Conference on Machine Learning}, 2013.

\bibitem[Recht et~al.(2019)Recht, Roelofs, Schmidt, and
  Shankar]{recht2019imagenet}
Benjamin Recht, Rebecca Roelofs, Ludwig Schmidt, and Vaishaal Shankar.
\newblock Do imagenet classifiers generalize to imagenet?
\newblock In \emph{International Conference on Machine Learning}, 2019.

\bibitem[Sagawa et~al.(2019)Sagawa, Koh, Hashimoto, and
  Liang]{sagawa2019distributionally}
Shiori Sagawa, Pang~Wei Koh, Tatsunori~B Hashimoto, and Percy Liang.
\newblock Distributionally robust neural networks for group shifts: On the
  importance of regularization for worst-case generalization.
\newblock 2019.

\bibitem[Salman et~al.(2020)Salman, Ilyas, Engstrom, Vemprala, Madry, and
  Kapoor]{salman2020unadversarial}
Hadi Salman, Andrew Ilyas, Logan Engstrom, Sai Vemprala, Aleksander Madry, and
  Ashish Kapoor.
\newblock Unadversarial examples: designing objects for robust vision.
\newblock Preprint arXiv:2012.12235, 2020.

\bibitem[Schneider et~al.(2020)Schneider, Rusak, Eck, Bringmann, Brendel, and
  Bethge]{schneider2020improving}
Steffen Schneider, Evgenia Rusak, Luisa Eck, Oliver Bringmann, Wieland Brendel,
  and Matthias Bethge.
\newblock Improving robustness against common corruptions by covariate shift
  adaptation.
\newblock In \emph{Advances in Neural Information Processing Systems}, 2020.

\bibitem[Seo et~al.(2022)Seo, Lee, and Han]{seo2022information}
Seonguk Seo, Joon-Young Lee, and Bohyung Han.
\newblock Information-theoretic bias reduction via causal view of spurious
  correlation.
\newblock Preprint arXiv:2201.03121, 2022.

\bibitem[Shiryaev(2016)]{shiryaev2016probability}
Albert~N Shiryaev.
\newblock \emph{Probability-1}, volume~95.
\newblock Springer, 2016.

\bibitem[Sinha et~al.(2018)Sinha, Namkoong, and Duchi]{sinha2018certifying}
Aman Sinha, Hongseok Namkoong, and John Duchi.
\newblock Certifying some distributional robustness with principled adversarial
  training.
\newblock In \emph{International Conference on Learning Representations}, 2018.

\bibitem[Sohoni et~al.(2020)Sohoni, Dunnmon, Angus, Gu, and
  R{\'e}]{sohoni2020no}
Nimit Sohoni, Jared Dunnmon, Geoffrey Angus, Albert Gu, and Christopher R{\'e}.
\newblock No subclass left behind: Fine-grained robustness in coarse-grained
  classification problems.
\newblock \emph{Advances in Neural Information Processing Systems}, 2020.

\bibitem[Soudry et~al.(2018)Soudry, Hoffer, Nacson, Gunasekar, and
  Srebro]{soudry2018implicit}
Daniel Soudry, Elad Hoffer, Mor~Shpigel Nacson, Suriya Gunasekar, and Nathan
  Srebro.
\newblock The implicit bias of gradient descent on separable data.
\newblock \emph{The Journal of Machine Learning Research}, 19\penalty0
  (1):\penalty0 2822--2878, 2018.

\bibitem[Steinke \& Zakynthinou(2020)Steinke and
  Zakynthinou]{steinke2020reasoning}
Thomas Steinke and Lydia Zakynthinou.
\newblock Reasoning about generalization via conditional mutual information.
\newblock In \emph{Conference on Learning Theory}, 2020.

\bibitem[Tu et~al.(2020)Tu, Lalwani, Gella, and He]{tu2020empirical}
Lifu Tu, Garima Lalwani, Spandana Gella, and He~He.
\newblock An empirical study on robustness to spurious correlations using
  pre-trained language models.
\newblock \emph{Transactions of the Association for Computational Linguistics},
  8:\penalty0 621--633, 2020.

\bibitem[van~der Vaart \& Wellner(2000)van~der Vaart and Wellner]{van2000weak}
Aad~W. van~der Vaart and Jon~A. Wellner.
\newblock \emph{Weak convergence and empirical processes}.
\newblock Springer series in statistics. Springer, 2000.

\bibitem[Vapnik(1999)]{vapnik1999nature}
Vladimir Vapnik.
\newblock \emph{The nature of statistical learning theory}.
\newblock Springer science \& business media, 1999.

\bibitem[Veitch et~al.(2021)Veitch, D'Amour, Yadlowsky, and
  Eisenstein]{veitch2021counterfactual}
Victor Veitch, Alexander D'Amour, Steve Yadlowsky, and Jacob Eisenstein.
\newblock Counterfactual invariance to spurious correlations: Why and how to
  pass stress tests.
\newblock In \emph{Advances in Neural Information Processing Systems}, 2021.

\bibitem[Vershynin(2018)]{vershynin2018}
Roman Vershynin.
\newblock \emph{High-dimensional probability: An introduction with applications
  in data science}.
\newblock Cambridge Series in Statistical and Probabilistic Mathematics.
  Cambridge University Press, 2018.

\bibitem[Volpi et~al.(2018)Volpi, Namkoong, Sener, Duchi, Murino, and
  Savarese]{volpi2018generalizing}
Riccardo Volpi, Hongseok Namkoong, Ozan Sener, John Duchi, Vittorio Murino, and
  Silvio Savarese.
\newblock Generalizing to unseen domains via adversarial data augmentation.
\newblock In \emph{Advances in Neural Information Processing Systems}, 2018.

\bibitem[Wah et~al.(2011)Wah, Branson, Welinder, Perona, and
  Belongie]{wah2011caltech}
Catherine Wah, Steve Branson, Peter Welinder, Pietro Perona, and Serge
  Belongie.
\newblock The caltech-ucsd birds-200-2011 dataset.
\newblock Technical report, 2011.

\bibitem[Wainwright(2019)]{wainwright2019}
Martin~J. Wainwright.
\newblock \emph{High-dimensional statistics: A non-asymptotic viewpoint}.
\newblock Cambridge Series in Statistical and Probabilistic Mathematics.
  Cambridge University Press, 2019.

\bibitem[Wald et~al.(2021)Wald, Feder, Greenfeld, and
  Shalit]{wald2021calibration}
Yoav Wald, Amir Feder, Daniel Greenfeld, and Uri Shalit.
\newblock On calibration and out-of-domain generalization.
\newblock \emph{Advances in Neural Information Processing Systems}, 2021.

\bibitem[Wang et~al.(2022)Wang, Yi, Chen, and Zhu]{wang2022out}
Ruoyu Wang, Mingyang Yi, Zhitang Chen, and Shengyu Zhu.
\newblock Out-of-distribution generalization with causal invariant
  transformations.
\newblock In \emph{Conference on Computer Vision and Pattern Recognition},
  2022.

\bibitem[Williams et~al.(2018)Williams, Nangia, and Bowman]{williams2018broad}
Adina Williams, Nikita Nangia, and Samuel Bowman.
\newblock A broad-coverage challenge corpus for sentence understanding through
  inference.
\newblock In \emph{Conference of the North American Chapter of the Association
  for Computational Linguistics}, 2018.

\bibitem[Xie et~al.(2020)Xie, Kumar, Jones, Khani, Ma, and Liang]{xie2020in}
Sang~Michael Xie, Ananya Kumar, Robbie Jones, Fereshte Khani, Tengyu Ma, and
  Percy Liang.
\newblock In-n-out: Pre-training and self-training using auxiliary information
  for out-of-distribution robustness.
\newblock In \emph{International Conference on Learning Representations}, 2020.

\bibitem[Xu \& Raginsky(2017)Xu and Raginsky]{xu2017information}
Aolin Xu and Maxim Raginsky.
\newblock Information-theoretic analysis of generalization capability of
  learning algorithms.
\newblock In \emph{Advances in Neural Information Processing Systems}, 2017.

\bibitem[Ye et~al.(2021)Ye, Xie, Cai, Li, Li, and Wang]{ye2021towards}
Haotian Ye, Chuanlong Xie, Tianle Cai, Ruichen Li, Zhenguo Li, and Liwei Wang.
\newblock Towards a theoretical framework of out-of-distribution
  generalization.
\newblock \emph{Advances in Neural Information Processing Systems}, 2021.

\bibitem[Ye et~al.(2022)Ye, Li, Hong, Bai, Chen, Zhou, and Li]{ye2021ood}
Nanyang Ye, Kaican Li, Lanqing Hong, Haoyue Bai, Yiting Chen, Fengwei Zhou, and
  Zhenguo Li.
\newblock Ood-bench: Benchmarking and understanding out-of-distribution
  generalization datasets and algorithms.
\newblock In \emph{Conference on Computer Vision and Pattern Recognition},
  2022.

\bibitem[Yi et~al.(2021{\natexlab{a}})Yi, Hou, Shang, Jiang, Liu, and
  Ma]{yi2021reweighting}
Mingyang Yi, Lu~Hou, Lifeng Shang, Xin Jiang, Qun Liu, and Zhi-Ming Ma.
\newblock Reweighting augmented samples by minimizing the maximal expected
  loss.
\newblock In \emph{International Conference on Learning Representations},
  2021{\natexlab{a}}.

\bibitem[Yi et~al.(2021{\natexlab{b}})Yi, Hou, Sun, Shang, Jiang, Liu, and
  Ma]{yi2021improved}
Mingyang Yi, Lu~Hou, Jiacheng Sun, Lifeng Shang, Xin Jiang, Qun Liu, and
  Zhi-Ming Ma.
\newblock Improved ood generalization via adversarial training and
  pre-training.
\newblock In \emph{International Conference on Machine Learning},
  2021{\natexlab{b}}.

\bibitem[Zhou et~al.(2017)Zhou, Lapedriza, Khosla, Oliva, and
  Torralba]{zhou2017places}
Bolei Zhou, Agata Lapedriza, Aditya Khosla, Aude Oliva, and Antonio Torralba.
\newblock Places: A 10 million image database for scene recognition.
\newblock \emph{IEEE Transactions on Pattern Analysis and Machine
  Intelligence}, 40\penalty0 (6):\penalty0 1452--1464, 2017.

\bibitem[Zhou et~al.(2022)Zhou, Lin, Pi, Zhang, Xu, Cui, and
  Zhang]{zhou2022model}
Xiao Zhou, Yong Lin, Renjie Pi, Weizhong Zhang, Renzhe Xu, Peng Cui, and Tong
  Zhang.
\newblock Model agnostic sample reweighting for out-of-distribution learning.
\newblock In \emph{International Conference on Machine Learning}, 2022.

\bibitem[Zhu et~al.(2017)Zhu, Park, Isola, and Efros]{zhu2017unpaired}
Jun-Yan Zhu, Taesung Park, Phillip Isola, and Alexei~A Efros.
\newblock Unpaired image-to-image translation using cycle-consistent
  adversarial networks.
\newblock In \emph{IEEE International Conference on Computer Vision}, 2017.

\end{thebibliography}
	\bibliographystyle{iclr2023_conference}
	\newpage

	\appendix
	\section{Label Shift}\label{subsec: label shift}
In the sequel, we may omit the subscribe if no obfuscation. Our discussions in the main body of this paper are built upon the assumption that marginal distribution of label $Y$ is invariant i.e., $P_{Y} = Q_{Y}$. In this section, we explore OOD generalization without such invariant assumption. Before presenting our discussion, we give the definition of total variation distance. 
\begin{definition}
	The total variation distance between two distributions $P, Q$ defined on the same measurable space $\cX$ is 
	\begin{equation}
		\small
		\mathrm{TV}(P, Q) = \frac{1}{2}\int_{\cX}\left|dP(x) - dQ(x)\right|. 
	\end{equation}
\end{definition}
In Theorem \ref{thm:conditional independence}, we show that the gap between the performances of the model on training and OOD test data disappears if the model satisfies conditional independence such that $f(X)\perp Z\mid Y$. However, we show by the following example that the gap will not disappear if the marginal distribution of $Y$ also varies across training and test data. 
\begin{example}\label{ex: label shift}\emph{
		Suppose $Y, Z\in \{-1,1\}$ and a specialized loss function 
		\begin{equation}
			\small
			\cL(f(X), Y) = 1_{\{Y = 1\}}(5-f(X)) + 1_{\{Y = -1\}}(2+f(X)), 
		\end{equation}
		where $f(\cdot)$ is any classifier whose output is in $\{-1,1\}$. Let $P$, $Q$ be two distributions such that $P_{X\mid Y,Z} = Q_{X\mid Y,Z}$ but $P_{Y}\neq Q_{Y}$. We suppose $X \perp Z\mid Y$, and thus $f(X)\perp Z\mid Y$. Thus, 
		\begin{equation}
			\small
			P_{X\mid Y}(\bx\mid y) = \sum_{z\in\cZ} P_{X, Z\mid Y}(\bx, z\mid y) = P_{X\mid Y}(\bx\mid y)\sum_{z\in\cZ} P_{Z\mid Y}(z\mid y),
		\end{equation}
		is unrelated to $P_{Z\mid Y}$. Then we have $P_{X\mid Y} = Q_{X\mid Y}$. Thus 
		\begin{equation}
			\small
			\begin{aligned}
				\mE_{P}[\cL(f(X),Y)] &=  P_{Y}(Y=1)\mE_{P}[\cL(f(X),Y)\mid Y=1] + P_{Y}(Y = -1)\mE_{P}[\cL(f(X),Y)\mid Y = -1], \\
				\mE_{Q}[\cL(f(X),Y)] &= Q_{Y}(Y=1)\mE_{Q}[\cL(f(X),Y)\mid Y=1] + Q_{Y}(Y = -1)\mE_{Q}[\cL(f(X),Y)\mid Y = -1].
			\end{aligned}
		\end{equation}
		Since $P_{X\mid Y} = Q_{X\mid Y}$,
		\begin{equation}\label{eq:tv lower bound}
			\small
			\begin{aligned}
				& |\mE_{P}[\cL(f(X),Y)] - \mE_{Q}[\cL(f(X),Y)]| \\
				& = |(P_{Y}(Y=1) - Q_{Y}(Y=1))(\mE_{P}[\cL(f(X),Y)\mid Y=1] - \mE_{Q}[\cL(f(X),Y)\mid Y = -1])|\\
				& \geq |4 - 3| \times |P(Y=1) - Q(Y=1)| \\
				& = |P_{Y}(Y=1) - Q_{Y}(Y=1)| \\
				& = \mathrm{TV}(P_{Y}, Q_{Y}),
			\end{aligned}
		\end{equation}
		where $\mathrm{TV}(P_{Y}, Q_{Y})$ is the total variation distance of the marginal distributions $P_{Y}, Q_{Y}$. This inequality holds for any $f(X)$, and hence the gap can never be removed by representation learning like what we do in Theorem \ref{thm:conditional independence}.}
\end{example}
\par
The example indicates that under shifted label distribution, the conditional independent model can not generalize on OOD data. Thus, we consider the reweighted loss to fix the bias brought by the shifted label distribution. The formal result is stated as follows. 
\begin{theorem}\label{thm: label shift}
	Let $P, Q$ be two distributions such that $P_{X\mid Y, Z} = Q_{X\mid Y, Z}$ but $P_{Y}$ does not necessary equals to $Q_{Y}$. $w(y):\cY\rightarrow\bbR^{+}$ is a weighting function satisfies $E_{P}[w(Y)] = 1$. Then if $f(X)\perp Z\mid Y$,
	\begin{equation}
		\small
		\left|\mE_{P}[w(Y)\cL(f(X),Y)] - \mE_{Q}[\cL(f(X),Y)]\right| \leq 2M \mathrm{TV}(P^{w}_{Y}, Q)
	\end{equation}
	where $P^{w}_{Y}$ is the reweighted label distribution defined as $P^{w}_{Y}(A) = \int_{A}w(y)dP_{Y}(y)$ for any measurable set $A\subset\cY$. 
\end{theorem}
\begin{proof}
	Because $P_{X\mid Y, Z} = Q_{X\mid Y, Z}$ and $f(X)\perp Z\mid Y$, as in Appendix \ref{app:proofs in sec guaranteed}, we have $P(f(X)\mid Y) = Q(f(X)\mid Y)$. Thus 
	\begin{equation}\label{eq:tv upper bound}
		\small
		\begin{aligned}
			& \left|\mE_{P}[w(Y)\cL(f(X),Y)] - \mE_{Q}[\cL(f(X),Y)]\right| \\
			& = \left|\int_{\cY} w(y)\mE_{P}[\cL(f(X),Y)\mid Y=y]dP_{Y}(y) - \int_{y\in \cY}\mE_{Q}[\cL(f(X),Y)\mid Y=y] dQ_{Y}(y)\right| \\ 
			& \leq \int_{\cY}  M\left|w(y)dP_{Y}(y) - dQ_{Y}(y)\right| \\
			& = 2M \mathrm{TV}(P^{w}_{Y}, Q_{Y}).
		\end{aligned}
	\end{equation}
\end{proof}
\begin{remark}
	The total variation distance $\mathrm{TV}(P^{w}_{Y}, Q_{Y})$ appears in the upper bound to the gap between the two population risk in \eqref{eq:tv upper bound}. Moreover, this terms seems to be inevitable since it also appears in the lower bound in \eqref{eq:tv lower bound}. 
\end{remark}
According to Theorem \ref{thm: label shift}, we have invariance relationship $\mE_{P}[w(Y)\cL(f(X),Y)] = \mE_{Q}[\cL(f(X),Y)]$ if we can take $w(y) = dQ_{Y}(y) / dP_{Y}(y)$. Thus if the label distribution in the test data is available, minimizing the reweighted loss with its weights decided by the ration of two label distributions can guarantee the OOD generalization capability of the model. 
\par
However, the label distribution of test data are usually unavailable in practical. Thus for unknown test label distribution, we alternatively chose the weight $w(\cdot)$ to minimize the worst-case upper bound 
\begin{equation}
	\small
	\sup_{Q}\mathrm{TV}(P^{w}_{Y}, Q_{Y}) = \frac{1}{2}\sup_{Q}\int_{\cY} \left|w(y)dP_{Y}(y) - dQ_{Y}(y)\right|,
\end{equation}
given the training distribution $P$, where the supremum is taken over all distributions $Q$ such that $Q_{X\mid Y,Z} = P_{X\mid Y,Z}$. Then by minimizing the reweighted loss under such weight $w(\cdot)$, we get a model with minimized worst-case risk over different distributions. 
\begin{proposition}
	Suppose that $\cY$ is a discrete space, then if $P_{Y}(Y = y) > 0$ for all $y\in \cY$ and $w^{*}(y) = \frac{1}{|\cY|P_{Y}(Y = y)}$, we then have
	\begin{equation}
		\small
		w^{*}(\cdot)\in\argmin_{w(\cdot):\mE_{P}[w(Y)] = 1}\left\{\sup_{Q\in\cP}\mathrm{TV}(P^{w}_{Y}, Q_{Y})\right\},
	\end{equation}
	where $|\cY|$ is the cardinal of $\cY$.
\end{proposition}
\begin{proof}
	From Section A.6.2 in \citep{van2000weak}, we know that $\mathrm{TV}(P^{w}_{Y}, Q_{Y}) = \sup_{A\subset\cY}|\sum_{y\in A}w(y)P_{Y}(Y = y) - Q_{Y}(Y = y)|$. Thus 
	\begin{equation}
		\small
		\begin{aligned}
			\sup_{Q\in\cP}\mathrm{TV}(P^{w}_{Y}, Q_{Y}) &= \sup_{Q\in\cP}\sup_{A\subset\cY}\left|\sum_{y\in A}w(y)P_{Y}(Y = y) - Q_{Y}(Y = y)\right| \\
			& = \sup_{Q\in\cP}\sum_{y\in\{y^{\prime}:w(y^{\prime})P_{Y}(Y = y^{\prime}) \geq Q_{Y}(Y = y^{\prime})\}}\left(w(y)P_{Y}(Y = y) - Q_{Y}(Y = y)\right) \\
			& = 1 - \min_{y\in\cY}w(y)P_{Y}(Y = y)
		\end{aligned}
	\end{equation}
	due to $w(y)P_{Y}(Y = y) \geq 0$ and $\mE_{P}[w(Y)] = 1$. Then, we have
	\begin{equation}
		\small
		\begin{aligned}
			\min_{w(\cdot):\mE_{P}[w(Y)] = 1}\sup_{Q\in\cP}\mathrm{TV}(P^{w}_{Y}, Q_{Y}) &= \min_{w(\cdot):\mE_{P}[w(Y)] = 1}\left\{1 - \min_{y\in\cY}w(y)P_{Y}(Y = y)\right\} \\
			& = 1 - \max_{w(\cdot):\mE_{P}[w(Y)] = 1}\left\{\min_{y\in\cY}w(y)P_{Y}(Y = y)\right\} \\
			& = 1 - w^{*}(y)P_{Y}(Y=y) \\
			& = \frac{|\cY| - 1}{|\cY|}.
		\end{aligned}
	\end{equation}
	The third equality is due to $|\cY|\min_{y\in\cY}w(y)P_{Y}(Y = y)\leq \sum_{y\in\cY}w(y)P_{Y}(Y = y) = 1$, and the equality is taken when $w(\cdot) = w^{*}(\cdot)$. 
\end{proof}
\section{Proofs in Section \ref{sec:Generalizing on OOD Data Under Spurious Correlation}}\label{app:proofs in sec generalizing}
In this section, we present the proofs of results in Section \ref{sec:Generalizing on OOD Data Under Spurious Correlation}. 
\counterexample*
\begin{proof}
	Let us consider the following example that 
	\begin{equation}\label{eq:data}
		\small
		X = \begin{pmatrix}
			Y\cdot \bmu_{1} \\
			Z\cdot \bmu_{2}
		\end{pmatrix} + \bxi, 
	\end{equation}
	where $\bxi\sim\cN(\textbf{0}, I_{d_{1}+d_{2}})$, $Y, Z\in\{-1, 1\}$ and follow the standard binomial distribution. Denote the training distribution as $P$. In this example, $Z$ is the spurious attributes. The correlation coefficient between $Y$ and $Z$ is denoted as $\sigma_{YZ}(Q)$ for $Q\in\cP$. One can verify that 
	\begin{equation}
		\small
		\sigma_{YZ}(Q) = \mE_{Q}[YZ] = Q(Y = Z) - Q(Y\neq Z) = 2Q(Y = Z) - 1.
	\end{equation}
	Let us consider the linear classifier $f_{\btheta}(\bx) = \btheta^{\top}\bx$ and its loss on data $(X, Y)$ is the exponential loss \cite{soudry2018implicit}
	\begin{equation}
		\small
		\cL(f_{\btheta}(X), Y) = e^{-Yf_{\btheta}(X)}.
	\end{equation}
	Thus we can compute the population risk  
	\begin{equation}
		\small
		\begin{aligned}
			R_{pop}(P, f_{\btheta}) & = \mE_{P}\left[\exp(-Yf_{\btheta}(X))\right] \\
			& =  \mE\left[\exp\left(-\btheta_{1}^{\top}\bmu_{1} - YZ\btheta_{2}^{\top}\bmu_{2} - \btheta^{\top}\bxi\right)\right] \\
			& =  \mE\left[\exp\left(-\btheta_{1}^{\top}\bmu_{1} - \btheta_{2}^{\top}\bmu_{2} - \btheta^{\top}\bxi\right)\mid Y = Z\right]P(Y = Z) \\
			& + \mE\left[\exp\left(-\btheta_{1}^{\top}\bmu_{1} + \btheta_{2}^{\top}\bmu_{2} + \btheta^{\top}\bxi\right)\mid Y \neq Z\right]P(Y \neq Z) \\
			& = \left(\frac{1 + \sigma_{YZ}(P)}{2}\right)\exp\left(-\btheta_{1}^{\top}\bmu_{1} - \btheta_{2}^{\top}\bmu_{2} + \frac{\|\btheta\|^{2}}{2}\right) \\
			& + \left(\frac{1 - \sigma_{YZ}(P)}{2}\right)\exp\left(-\btheta_{1}^{\top}\bmu_{1} + \btheta_{2}^{\top}\bmu_{2} + \frac{\|\btheta\|^{2}}{2}\right),
		\end{aligned}
	\end{equation}
	Since $R_{pop}(P, f_{\btheta})$ is continuous w.r.t. to $\sigma_{YZ}(P)$ and $\btheta$, we have that $\btheta^{*}(P) = \argmin_{\btheta}R_{pop}(P, f_{\btheta})$ is continuous to $\sigma_{YZ}(P)$. Since $\sigma_{YZ}(P)\in[-1, 1]$, we conclude $\|\btheta^{*}(P)\|$ is upper bounded. W.o.l.g. we assume $\|\btheta^{*}(P)\|\leq 1$ for any $\sigma_{YZ}(P)\in[-1, 1]$, then for any $\sigma_{YZ}(P)$ we know $\btheta^{*}(P)$ satisfies the first order optimality condition such that 
	\begin{equation}
		\small
		\begin{aligned}
			0 &= \left(\frac{1 + \sigma_{YZ}(P)}{2}\right)\left(\btheta^{*}(P) - \bmu\right)\exp\left(-\btheta^{\top}\bmu + \frac{\|\btheta\|^{2}}{2}\right) \\
			& + \left(\frac{1 - \sigma_{YZ}(P)}{2}\right)\left(\btheta^{*}(P) - \tilde{\bmu}\right)\exp\left(-\btheta^{\top}\tilde{\bmu} +  \frac{\|\btheta\|^{2}}{2}\right).
		\end{aligned}
	\end{equation}
	where $\tilde{\bmu} = (\bmu^{\top}_{1}, -\bmu_{2}^{\top})^{\top}$. Thus for $\sigma_{YZ}(P)\neq \pm 1$ we have 
	\begin{equation}\label{eq:first order condition 1}
		\small
		\btheta^{*}(P) - \bmu = \left(\frac{1 - \sigma_{YZ}(P)}{1 + \sigma_{YZ}(P)}\right)\left(\tilde{\bmu} - \btheta^{*}(P)\right)\exp\left(2\btheta_{2}^{\top}\bmu_{2}\right).
	\end{equation}
	Thus we can take a $\sigma_{YZ}(P)\to 1$ to make $\|\btheta^{*}(P) - \bmu\|\leq \epsilon\|\bmu\|$ for any small $\epsilon$ where the $\epsilon$ can be independent of $\bmu$. 
	\par 
	Now we show that the linear model $f_{\btheta^{*}(P)}(\cdot)$ with its prediction on $Y$ as $\text{sign}(f_{\btheta^{*}(P)}(\cdot))$ can make correct prediction with high probability. Let us see the error of linear model $f_{\btheta^{*}(P)}(\cdot)$ on the data from training distribution $P$. Simple algebra show that $\btheta^{*\top}(P)X\mid Y\sim\cN(Y\btheta^{*\top}_{1}(P)\bmu_{1} + Z\btheta^{*\top}_{2}(P)\bmu_{2}, \|\btheta^{*\top}_{1}(P)\|^{2})$. Then under condition of $Y = Z$, we have 
	\begin{equation}
		\small
		\begin{aligned}
			\btheta^{*\top}(P)X - Y\|\bmu\|^{2} = Y\btheta^{*\top}(P)\bmu - Y\|\bmu\|^{2} + \btheta^{*\top}(P)\bxi =  Y\left(\btheta^{*}(P) - \bmu\right)^{\top}\bmu + \btheta^{*\top}(P)\bxi.
		\end{aligned}
	\end{equation}
	Thus from the sub-Gaussian property of Gaussian random variable, for any $\delta > 0$
	\begin{equation}\label{eq:high probability inequality}
		\small
		\begin{aligned}
			& P_{X\mid Y}\left(\left|\btheta^{*\top}(P)X - Y\|\bmu\|^{2}\right|\geq \delta\mid Y\right) \\
			&= P_{X\mid Y}\left(\left|\btheta^{*\top}(P)X - Y\|\bmu\|^{2}\right|\geq \delta\mid Y, Y = Z\right)\left(\frac{1 + \sigma_{YZ}(P)}{2}\right) \\
			& + P_{X\mid Y}\left(\left|\btheta^{*\top}(P)X - Y\|\bmu\|^{2}\right|\geq \delta\mid Y, Y \neq Z\right)\left(\frac{1 - \sigma_{YZ}(P)}{2}\right) \\
			& \leq P_{X\mid Y}\left(\left|\btheta^{*\top}(P)\bxi\right|\geq \delta - \epsilon\|\bmu\|^{2}\mid Y, Y = Z\right)\left(\frac{1 + \sigma_{YZ}(P)}{2}\right) + \left(\frac{1 - \sigma_{YZ}(P)}{2}\right) \\
			& \leq \exp\left(-\frac{(\delta - \epsilon\|\bmu\|^{2})^{2}}{2\left\|\btheta^{*}(P)\right\|^{2}}\right)\left(\frac{1 + \sigma_{YZ}(P)}{2}\right) + \left(\frac{1 - \sigma_{YZ}(P)}{2}\right).
		\end{aligned}
	\end{equation}
	We may take $\delta = \|\bmu\|^{2}/2 + \epsilon\|\bmu\|^{2}$, and $\sigma_{YZ}(P)\to 1$, due to $\|\btheta(\sigma_{YZ}(P))\| \leq 1$ and for a large enough $\|\bmu\|$, with a high probability, we have 
	\begin{equation}
		\small
		Y\|\bmu\|^{2} - \left(\frac{1}{2} + \epsilon\right)\|\bmu\|^{2}\leq \btheta^{*\top}(P)X \leq Y\|\bmu\|^{2} + \left(\frac{1}{2} + \epsilon\right)\|\bmu\|^{2}.
	\end{equation}
	Since $\epsilon\to 0$ for $\sigma_{YZ}(P)\to 1$, we have proved that the population minimizer $\btheta^{*}(P)$ has nearly perfect performance on the data from training distribution. 
	\par
	However, a similar argument of \eqref{eq:high probability inequality} shows that for data drawn from distribution $Q\in\cP$
	\begin{equation}
		\small
		\begin{aligned}
		Q_{X\mid Y}\left(\left|\btheta^{*\top}(P)X - Y(\|\bmu_{1}\|^{2} - \|\bmu_{2}\|^{2})\right|\geq \delta\mid Y\right) & \leq  \exp\left(-\frac{(\delta - \epsilon\|\bmu\|^{2})^{2}}{2\left\|\btheta^{*}(P)\right\|^{2}}\right)\left(\frac{1 - \sigma_{YZ}(Q)}{2}\right) \\
		& + \left(\frac{1 + \sigma_{YZ}(Q)}{2}\right)
		\end{aligned}
	\end{equation}
	for any $\delta > 0$. Again, by taking $\sigma_{YZ}(Q)\to -1$ we get 
	\begin{equation}
		\small
		\begin{aligned}
		Y(\|\bmu_{1}\|^{2} - \|\bmu_{2}\|^{2}) & - \left(\frac{1}{2} + \epsilon\right)\left(\|\bmu_{1}\|^{2} + \|\bmu_{2}\|\right) \leq \btheta^{*\top}(P)X \\
		& \leq Y(\|\bmu_{1}\|^{2} - \|\bmu_{2}\|^{2}) + \left(\frac{1}{2} + \epsilon\right)\left(\|\bmu_{1}\|^{2} + \|\bmu_{2}\|\right)
		\end{aligned}
	\end{equation}
	with high probability. We can take, for example, $\|\bmu_{2}\|^{2} > \left(\frac{3 + 2\epsilon}{1 - \epsilon}\right)\|\bmu_{1}\|^{2}$ for $\epsilon\to0$, then under $Y = -1$, the inequality becomes 
	\begin{equation}
		\small
		0 < \left(\frac{1}{2} - \epsilon\right)\|\bmu_{2}\|^{2} - \left(\frac{3}{2} + \epsilon\right)\|\bmu_{1}\|^{2} \leq \btheta^{*\top}(P)X \leq \left(\frac{3}{2} + \epsilon\right)\|\bmu_{2}\|^{2} - \left(\frac{1}{2} - \epsilon\right)\|\bmu_{1}\|^{2} ,
	\end{equation}
	which shows the prediction given by $f_{\btheta^{*}(P)}(\cdot)$ for $Y = -1$ is incorrect with high probability. A similar argument can be verified for $Y = 1$. Then we complete our proof. 
\end{proof}
\thminvariantmodel*
\begin{proof}
 	The difference of $Y\mid f(X)$ for any $(X, Y)\sim Q$ with $Q\in \cP$ originates from the different spurious correlation i.e., the different $Q_{Z\mid Y}$. Thus to obtain our result, it is suffice to prove that the distribution of $Y\mid f(X)$ is independent of $Q_{Z\mid Y}$. To see this, for any measurable sets $A, B\subset \cY$, 
	\begin{equation}
		\small
		\begin{aligned}
			& Q_{Y\mid X}(Y\in A\mid f(X)\in B) \\
			& = \frac{Q_{X\mid Y}(f(X)\in B\mid Y\in A)Q_{Y}(Y\in A)}{Q_{X\mid Y}(f(X)\in B\mid Y\in A)Q_{Y}(Y\in A) + Q_{X\mid Y}(f(X)\in B\mid Y\notin A)Q_{Y}(Y\notin A)} \\
			& = \frac{1}{1 + \frac{Q_{X\mid Y}(f(X)\in B\mid Y\notin A)Q_{Y}(Y\notin A)}{Q_{X\mid Y}(f(X)\in B\mid Y\in A)Q_{Y}(Y\in A)}}. 
		\end{aligned}
	\end{equation}
	As $Q_{Y}(Y\notin A)/Q_{Y}(Y\in A)$ is invariant across $Q\in \cP$. Then for the $Q_{X\mid Y}(f(X)\in B\mid Y\notin A) / Q_{X\mid Y}(f(X)\in B\mid Y\in A)$, we see 
	\begin{equation}
		\small
		\begin{aligned}
			\frac{Q_{X\mid Y}(f(X)\in B\mid Y\notin A)}{Q_{X\mid Y}(f(X)\in B\mid Y\in A)} & = \frac{\int_{\cZ} Q_{X, Z\mid Y}(f(X)\in B, z\mid Y\notin A)dz}{\int_{\cZ} Q_{X, Z\mid Y}(f(X)\in B, z\mid Y \in A)dz} \\
			& = \frac{Q_{X\mid Y}(f(X)\in B\mid Y\notin A)\int_{\cZ} Q_{Z\mid Y}(z\mid Y\notin A)dz}{Q_{X\mid Y}(f(X)\in B\mid Y\in A)\int_{\cZ} Q_{Z\mid Y}(z\mid Y\in A)dz},
		\end{aligned}
	\end{equation}
	where the second equality is from the independent condition that $f(X)\perp Z\mid Y$. From the calculation, we figure out that the distribution of $Y\mid f(X)$ is independent of spurious correlation $P_{Z \mid Y}$ due to the arbitrariness of $A, B\in \cY$. Then we prove $Y\mid f(X)$ is invariant over $Q\in \cP$.. 
	\par
	To provide the invariance of $\mE_{Q}[\cL(f(X), Y)]$, it is suffice to show that for the union distribution of $(f(X), Y)$ is invariant w.r.t. $Q$ for $Q\in \cP$. Thus for any sets $A, B\subset \cY$ and $(X, Y)\sim Q\in \cP$ 
	\begin{equation}
		\small
		\begin{aligned}
			Q_{X, Y}(Y\in A, f(X)\in B) & = Q_{X\mid Y}(f(X)\in B\mid Y\in A)Q_{Y}(Y\in A) \\
			& = Q_{Y}(Y\in A)\int_{\cZ} Q_{X, Z\mid Y}(f(X)\in B, z\mid Y\in A)dz \\
			& = Q_{Y}(Y\in A)Q_{X\mid Y}(f(X)\in B\mid Y\in A)\int_{\cZ}Q_{Z\mid Y}(z\mid Y\in A) dz.
		\end{aligned}
	\end{equation}
	Since $f(X)\perp Z\mid Y$, we figure out the $Q_{X, Y}(Y\in A, f(X)\in B)$ is independent of with spurious correlation $Q_{Z \mid Y}$. Then due to the arbitrary of $A, B\in \cY$, we summarize that the union distribution of $(f(X), Y)$ is invariant w.r.t. $Q$ for $Q\in \cP$. Then the proof is completed. 
\end{proof}
\section{Proofs in Section \ref{sec:learning OOD generalizable model}}
\subsection{Proofs in Section \ref{sec:guaranteed OOD generalization error}}\label{app:proofs in sec guaranteed}
\boundedgap*
\begin{proof}
	This theorem can be computed via the assumption in \eqref{eq:data structual}. We have 
	\begin{equation}
		\small
		\begin{aligned}
			\mE_{P}[\cL(f(X), Y)] & = \mE_{P}\left[\mE[\cL(f(X), Y)\mid Y, Z]\right] \\
			& = \mE_{Y}\left[\int_{\cZ} \mE_{X}[\cL(f(X), Y)\mid Y, Z = z]dP(z\mid Y)\right].
		\end{aligned}
	\end{equation}
	Due to \eqref{eq:data structual}, the first expectation is invariant with $P, Q\in \cP$, while the second expectation is a function of $(Y, z)$ independent the choice of $P$. Thus
	\begin{equation}
		\small
		\begin{aligned}
			\left|\mE_{P}[\cL(f(X), Y)] - \mE_{Q}[\cL(f(X), Y)]\right|  & \leq \mE\left[\sup_{z_{1}, z_{2}}\left|\mE[\cL(f(X),Y)\mid Y,  z_{1}] - \mE[\cL(f(X), Y)\mid Y, z_{2}]\right|\right] \\
			& \leq \mathrm{CSV}(f),
		\end{aligned}
	\end{equation}
	where the last inequality is due to the loss function is non-negative. Then due to 
	\begin{equation}
		\small
		\begin{aligned}
			\sup_{Q\in\cP}\left|R_{emp}(f, P) - \!\!\right.\left. R_{pop}(f, Q)\right|  &\leq \left|R_{emp}(f, P) - R_{pop}(f, P)\right| + \sup_{Q\in\cP}\left|R_{pop}(f, P) - R_{pop}(f, Q)\right| \\
			& \leq  \left|R_{emp}(f, P) - R_{pop}(f, P)\right| + \mathrm{CSV}(f),
		\end{aligned}
	\end{equation}
	we get the theorem. 
\end{proof}
\par
Next we provide the definitions of KL-divergence, mutual information, and conditional mutual information which are useful to prove Theorem \ref{thm:gen gap}. 
\begin{definition}[KL-Divergence]
	Let $P, Q$ be two distributions with the same support and $P$ is absolutely continuous w.r.t. $Q$. Then the KL divergence from $Q$ to $P$ is 
	\begin{equation}
		\small
		D_{\mathrm{KL}}(P\parallel Q) = \mE_{V\sim P}\left[\log{\frac{dP}{dQ}(V)}\right],
	\end{equation}
	where $\frac{dP}{dQ}$ is the Radon–Nikodym derivative of $P$ w.r.t. $Q$. 
\end{definition}
\begin{definition}[Mutual Information]
	For random variables $V_{1}, V_{2}$ with joint distribution $P_{V_{1}, V_{2}}$, the mutual information between them is 
	\begin{equation}
		\small
		I(V_{1}; V_{2}) = D_{\mathrm{KL}}(P_{V_{1}, V_{2}}\parallel P_{V_{1}} \times P_{V_{2}}).
	\end{equation} 
\end{definition} 
\begin{definition}[Conditional Mutual Information]
	For three random variables $U, V, W$, the mutual information between $U, V$ conditional on $W$ is 
	\begin{equation}
		\small
		I(U; V\mid W) = \mE_{w\sim P_{W}}\left[I(U\mid W = w; V\mid W = w)\right]. 
	\end{equation}
\end{definition}
Before presenting the proof of Theorem \ref{thm:gen gap}, we need the following lemma.
\begin{lemma}\label{lem:chain rule}
	Let $U, V, W$ be three random variables such that $U$ and $W$ are independent with each other, then 
	\begin{equation}
		\small
		I(W; U + V) \leq I(W; V\mid U). 	
	\end{equation}
\end{lemma}
\begin{proof}
	By Data Processing Inequality \citep{xu2017information}, we have 
	\begin{equation}
		\small
			I(W; U + V) \leq I(W; U, V) = I(W; U) + I(W; V\mid U) = I(W; V\mid U)
	\end{equation}
	Thus the proof is completed. 
\end{proof}
Now we are ready to give the proof of Theorem \ref{thm:gen gap}.
\gengap*
\begin{proof}
	Let $\tilde{\cS} = \{(\tilde{x}_{i}, \ty_{i})\}$ be another $n$ samples drawn from $P$ independent of $\cS$. W.o.l.g., we assume $\mE_{P\times f_{\btheta_{\cS}}}[\cL(f_{\btheta_{\cS}}(\bx), y)] = 0$, otherwise we can replace $\cL(f_{\btheta_{\cS}}(\bx_{i}), y_{i})$  with $\cL(f_{\btheta_{\cS}}(\bx_{i}), y_{i}) - \mE_{P\times f_{\btheta_{\cS}}}[\cL(f_{\btheta_{\cS}}(\bx), y)]$. For any $\lambda > 0$ by Donsker-Varadhan's inequality, 
	\begin{equation}
		\small
		\begin{aligned}
			D_{\mathrm{KL}}(P_{\cS\times f_{\btheta_{\cS}}}\parallel P_{\cS}\times P_{f_{\btheta_{\cS}}}) & \geq \mE_{\cS\times f_{\btheta_{\cS}}}\left[\frac{\lambda}{n}\sum\limits_{i=1}^{n}\cL(f_{\btheta_{\cS}}(x_{i}), y_{i})\right] - \log\mE_{\tilde{\cS}\times f_{\btheta_{\cS}}}\left[\exp\left(\frac{\lambda}{n}\sum\limits_{i=1}^{n}\cL(f_{\btheta_{\cS}}(\tilde{x}_{i}), \tilde{y}_{i})\right)\right].
		\end{aligned}
	\end{equation}
	Then for any $\btheta$, $\lambda > 0$, and Lebesgue measurable function $g(\cdot)$, 
	\begin{equation}
		\small
		\begin{aligned}
			\lambda\mE\left[R_{emp}(f_{\btheta_{\cS}}, P)\right] & \leq D_{\mathrm{KL}}(P_{\cS\times f_{\btheta_{\cS}}}\parallel P_{\cS}\times P_{f_{\btheta_{\cS}}}) + \log\mE_{\tilde{\cS}\times\btheta_{\cS}}\left[\exp\left(\frac{\lambda}{n}\sum\limits_{i=1}^{n}\cL(f_{\btheta_{\cS}}(\tilde{x}_{i}), \tilde{y}_{i})\right)\right] \\
			& \overset{a}{\leq} I(\cS_{\bx}, \cS_{y}; f_{\btheta_{\cS}}) + \frac{\lambda^{2}M^{2}}{8n} \\
			& = I(\cS_{\bx}; f_{\btheta_{\cS}}\mid \cS_{y}) + I(\cS_{y}; f_{\btheta_{\cS}}) + \frac{\lambda^{2}M^{2}}{8n} \\
			& \overset{b}{\leq} I(\cS_{\bx - g(z)}; f_{\btheta_{\cS}}\mid \cS_{y}, \cS_{g(z)}) + I(\cS_{y}; f_{\btheta_{\cS}}) + \frac{\lambda^{2}M^{2}}{8n}, 
		\end{aligned}
	\end{equation}
	where $a$ is due to the definition of mutual information, $\cL(f_{\btheta_{\cS}}(\tilde{\bx}_{i}), y_{i})$ is $\frac{M}{2}$-sub Gaussian, $b$ is from Lemma \ref{lem:chain rule}, and the last equality is due to the conditional independence of the model. Thus, we conclude that 
	\begin{equation}
		\small
		\mE\left[R_{emp}(f_{\btheta_{\cS}}, P)\right] \leq \inf_{g}\sqrt{\frac{M^{2}}{4n}\left(I(\cS_{\bx - g(z)}; f_{\btheta_{\cS}}\mid \cS_{y}, \cS_{g(z)}) + I(\cS_{y}; f_{\btheta_{\cS}})\right)}. 
	\end{equation}
	Thus, we complete the proof. 
\end{proof}
\subsection{Proofs in Section \ref{sec:estimators of CSV}}\label{app: proofs in sec:estimators of CSV}
Let $\cF(\Theta)$ and $\|\cdot\|_{L_{\infty}}$ respectively be the parameterized function class and $L_{\infty}$-norm on $\cF(\Theta)$ defined as 
\begin{equation}
	\small
	\|f_{\btheta_{1}} - f_{\btheta_{2}}\|_{L_{\infty}} = \sup_{\bx}|f_{\btheta_{1}}(\bx) - f_{\btheta_{2}}(\bx)| 
\end{equation}
for any $f_{\btheta_{1}}, f_{\btheta_{2}}\in\cF(\Theta)$. To provide the proof of Theorem \ref{thm:gen bound}, we need the following definition of covering number.
\begin{definition}
	A $\epsilon$-cover of metric space $(\epsilon, \cF(\Theta), \|\cdot\|_{L_{\infty}})$ is any point set $\{f_{\btheta_{i}}(\cdot)\}\subseteq\cF(\Theta)$ such that for any $f_{\btheta}(\cdot)\in \cF(\Theta)$, there exists $\btheta_{i}$ satisfies $\|f_{\btheta} - f_{\btheta_{i}}\|_{L_{\infty}}\leq \epsilon$. The covering number $N(\epsilon, \cF(\Theta), \|\cdot\|_{L_{\infty}})$ is the cardinality of the smallest $\epsilon$-cover.  
\end{definition}
\genbound*
\begin{proof}
	First, for any given $\btheta\in \Theta$ and given $Y = k, Z = z$, due to $0\leq \cL(f_{\btheta}(X), Y) \leq M$, by Azuma-Hoeffding’s inequality (Corollary 2.20 in \citep{wainwright2019}), we know that the with probability at least $1 - \delta$ 
	\begin{equation}
		\small
		\begin{aligned}
			& \hat{\cL}_{kz}(f_{\btheta}) - M\sqrt{\frac{\log{(2/\delta)}}{2n_{kz}}} \leq \cL_{kz}(f_{\btheta}) \leq  \hat{\cL}_{kz}(f_{\btheta}) + M\sqrt{\frac{\log{(2/\delta)}}{2n_{kz}}}.
		\end{aligned}
	\end{equation}
	Then we see 
	\begin{equation}
		\small
		\begin{aligned}
			& \sup_{z_{1}, z_{2}}\left(\cL_{kz_{1}}(f_{\btheta}) - \cL_{kz_{2}}(f_{\btheta})\right)
			\leq \sup_{z_{1}, z_{2}}\left[\left(\cL_{kz_{1}}(f_{\btheta}) - \hat{\cL}_{kz_{1}}(f_{\btheta})\right) - \left(\cL_{kz_{2}}(f_{\btheta}) - \hat{\cL}_{kz_{2}}(f_{\btheta})\right)  \right] \\
			& + \sup_{z_{1},z_{2}}\left(\hat{\cL}_{kz_{1}}(f_{\btheta}) - \hat{\cL}_{kz_{2}}(f_{\btheta})\right) \\
			& \leq M\log{\left(\frac{2}{K_{z}\delta}\right)}\sup_{z_{1}, z_{2}}\left(\sqrt{\frac{1}{2n_{kz_{1}}}} + \sqrt{\frac{1}{2n_{kz_{2}}}}\right) + \sup_{z_{1},z_{2}}\left(\hat{\cL}_{kz_{1}}(f_{\btheta}) - \hat{\cL}_{kz_{2}}(f_{\btheta})\right)
		\end{aligned}
	\end{equation}
	holds with probability at least $1 - \delta$. Since the function class $\cF(\Theta)$ is bounded by $M$ under $\|\cdot\|_{L_{\infty}}$, it has finite covering number $N\left(\epsilon, \cF(\Theta), \|\cdot\|_{L_{\infty}}\right)$. Let $f_{\btheta_{1}}(\cdot),\cdots, f_{\btheta_{N}}(\cdot)\in \cF(\Theta)$ be a $\epsilon$-covering of $\cF(\Theta)$ with $N \leq N\left(\epsilon, \cF(\Theta), \|\cdot\|_{L_{\infty}}\right)$ such that $\forall f_{\btheta}\in \cF(\Theta)$, $\exists q\in\{1,\cdots, N\}$, $\|f_{\btheta} - f_{\btheta_{q}}\|_{L_{\infty}}\leq \epsilon$. Thus combining the above inequality, for any $f_{\btheta}(\cdot)$ and its corresponded $f_{\btheta_{q}}(\cdot)$, we have
	\begin{equation}\label{eq:conditional expectation estimation}
		\small
		\begin{aligned}
			\sup_{z_{1}, z_{2}} & \left(\cL_{kz_{1}}(f_{\btheta}) - \cL_{kz_{2}}(f_{\btheta})\right) \leq \sup_{z_{1}, z_{2}}\left[\left(\cL_{kz_{1}}(f_{\btheta}) - \cL_{kz_{1}}(f_{\btheta_{q}})\right) + \left(\cL_{kz_{2}}(f_{\btheta_{q}}) -  \cL_{kz_{2}}(f_{\btheta})\right)\right] \\
			& + \sup_{z_{1}, z_{2}}\left(\cL_{kz_{1}}(f_{\btheta_{q}}) - \cL_{kz_{2}}(f_{\btheta_{q}})\right) \\
			& \overset{a}{\leq} 2\epsilon + M\left(\log{\left(\frac{2}{K_{z}\delta}\right)} + N\left(\cF(\Theta), \epsilon, \|\cdot\|_{L_{\infty}}\right)\right)\sup_{z_{1}, z_{2}}\left(\sqrt{\frac{1}{2n_{kz_{1}}}} + \sqrt{\frac{1}{2n_{kz_{2}}}}\right)\\ & + \sup_{z_{1},z_{2}}\left(\hat{\cL}_{kz_{1}}(f_{\btheta_{q}}) - \hat{\cL}_{kz_{2}}(f_{\btheta_{q}})\right)
		\end{aligned}
	\end{equation}
	holds with probability at least $1 - \delta$ for any $\epsilon > 0$. Here the inequality $a$ is due to the definition of $L_{\infty}$-norm on $\cF(\Theta)$.  
	\par
	On the other hand, as 
	\begin{equation}\label{eq:CV estimation}
		\small
		\begin{aligned}
			\mathrm{CSV}(f_{\btheta}) & = \sum\limits_{k = 1}^{K_{y}}\sup_{z_{1}, z_{2}}\left(\cL_{kz_{1}}(f_{\btheta}) - \cL_{kz_{2}}(f_{\btheta})\right)P(Y = k), 
		\end{aligned}
	\end{equation}
	We estimate the  $P(Y = k)$ with its empirical counterpart $n_{k} = \sum_{z=1}^{K_{z}}n_{kz} / n$. For bounded sub-Gaussian variable $\textbf{1}_{\{Y = k\}}$, we have   
	\begin{equation}
		\small
		\mE\left[\textbf{1}_{\{Y=k\}}\right] - \frac{1}{n}\sum\limits_{i=1}^{n}\textbf{1}_{\{y_{i}=k\}} = P(Y = k) - \hat{p}_{k} \leq \sqrt{\frac{\log{(1/\delta)}}{2n}}
	\end{equation}
	holds with probability at least $1 - \delta$. Plugging this into \eqref{eq:CV estimation} and combining \eqref{eq:conditional expectation estimation} we get 
	\begin{equation}\label{eq:CSV bound}
		\small
		\begin{aligned}
			& \mathrm{CSV}(f_{\btheta}) \leq \frac{1}{n}\sum\limits_{k = 1}^{K_{y}}\sup_{z_{1}, z_{2}}\left(\cL_{kz_{1}}(f_{\btheta}) - \cL_{kz_{2}}(f_{\btheta})\right)n_{k} + K_{y}\sqrt{\frac{\log{(2K_{y}/\delta)}}{2n}} \\
			& \leq \sum\limits_{k = 1}^{K_{y}}\sup_{z_{1}, z_{2}}\left(\hat{\cL}_{kz_{1}}(f_{\btheta}) - \hat{\cL}_{kz_{2}}(f_{\btheta})\right)\hat{p}_{k} + K_{y}\sqrt{\frac{\log{(2K_{y}/\delta)}}{2n}} \\
			& + \inf_{\epsilon}\left\{2\epsilon + M\left(\log{\left(\frac{2}{K_{z}\delta}\right)} + N\left(\cF(\Theta), \epsilon, \|\cdot\|_{L_{\infty}}\right)\right)\right\}\sum\limits_{k = 1}^{K_{y}}\sup_{z_{1},z_{2}}\left(\sqrt{\frac{1}{2n_{kz_{1}}}} + \sqrt{\frac{1}{2n_{kz_{2}}}}\right)\hat{p}_{k}
		\end{aligned}
	\end{equation}
	holds with probability at least $1 - \delta$ due to the definition of $\hat{p}_{k}$. Then suppose $\inf_{k\in[K_{y}],z\in[K_{z}]} n_{kz} / n_{k} \geq \alpha$, we have  
	\begin{equation}
		\small
		\begin{aligned}
		\sum\limits_{k = 1}^{K_{y}}\sup_{z_{1},z_{2}}\left(\sqrt{\frac{1}{2n_{kz_{1}}}} + \sqrt{\frac{1}{2n_{kz_{2}}}}\right)\hat{p}_{k} & \leq \frac{1}{n}\sum_{k\in[K_{y}]}\frac{\sqrt{2}n_{k}}{\min_{k\in[K_{y}], z\in [K_{z}]}\sqrt{n_{kz}}} \\
		& \leq \sqrt{\frac{2}{\alpha}}\sum_{k\in[K_{y}]}\frac{\sqrt{n_{k}}}{n} \\
	 	& \leq \sqrt{\frac{2K_{y}}{\alpha n}},
		\end{aligned}
	\end{equation}
	where the last inequality is due to the Schwarz's inequality. Combining this with \eqref{eq:CSV bound}, we get 
	\begin{equation}
		\small
		\mathrm{CSV}(f_{\btheta}) \leq \widehat{\rm{CSV}}(f_{\btheta}) + \cO\left(\frac{N\left(\cF(\Theta), \epsilon, \|\cdot\|_{L_{\infty}}\right) + \log{(1/\delta)}}{\sqrt{n}}\right).
	\end{equation}
	That completes our proof. 
\end{proof}
\subsection{Proofs in Section \ref{sec:estimating without spurious attributes}}\label{app:proof of sec:estimating without spurious attributes}
In this section we aim at proving that the \eqref{eq:quntile estimator} is a sharp estimator to the CSV when we know a lower bound $c$ for $\pi_{kz}$. Note that our problem is equivalent to estimate $\sup_{P_{k}}\mE_{P_{k}}[V] - \inf_{P_{k}}\mE_{P_{k}}[V]$ via $n$ samples $\{V_{i}\}$ drawn from mixture distribution $P = \sum_{k\in[K]}\pi_{k}P_{k}$ with known $\pi_{k} > c$. W.o.l.g., suppose $\mE_{P_{1}}[V] \leq \mE_{P_{2}}[V] \leq, \dots, \leq \mE_{P_{K}}[V]$, then we show $\mE_{P}[V\mid V\geq q_{P}(1-c)] - \mE_{P}[V\mid V\leq q_{P}(c)]$ is a upper bound to $\mE_{P_{K}}[V] - \mE_{P_{1}}[V]$ and the bound is sharp when $K\geq 3$. 
\begin{proposition}
	For $k=1,\cdots, K$ and $\pi_{k}\geq c$, suppose $P_{k}$ are absolutely continuous w.r.t. Lebesgue measure, then we have 
	\begin{equation}\label{eq: upper bound CSV}
		\small
		\mE_{P}[V\mid V\geq q_{P}(1 - c)] - \mE_{P}[V\mid V\leq q_{P}(c)] \geq \mE_{P_{K}}[V] - \mE_{P_{1}}[V],
	\end{equation}
	and the equality can be taken form some $P_{k}$ and $\pi_{k}$ ($k = 1,\dots, K$) if $K \geq 3$. 
\end{proposition}
\begin{proof}
	Let $p_{k}(v)$ be the density function of $P_{k}$. Then $p(\cdot) = \sum_{k=1}^{K}\pi_{k}p_{k}(\cdot)$ is the density function of $P$. One can verify that 
	\begin{equation}
		\small
		\begin{aligned}
			\left(v - q_{P}(1 - c)\right)\left(\frac{p(v)}{c}\textbf{1}_{\{v \geq q_{P}(1 - c)\}} - p_{K}(v)\right) \geq 0 
		\end{aligned}
	\end{equation}
	for any $v$ since $p_{2}(v) \geq 0$ and $0 < c \leq \pi_{K}$. Thus taking integral w.r.t. $v$ we get 
	\begin{equation}
		\small
		\mE_{P}[V\mid V\geq q_{P}(1 - c)] - \mE_{P_{K}}[V] = \int_{\bbR} \left(v - q_{P}(1 - c)\right)\left(\frac{p(v)}{c}\textbf{1}_{\{v \geq q_{P}(1 - c)\}} - p_{K}(v)\right)dv \geq 0. 
	\end{equation} 
	We can apply the similar argument to prove $\mE_{P}[V\mid V\leq q_{P}(c)] \leq \mE_{P_{1}}[V]$. Combining the two inequalities implies \eqref{eq: upper bound CSV}. 
	
	On the other hand, if $K \geq 3$, we take 
	\begin{equation}
		\small
		\begin{aligned}
			\pi_{1} & = \pi_{K} = c; \\
			\pi_{2} & =, \dots, = \pi_{K-1} = \frac{1 - 2c}{K-2},
		\end{aligned}
	\end{equation}
	and 
	\begin{equation}
		\small
		\begin{aligned}
			p_{1}(v) &= \frac{p(v)}{c}\textbf{1}_{\{v\leq q_{P}(c)\}}; \\
			p_{K}(v) &= \frac{p(v)}{c} \textbf{1}_{\{v\geq q_{P}(1-c)\}}; \\
			p_{2}(v) &= ,\cdots, = p_{K-1}(V) = \frac{p(v)}{1 - 2c}\textbf{1}_{\{q_{P}(c)\leq v \leq q_{P}(1-c)\}}.
		\end{aligned}
	\end{equation}
	Then it is easy to verify that
	\begin{equation}
		\small
		\mE_{P_{K}}[V\mid V\geq q_{P}(1 - c)] - \mE_{P}[V\mid V\leq q_{P}(c)] = \mE_{P_{K}}[V] - \mE_{P_{1}}[V]
	\end{equation}
	under this distribution.
\end{proof} 
According to this proposition, we can use the quintile conditional expectation to estimate the proposed CSV as we did in the main body of this paper. 
\section{Solving the Proposed Minimax Problem \eqref{eq:minimax problem}}\label{app: solving minimax problem}
In this section, we provide the convergence of the proposed Algorithm \ref{alg:sgd} to solve \eqref{eq:minimax problem}. We illustrate it in the regime of regularize training with $\widehat{\rm{CSV}}(f_{\btheta})$ i.e., $\bF^{k}(\btheta)$ defined in Section \ref{sec:regularize training with cv}. Then we have $m = K_{z}^{2}$ in this regime. Let us define 
\begin{equation}
	\small
	\begin{aligned}
		\Phi^{k}_{\rho}(\btheta)  &= R_{emp}(f_{\btheta}, P) + \lambda\max_{\bu\in\Delta_{K_{z}^{2}}}\bu^{\top}\bF^{k}(\btheta) - \rho\lambda\bu^{\top}\log{\left(K_{z}^{2}\bu\right)} = \max_{\bu\in\Delta_{K_{z}^{2}}}\phi^{k}_{\rho}(\btheta, \bu); \\
		\hat{\Phi}^{k}_{\rho}(\btheta)  &= R_{emp}(f_{\btheta}, P) + \lambda\max_{\bu\in\Delta_{K_{z}^{2}}}\bu^{\top}\hat{\bF}^{k}(\btheta) - \rho\lambda\bu^{\top}\log{\left(K_{z}^{2}\bu\right)} = \max_{\bu\in\Delta_{K_{z}^{2}}}\hat{\phi}^{k}_{\rho}(\btheta, \bu).
	\end{aligned}
\end{equation} 
We have the following lemma to state the some continuities.
\begin{lemma}\label{lem:continuty}
	Under Assumption \ref{ass:feature space}-\ref{ass:Lip}, we have the following conclusions
	\begin{enumerate}
		\item For $\phi_{\rho}^{k}(\btheta, \bu)$ with any $\rho$ and $k$ we have 
		\begin{equation}
			\small
			\begin{aligned}
				& \left\|\nabla_{\btheta}\phi_{\rho}^{k}(\btheta_{1}, \bu) - \nabla_{\btheta}\phi_{\rho}^{k}(\btheta_{2}, \bu)\right\| \leq (1 + 2\lambda K_{z}L_{1})\|\btheta_{1} - \btheta_{2}\| = L_{11}\|\btheta_{1} - \btheta_{2}\|; \\
				& \left\|\nabla_{\btheta}\phi_{\rho}^{k}(\btheta, \bu_{1}) - \nabla_{\btheta}\phi_{\rho}^{k}(\btheta, \bu_{2})\right\| \leq 2\lambda K_{z}L_{0}\|\bu_{1} - \bu_{2}\| = L_{12}\|\bu_{1} - \bu_{2}\|; \\
				& \left\|\nabla_{\bu}\phi_{\rho}^{k}(\btheta_{1}, \bu) - \nabla_{\bu}\phi_{\rho}^{k}(\btheta_{2}, \bu)\right\| \leq 2\lambda K_{z}L_{0}\|\btheta_{1} - \btheta_{2}\| = L_{12}\|\bu_{1} - \bu_{2}\|.
			\end{aligned}
		\end{equation}
		\item Let $\bu^{*}_{k}(\btheta, \rho) = \arg\max_{\bu\in\Delta_{K_{z}^{2}}}\phi^{k}_{\rho}(\btheta, \bu)$ and $\hat{\bu}^{*}_{k}(\btheta, \rho) = \arg\max_{\bu\in\Delta_{K_{z}^{2}}}\hat{\phi}^{k}_{\rho}(\btheta, \bu)$ then 
		\begin{equation}
			\small
			\left\|\bu^{*}_{k}(\btheta, \rho) - \hat{\bu}^{*}_{k}(\btheta, \rho)\right\| \leq \frac{1}{\rho}\left\|\hat{\bF}(\btheta) - \bF(\btheta)\right\|.
		\end{equation}
		\item $\Phi_{\rho}^{k}(\btheta)$ is $L_{1} + \lambda\left(L_{11} + L_{12}^{2}/\rho\right)$-smoothness
	\end{enumerate}
\end{lemma}  	
\begin{proof}
	Let us proof the conclusions by order. For the first conclusion, we have 
	\begin{equation}
		\small
		\nabla_{\btheta}\phi^{k}_{\rho}(\btheta, \bu) = \nabla_{\btheta}R_{emp}(f_{\btheta}, P) + \lambda \nabla_{\btheta}\bF^{k}(\btheta)^{\top}\bu; \qquad \nabla_{\bu}\phi^{k}_{\rho}(\btheta, \bu) = \lambda \bF^{k}(\btheta) - \rho\lambda\left(\log{\left(K_{z}^{2}\bu\right)} + \textbf{e} \right),
	\end{equation}
	where $\textbf{e} = (1, \cdots, 1)$. Thus, by Schwarz's inequality, one can verify 
	\begin{equation}
		\small
		\begin{aligned}
			& \left\|\nabla_{\btheta}\phi_{\rho}^{k}(\btheta_{1}, \bu) - \nabla_{\btheta}\phi_{\rho}^{k}(\btheta_{2}, \bu)\right\|  \leq \|\nabla_{\btheta}R_{emp}(f_{\btheta_{1}}, P) - \nabla_{\btheta}R_{emp}(f_{\btheta_{2}}, P)\| \\
			& + \lambda\left\|\bu^{\top}\left(\nabla_{\btheta}\bF^{k}(\btheta_{1}) - \nabla_{\btheta}\bF^{k}(\btheta_{2})\right)\right\| \\
			& \leq \lambda \left(\sum\limits_{i\in[K_{z}^{2}],j\in[K_{z}^{2}]}\bu(i)\bu(j)\sup_{(z_{1}, z_{2})}\left\|\left(\nabla\hat{\cL}_{kz_{1}}(f_{\btheta_{1}}) -  \nabla\hat{\cL}_{kz_{2}}(f_{\btheta_{1}})\right)\right. - \left.\left(\nabla\hat{\cL}_{kz_{1}}(f_{\btheta_{2}}) -  \nabla\hat{\cL}_{kz_{2}}(f_{\btheta_{2}})\right)\right\|^{2}\right)^{\frac{1}{2}} \\
			& + L_{1}\|\btheta_{1} - \btheta_{2}\| \\
			& \leq (1 + 2\lambda K_{z})L_{1}\|\btheta_{1} - \btheta_{2}\|,
		\end{aligned}
	\end{equation}
	and
	\begin{equation}
		\small
		\begin{aligned}
			\left\|\nabla_{\btheta}\phi_{\rho}^{k}(\btheta, \bu_{1}) - \nabla_{\btheta}\phi_{\rho}^{k}(\btheta, \bu_{2})\right\| &\leq \lambda\|\bu_{1} - \bu_{2}\|\left\|\nabla_{\btheta}\bF^{k}(\btheta)\right\| \leq 2\lambda K_{z}L_{0}\|\bu_{1} - \bu_{2}\|,
		\end{aligned}
	\end{equation}
	and 
	\begin{equation}
		\small
		\begin{aligned}
			\left\|\nabla_{\bu}\phi_{\rho}^{k}(\btheta_{1}, \bu) - \nabla_{\bu}\phi_{\rho}^{k}(\btheta_{2}, \bu)\right\| \leq  \lambda\left\|\bF^{k}(\btheta_{1}) - \bF^{k}(\btheta_{2})\right\| \leq 2\lambda K_{z}L_{0}\|\btheta_{1} - \btheta_{2}\|.
		\end{aligned}
	\end{equation}
	Thus we complete the proof to the first conclusion. 
	\par
	For the second conclusion, by the Lagrange's multiplier method or Theorem in \citep{yi2021reweighting}, we have the unique closed-form solution of $\bu^{*}_{k}(\btheta, \rho)\in\Delta_{K_{z}^{2}}$ that 
	\begin{equation}
		\small
		\begin{aligned}
			\bu^{*}_{k}(\btheta, \rho) = \frac{\exp\left(\frac{1}{\rho}\bF^{k}(\btheta)(j)\right)}{\sum_{j=1}^{K_{z}^{2}}\exp\left(\frac{1}{\rho}\bF^{k}(\btheta)(j)\right)} = \mathrm{Softmax}\left(\frac{\bF^{k}(\btheta)}{\rho}\right); \\
			\hat{\bu}^{*}_{k}(\btheta, \rho) = \frac{\exp\left(\frac{1}{\rho}\bF^{k}(\btheta)(j)\right)}{\sum_{j=1}^{K_{z}^{2}}\exp\left(\frac{1}{\rho}\hat{\bF}^{k}(\btheta)(j)\right)} = \mathrm{Softmax}\left(\frac{\hat{\bF}^{k}(\btheta)}{\rho}\right).
		\end{aligned}
	\end{equation}
	On the other hand, due to $\bu\in\Delta_{K_{z}^{2}}$ we have 
	\begin{equation}
		\small
		\nabla_{\bu\bu}^{2}\phi^{k}_{\rho}(\btheta, \bu) = -\rho\lambda\diag\left(\frac{1}{\bu(i)},\cdots, \frac{1}{\bu(K_{z}^{2})}\right) \preceq -\rho\lambda\bI,  
	\end{equation}
	where $\bA\succeq \bB$ means that $\bA - \bB$ is a semi-positive definite matrix and $\bI$ is the identity matrix. The similar conclusion holds for $\hat{\phi}^{k}_{\rho}(\btheta, \bu)$. Thus both $\phi^{k}_{\rho}(\btheta, \bu)$ and $\hat{\phi}^{k}_{\rho}(\btheta, \bu)$ are $\rho$-strongly concave w.r.t. $\bu$. Then 
	\begin{equation}
		\small
		\begin{aligned}
			\phi^{k}_{\rho}(\btheta, \bu^{*}_{k}(\btheta, \rho)) & \leq \phi^{k}_{\rho}(\btheta, \hat{\bu}^{*}_{k}(\btheta, \rho)) + \left\langle\nabla_{\bu}\phi^{k}_{\rho}(\btheta, \hat{\bu}^{*}_{k}(\btheta, \rho)), \bu^{*}_{k}(\btheta, \rho) - \hat{\bu}^{*}_{k}(\btheta, \rho)\right\rangle \\
			& - \frac{\rho\lambda}{2}\left\|\bu^{*}_{k}(\btheta, \rho) - \hat{\bu}^{*}_{k}(\btheta, \rho)\right\|^{2};\\
			\phi^{k}_{\rho}(\btheta, \hat{\bu}^{*}_{k}(\btheta, \rho)) & \leq \phi^{k}_{\rho}(\btheta, \bu^{*}_{k}(\btheta, \rho)) + \left\langle\nabla_{\bu}\phi^{k}_{\rho}(\btheta, \bu^{*}_{k}(\btheta, \rho)), \hat{\bu}^{*}_{k}(\btheta, \rho) - \bu^{*}_{k}(\btheta, \rho)\right\rangle \\
			& - \frac{\rho\lambda}{2}\left\|\bu^{*}_{k}(\btheta, \rho) - \hat{\bu}^{*}_{k}(\btheta, \rho)\right\|^{2}.
		\end{aligned}
	\end{equation}
	Plugging the two above inequalities, we have that 
	\begin{equation}
		\small
		\left\langle\nabla_{\bu}\phi^{k}_{\rho}(\btheta, \hat{\bu}^{*}_{k}(\btheta, \rho)), \bu^{*}_{k}(\btheta, \rho) - \hat{\bu}^{*}_{k}(\btheta, \rho)\right\rangle \geq \rho\lambda\left\|\bu^{*}_{k}(\btheta, \rho) - \hat{\bu}^{*}_{k}(\btheta, \rho)\right\|^{2} 
	\end{equation}
	due to the $\left\langle\nabla_{\bu}\phi^{k}_{\rho}(\btheta, \hat{\bu}^{*}_{k}(\btheta, \rho)), \bu^{*}_{k}(\btheta, \rho) - \hat{\bu}^{*}_{k}(\btheta, \rho)\right\rangle \leq 0$. On the other hand, as  
	\begin{equation}
		\small
		\left\langle\nabla_{\bu}\hat{\phi}^{k}_{\rho}(\btheta, \hat{\bu}^{*}_{k}(\btheta, \rho)), \bu^{*}_{k}(\btheta, \rho) - \hat{\bu}^{*}_{k}(\btheta, \rho)\right\rangle \leq 0.
	\end{equation}
	Plugging this into the above inequality, we get 
	\begin{equation}
		\small
		\begin{aligned}
			\rho\lambda\left\|\bu^{*}_{k}(\btheta, \rho) - \hat{\bu}^{*}_{k}(\btheta, \rho)\right\|^{2} & \leq \left\langle\nabla_{\bu}\phi^{k}_{\rho}(\btheta, \hat{\bu}^{*}_{k}(\btheta, \rho)) - \nabla_{\bu}\hat{\phi}^{k}_{\rho}(\btheta, \hat{\bu}^{*}_{k}(\btheta, \rho)), \bu^{*}_{k}(\btheta, \rho) - \hat{\bu}^{*}_{k}(\btheta, \rho)\right\rangle \\
			& = \lambda\left\langle\bF^{k}(\btheta) - \hat{\bF}^{k}(\btheta), \bu^{*}_{k}(\btheta, \rho) - \hat{\bu}^{*}_{k}(\btheta, \rho)\right\rangle \\
			& \leq \lambda\left\|\bF^{k}(\btheta) - \hat{\bF}^{k}(\btheta)\right\|\left\|\bu^{*}_{k}(\btheta, \rho) - \hat{\bu}^{*}_{k}(\btheta, \rho)\right\|. 
		\end{aligned}
	\end{equation}	
	Thus the conclusion is proofed. 
	\par
	Finally, we prove the third conclusion. Similar to the proof of the second conclusion, we have 
	\begin{equation}
		\small
		\left\|\bu^{*}_{k}(\btheta_{1}, \rho) - \bu^{*}_{k}(\btheta_{2}, \rho)\right\| \leq \frac{1}{\rho}\|\btheta_{1} - \btheta_{2}\|. 
	\end{equation}
	Since $\Delta_{K_{z}^{2}}$ is convex, bounded, and $\bu^{*}_{k}(\btheta, \rho)$ is unique for any $\btheta$, by Danskin's Theorem \citep{bernhard1995theorem}, we have 
	\begin{equation}
		\small
		\begin{aligned}
			\left\|\nabla\Phi_{\rho}^{k}(\btheta_{1}) - \nabla\Phi_{\rho}^{k}(\btheta_{2})\right\|
			& \leq \left\|\nabla_{\btheta}R_{emp}(f_{\btheta_{1}}, P) - \nabla_{\btheta}R_{emp}(f_{\btheta_{2}}, P) \right\| \\
			& + \lambda\left\|\nabla_{\btheta}\bF^{k}(\btheta_{1})^{\top}\bu^{*}_{k}(\btheta_{1}, \rho) - \nabla_{\btheta}\bF^{k}(\btheta_{2})^{\top}\bu^{*}_{k}(\btheta_{2}, \rho)\right\| \\
			& \leq L_{1}\left\|\btheta_{1} - \btheta_{2}\right\| + \lambda\left\|\nabla_{\btheta}\bF^{k}(\btheta_{1})^{\top}\left(\bu^{*}_{k}(\btheta_{1}, \rho) - \bu^{*}_{k}(\btheta_{2}, \rho)\right)\right\| \\
			& + \lambda\left\|\bu^{*}_{k}(\btheta_{2}, \rho)^{\top}\left(\nabla_{\btheta}\bF^{k}(\btheta_{1}) - \nabla_{\btheta}\bF^{k}(\btheta_{2})\right)\right\| \\
			& \leq L_{1}\left\|\btheta_{1} - \btheta_{2}\right\| + \frac{\lambda L_{12}^{2}}{\rho}\left\|\btheta_{1} - \btheta_{2}\right\| + \lambda L_{11}\left\|\btheta_{1} - \btheta_{2}\right\| \\
			& = \left(L_{1} + \frac{\lambda L_{12}^{2}}{\rho} + \lambda L_{11}\right)\left\|\btheta_{1} - \btheta_{2}\right\|,
		\end{aligned}
	\end{equation}
	which implies our conclusion.   
\end{proof}
\par
We present the following lemma to state the descent property of the obtained iterates via Algorithm \ref{alg:sgd}. Let us define 
\begin{equation}
	\small
	\hat{\phi}^{k}(\btheta, \bu) = R_{emp}(f_{\btheta}, P) + \lambda\bu^{\top}\bF^{k}
\end{equation}
As we have assume that unbiased estimators $\hat{R}_{emp}(f_{\btheta}, P)$ and $\hat{\bF^{k}}(\btheta)$ have bounded variance, then according to 
\begin{equation}
	\small
	\begin{aligned}
	\mE\left[\left(\hat{\phi}^{k}(\btheta, \bu) - \phi^{k}(\btheta, \bu)\right)^{2}\right] & \leq 2\mE\left[\left(\hat{R}_{emp}(f_{\btheta}, P) - R_{emp}(f_{\btheta}, P)^{2}\right)\right] \\
	& + 2\lambda^{2}\mE\left[\left\|\hat{\bF}^{k}(\btheta) - \bF^{k}(\btheta)\right\|^{2}\right],
	\end{aligned}
\end{equation}
w.o.l.g. we assume that 
\begin{equation}
	\small
	\max\left\{\mE\left[\left(\hat{\phi}^{k}(\btheta, \bu) - \phi^{k}(\btheta, \bu)\right)^{2}\right], \mE\left[\left\|\hat{\bF}^{k}(\btheta) - \bF^{k}(\btheta)\right\|^{2}\right]\right\} \leq \sigma^{2}.
\end{equation}
\begin{lemma}\label{lem:descent lemma}
	Let $L_{1} + \lambda\left(L_{11} + L_{12}^{2}/\rho\right) = \tilde{L}$, if we have the estimation such that $\mE[\hat{\phi}^{k}(\btheta, \bu)] = \phi^{k}(\btheta, \bu)$,  $\mE[(\hat{\phi}^{k}(\btheta, \bu) - \phi^{k}(\btheta, \bu))^{2}]\leq \sigma^{2}$ then
	\begin{equation}\label{eq:descent eq}
		\small
		\begin{aligned}
			\mE\left[\sum\limits_{k=1}^{K_{y}}\hat{p}_{k}\Phi_{\rho}^{k}(\btheta(t + 1))\right] & \leq \mE\left[\sum\limits_{k=1}^{K_{y}}\hat{p}_{k}\Phi_{\rho}^{k}(\btheta(t))\right] - \left(\frac{\eta_{\btheta}}{2} - \tilde{L}\eta_{\btheta}^{2}\right)\mE\left[\left\|\sum\limits_{k=1}^{K_{y}}\hat{p}_{k}\nabla\Phi_{\rho}^{k}(\btheta(t))\right\|^{2}\right] \\
			& + \sum\limits_{k=1}^{K_{y}}\left(\frac{\tilde{L}K_{y}L_{12}^{2}\hat{p}_{k}\eta_{\btheta}^{2}}{\rho} + \frac{\left(\sigma^{2} + L_{11}\right)L_{12}^{2}\eta_{\btheta}}{2\rho^{2}}\right)\hat{p}_{k}\mE\left[\left\|\bF^{k}(\btheta(t)) - \bF_{t}^{k}\right\|^{2}\right] \\
			& + \frac{\tilde{L}\eta_{\btheta}^{2}\sigma^{2}}{2}.
		\end{aligned}
	\end{equation}
\end{lemma}
\begin{proof}
	Due to the $\tilde{L}$-smoothness of $\Phi_{\rho}^{k}(\cdot)$ for any $k$ and $\rho$ we have 
	\begin{equation}\label{eq:Lip expansion}
		\small
		\begin{aligned}
			\sum\limits_{k=1}^{K_{y}}\hat{p}_{k}\Phi_{\rho}^{k}(\btheta(t + 1)) & \leq \sum\limits_{k=1}^{K_{y}}\hat{p}_{k}\Phi_{\rho}^{k}(\btheta(t)) +  \left\langle\sum\limits_{k=1}^{K_{y}}\hat{p}_{k}\nabla\Phi_{\rho}^{k}(\btheta(t)), \btheta(t + 1) - \btheta(t)\right\rangle + \frac{\tilde{L}}{2}\|\btheta(t + 1) - \btheta(t)\|^{2} \\
			& = \sum\limits_{k=1}^{K_{y}}\hat{p}_{k}\Phi_{\rho}^{k}(\btheta(t)) - \eta_{\btheta}\left\langle\sum\limits_{k=1}^{K_{y}}\hat{p}_{k}\nabla\Phi_{\rho}^{k}(\btheta(t)), \sum\limits_{k = 1}^{K_{y}}\hat{p}_{k}\nabla_{\btheta}\hat{\phi}_{\rho}^{k}(\btheta(t), \hat{\bu}_{k}^{*}(\btheta(t), \rho))\right\rangle \\
			& + \frac{\tilde{L}\eta_{\btheta}^{2}}{2}\left\|\sum\limits_{k = 1}^{K_{y}}\hat{p}_{k}\nabla_{\btheta}\hat{\phi}_{\rho}^{k}(\btheta(t), \hat{\bu}_{k}^{*}(\btheta(t), \rho))\right\|^{2} \\
			& \leq \sum\limits_{k=1}^{K_{y}}\hat{p}_{k}\Phi_{\rho}^{k}(\btheta(t)) - \eta_{\btheta}\left\|\sum\limits_{k=1}^{K_{y}}\hat{p}_{k}\nabla\Phi_{\rho}^{k}(\btheta(t))\right\|^{2} \\
			& + \eta_{\btheta}\left\langle\sum\limits_{k=1}^{K_{y}}\hat{p}_{k}\nabla\Phi_{\rho}^{k}(\btheta(t)), \sum\limits_{k=1}^{K_{y}}\hat{p}_{k}\left(\nabla\Phi_{\rho}^{k}(\btheta(t)) - \nabla_{\btheta}\hat{\phi}_{\rho}^{k}(\btheta(t), \bu_{k}^{*}(\btheta(t), \rho))\right)\right\rangle \\
			& + \eta_{\btheta}\left\langle\sum\limits_{k=1}^{K_{y}}\hat{p}_{k}\nabla\Phi_{\rho}^{k}(\btheta(t)), \sum\limits_{k=1}^{K_{y}}\hat{p}_{k}\left(\nabla_{\btheta}\hat{\phi}_{\rho}^{k}(\btheta(t), \bu_{k}^{*}(\btheta(t), \rho)) - \nabla_{\btheta}\hat{\phi}_{\rho}^{k}(\btheta(t), \bu_{k}(t))\right)\right\rangle \\
			& + \frac{\tilde{L}\eta_{\btheta}^{2}}{2}\left\|\sum\limits_{k = 1}^{K_{y}}\hat{p}_{k}\nabla_{\btheta}\hat{\phi}_{\rho}^{k}(\btheta(t), \bu_{k}(t))\right\|^{2}. 
		\end{aligned}
	\end{equation}
	On the other hand, using the fact that $\mE\left[\hat{\phi}_{\rho}^{k}(\btheta, \bu)\right] = \phi_{\rho}^{k}(\btheta, \bu)$, $\mE\left[\left(\hat{\phi}_{\rho}^{k}(\btheta, \bu) - \phi_{\rho}^{k}(\btheta, \bu)\right)^{2}\right] \leq \sigma^{2}$, and Young's inequality
	\begin{equation}
		\label{eq:theta gap 1}
		\small
		\begin{aligned}
			& \mE\left[\left\|\sum\limits_{k = 1}^{K_{y}}\hat{p}_{k}\nabla_{\btheta}\hat{\phi}_{\rho}^{k}(\btheta(t), \bu_{k}(t)\right\|^{2}\right] \\
			& = \mE\left[\left\|\sum\limits_{k = 1}^{K_{y}}\hat{p}_{k}\left(\nabla_{\btheta}\hat{\phi}_{\rho}^{k}(\btheta(t), \bu_{k}(t) - \nabla_{\btheta}\phi_{\rho}^{k}(\btheta(t), \bu_{k}(t))\right) + \sum\limits_{k = 1}^{K_{y}}\hat{p}_{k}\nabla_{\btheta}\phi_{\rho}^{k}(\btheta(t), \bu_{k}(t))\right\|^{2}\right] \\
			& \leq \sigma^{2} + \mE\left[\left\|\sum\limits_{k = 1}^{K_{y}}\hat{p}_{k}\nabla_{\btheta}\phi_{\rho}^{k}(\btheta(t), \bu_{k}(t))\right\|^{2}\right] \\
			& \leq \sigma^{2} + 2\mE\left[\left\|\sum\limits_{k = 1}^{K_{y}}\hat{p}_{k}\nabla\Phi_{\rho}^{k}(\btheta(t))\right\|^{2}\right] +  2\mE\left[\left\|\sum\limits_{k=1}^{K_{y}}\hat{p}_{k}\left(\nabla\Phi_{\rho}^{k}(\btheta(t)) - \nabla_{\btheta}\phi_{\rho}^{k}(\btheta(t), \bu_{k}(t))\right)\right\|^{2}\right]. 
		\end{aligned}
	\end{equation}
	Due to Danskin's theorem and (i), (ii) in Lemma \ref{lem:continuty} we have  
	\begin{equation}
		\label{eq:theta gap 2}
		\small
		\begin{aligned}
			\mE\left[\left\|\sum\limits_{k=1}^{K_{y}}\hat{p}_{k}\left(\nabla\Phi_{\rho}^{k}(\btheta(t)) - \nabla_{\btheta}\phi_{\rho}^{k}(\btheta(t), \bu_{k}(t))\right)\right\|^{2}\right] & \leq \mE\left[\left(\sum\limits_{k=1}^{K_{y}}\hat{p}_{k}L_{12}\left\|\bu^{*}_{k}(\btheta(t), \rho) - \bu_{k}(t)\right\|\right)^{2}\right] \\
			& \leq \mE\left[\left(\sum\limits_{k=1}^{K_{y}}\frac{\hat{p}_{k}L_{12}}{\rho}\left\|\bF^{k}(\btheta(t)) - \bF_{t}^{k}\right\|\right)^{2}\right] \\
			& \leq \frac{K_{y}L_{12}^{2}}{\rho^{2}}\sum\limits_{k=1}^{K_{y}}\hat{p}_{k}^{2}\mE\left[\left\|\bF^{k}(\btheta(t)) - \bF_{t}^{k}\right\|^{2}\right]. 
		\end{aligned}
	\end{equation}
	Finally, we have 
	\begin{equation}
		\small
		\begin{aligned}
			& \mE\left[\left\langle\sum\limits_{k=1}^{K_{y}}\hat{p}_{k}\nabla\Phi_{\rho}^{k}(\btheta(t)), \sum\limits_{k=1}^{K_{y}}\hat{p}_{k}\left(\nabla_{\btheta}\hat{\phi}_{\rho}^{k}(\btheta(t), \bu_{k}^{*}(\btheta(t), \rho)) - \nabla_{\btheta}\hat{\phi}_{\rho}^{k}(\btheta(t), \bu_{k}(t))\right)\right\rangle\right] \\
			& = \sum\limits_{k=1}^{K_{y}}\hat{p}_{k}\mE\left[\left\langle\sum\limits_{k=1}^{K_{y}}\hat{p}_{k}\nabla\Phi_{\rho}^{k}(\btheta(t)), \left(\nabla_{\btheta}\hat{\phi}_{\rho}^{k}(\btheta(t), \bu_{k}^{*}(\btheta(t), \rho)) - \nabla_{\btheta}\hat{\phi}_{\rho}^{k}(\btheta(t), \bu_{k}(t))\right)\right\rangle\right] \\
			& \leq \frac{1}{2}\mE\left[\left\|\sum\limits_{k=1}^{K_{y}}\hat{p}_{k}\nabla\Phi_{\rho}^{k}(\btheta(t))\right\|^{2}\right] + \frac{1}{2}\sum\limits_{k=1}^{K_{y}}\hat{p}_{k}\mE\left[\sum\limits_{k=1}^{K_{y}}\left\|\nabla_{\btheta}\hat{\phi}_{\rho}^{k}(\btheta(t), \bu_{k}^{*}(\btheta(t), \rho)) - \nabla_{\btheta}\hat{\phi}_{\rho}^{k}(\btheta(t), \bu_{k}(t))\right\|^{2}\right] \\
			& \leq \frac{1}{2}\mE\left[\left\|\sum\limits_{k=1}^{K_{y}}\hat{p}_{k}\nabla\Phi_{\rho}^{k}(\btheta(t))\right\|^{2}\right] + \frac{1}{2}\sum\limits_{k=1}^{K_{y}}\hat{p}_{k}\mE\left[\left\|\nabla\hat{\bF}^{k}(\btheta(t))\right\|^{2}\left\|\bu_{k}(t) - \bu_{k}^{*}(\btheta(t), \rho)\right\|^{2}\right] \\
			& \leq \frac{1}{2}\mE\left[\left\|\sum\limits_{k=1}^{K_{y}}\hat{p}_{k}\nabla\Phi_{\rho}^{k}(\btheta(t))\right\|^{2}\right] + \frac{\left(\sigma^{2} + L_{11}\right)L_{12}^{2}}{2\rho^{2}}\sum\limits_{k=1}^{K_{y}}\hat{p}_{k}\mE\left[\left\|\bF^{k}(\btheta(t)) - \bF_{t}^{k}\right\|^{2}\right].  
		\end{aligned}
	\end{equation}
	Plugging the three above inequalities into \eqref{eq:Lip expansion} and taking expectation to the both sides of the equality, we get 
	\begin{equation}
		\small
		\begin{aligned}
			& \mE\left[\sum\limits_{k=1}^{K_{y}}\hat{p}_{k}\Phi_{\rho}^{k}(\btheta(t + 1))\right]  \leq \mE\left[\sum\limits_{k=1}^{K_{y}}\hat{p}_{k}\Phi_{\rho}^{k}(\btheta(t))\right] - \left(\frac{\eta_{\btheta}}{2} - \tilde{L}\eta_{\btheta}^{2}\right)\mE\left[\left\|\sum\limits_{k=1}^{K_{y}}\hat{p}_{k}\nabla\Phi_{\rho}^{k}(\btheta(t))\right\|^{2}\right] \\
			& + \sum\limits_{k=1}^{K_{y}}\left(\frac{\tilde{L}K_{y}L_{12}^{2}\hat{p}_{k}\eta_{\btheta}^{2}}{\rho} + \frac{\left(\sigma^{2} + L_{11}\right)L_{12}^{2}\eta_{\btheta}}{2\rho^{2}}\right)\hat{p}_{k}\mE\left[\left\|\bF^{k}(\btheta(t)) - \bF_{t}^{k}\right\|^{2}\right] + \frac{\tilde{L}\eta_{\btheta}^{2}\sigma^{2}}{2}. 
		\end{aligned}
	\end{equation}
	This completes the proof of our theorem. 
\end{proof}
Then we proceed to the next lemma to characterize the dynamic of $\mE\left[\left\|\bF^{k}(\btheta(t)) - \bF_{t}^{k}\right\|^{2}\right]$. 
\begin{lemma}
	\label{lem:delta dynamic}
	For the $\bF_{t}^{k}$ defined in Algorithm \ref{alg:sgd}, and let $\delta^{k}(t) = \mE\left[\left\|\bF^{k}(\btheta(t)) - \bF_{t}^{k}\right\|^{2}\right]$, and $\delta(t) = \sum_{k=1}^{K_{y}}\hat{p}_{k}\delta^{k}(t)$, by choosing $\gamma \leq 2/3$, we have 
	\begin{equation}
		\small
		\begin{aligned}
			\delta(t + 1) & \leq \left(1 - \frac{\gamma}{2} + \frac{4\eta_{\btheta}^{2}L_{11}^{2}K_{y}L_{12}^{2}}{\gamma\rho^{2}}\right)\delta(t) + \left(2\gamma^{2} + \frac{2\eta_{\btheta}^{2}L_{11}^{2}}{\gamma}\right)\sigma^{2} \\
			& + \frac{4\eta_{\btheta}^{2}L_{11}^{2}}{\gamma}\mE\left[\left\|\sum\limits_{k = 1}^{K_{y}}\hat{p}_{k}\nabla\Phi_{\rho}^{k}(\btheta(t))\right\|^{2}\right]. 
		\end{aligned}
	\end{equation}
\end{lemma}
\begin{proof}
	W.o.l.g., we fix the $k$ during our proof. According to $\mE\left[\hat{\bF}^{k}(\btheta(t))\right] = \bF^{k}(\btheta(t))$ and $\mE\left[\left\|\hat{\bF}^{k}(\btheta(t)) - \bF^{k}(\btheta(t))\right\|^{2}\right]\leq \sigma^{2}$ we have 
	\begin{equation}
		\small
		\begin{aligned}
			\mE\left[\left\|\bF_{t + 1}^{k} - \bF^{k}(\btheta(t))\right\|^{2}\right] \leq (1 - \gamma)^{2}\delta^{k}(t) + \gamma^{2}\sigma^{2}\leq (1 - \gamma)\delta^{k}(t) + \gamma^{2}\sigma^{2}, 
		\end{aligned}
	\end{equation}
	where we use the fact $\gamma < 1$. Then due to the update rule of $\bF_{t}^{k}$, Young's inequality and the above inequality,  
	\begin{equation}
		\small
		\begin{aligned}
			& \delta^{k}(t + 1) = \mE\left[\left\|\bF_{t + 1}^{k} - \bF^{k}(\btheta(t)) + \bF^{k}(\btheta(t)) - \bF^{k}(\btheta(t + 1))\right\|^{2}\right] \\
			& \leq \left(1 + \frac{\gamma}{2 - 2\gamma}\right)\mE\left[\left\|\bF_{t + 1}^{k} - \bF^{k}(\btheta(t))\right\|^{2}\right] + \left(1 + \frac{2 - 2\gamma}{\gamma}\right)\mE\left[\left\|\bF^{k}(\btheta(t + 1)) - \bF^{k}(\btheta(t))\right\|^{2}\right] \\
			& \leq \left(1 - \frac{\gamma}{2}\right)\delta^{k}(t) + 2\gamma^{2}\sigma^{2} +  \frac{2}{\gamma}\mE\left[\left\|\bF^{k}(\btheta(t + 1)) - \bF^{k}(\btheta(t))\right\|^{2}\right].
		\end{aligned}
	\end{equation}
	On the other hand, due to the $L_{11}$-continuity of $\bF^{k}(\cdot)$ and \eqref{eq:theta gap 1}, \eqref{eq:theta gap 2} we see 
	\begin{equation}
		\small
		\begin{aligned}
			\mE\left[\left\|\bF^{k}(\btheta(t + 1)) - \bF^{k}(\btheta(t))\right\|^{2}\right] & \leq L_{11}^{2}\mE\left[\left\|\btheta(t + 1) - \btheta(t)\right\|^{2}\right] \\
			& = \eta_{\btheta}^{2}L_{11}^{2}\mE\left[\left\|\sum\limits_{k = 1}^{K_{y}}\hat{p}_{k}\nabla_{\btheta}\hat{\phi}_{\rho}^{k}(\btheta(t), \bu_{k}(t)\right\|^{2}\right] \\
			& \leq \eta_{\btheta}^{2}L_{11}^{2}\sigma^{2} + 2\eta_{\btheta}^{2}L_{11}^{2}\mE\left[\left\|\sum\limits_{k = 1}^{K_{y}}\hat{p}_{k}\nabla\Phi_{\rho}^{k}(\btheta(t))\right\|^{2}\right]\\
			& + \frac{2\eta_{\btheta}^{2}L_{11}^{2}K_{y}L_{12}^{2}}{\rho^{2}}\sum\limits_{k=1}^{K_{y}}\hat{p}_{k}^{2}\delta^{k}(t).
		\end{aligned}
	\end{equation}
	Plugging this into the above inequality and weighted summing over $k$ (by $\hat{p}_{k}$) we get 
	\begin{equation}
		\small
		\begin{aligned}
			\sum\limits_{k=1}^{K_{y}}\hat{p}_{k}\delta^{k}(t + 1) & \leq \left(1 - \frac{\gamma}{2} + \frac{4\eta_{\btheta}^{2}L_{11}^{2}K_{y}L_{12}^{2}}{\gamma\rho^{2}}\right)\sum\limits_{k=1}^{K_{y}}\hat{p}_{k}\delta^{k}(t) + \left(2\gamma^{2} + \frac{2\eta_{\btheta}^{2}L_{11}^{2}}{\gamma}\right)\sigma^{2} \\
			& + \frac{4\eta_{\btheta}^{2}L_{11}^{2}}{\gamma}\mE\left[\left\|\sum\limits_{k = 1}^{K_{y}}\hat{p}_{k}\nabla\Phi_{\rho}^{k}(\btheta(t))\right\|^{2}\right]
		\end{aligned}
	\end{equation}
	which completes the our proof. 
\end{proof}
\par
Now we are ready to state the convergence rate of the nonconvex-concave optimization problem. 
\convergencerate*
\begin{proof}
	First, note that $m = K_{z}^{2}$, for any $\btheta$,
	\begin{equation}\label{eq:phi bound}
		\small
		\Phi_{\rho}^{k}(\btheta) \leq M - \lambda\rho K_{z}\inf_{x}\{x\log{(K_{z}^{2}x)}\} = M + \frac{\lambda\rho}{K_{z}e}. 
	\end{equation}
	Due to the value of $\eta_{\btheta}\leq \frac{\gamma}{2\sqrt{6K_{y}}L_{11}L_{12}}$ we have that 
	\begin{equation}
		\small
		\left(1 - \frac{\gamma}{2} + \frac{4\eta_{\btheta}^{2}L_{11}^{2}K_{y}L_{12}^{2}}{\gamma\rho^{2}}\right) \leq 1 - \frac{\gamma}{3}. 
	\end{equation}
	Thus we have 
	\begin{equation}
		\small
		\begin{aligned}
			\delta(t) & \leq \left(1 - \frac{\gamma}{3}\right)^{t}4M^{2} + \left(2\gamma^{2} + \frac{2\eta_{\btheta}^{2}L_{11}^{2}}{\gamma}\right)\sigma^{2}\sum\limits_{j=0}^{t - 1}\left(1 - \frac{\gamma}{3}\right)^{j} \\
			& + \frac{4\eta_{\btheta}^{2}L_{11}^{2}}{\gamma}\sum\limits_{j=0}^{t - 1}\left(1 - \frac{\gamma}{3}\right)^{t - j - 1}\mE\left[\left\|\sum\limits_{k = 1}^{K_{y}}\hat{p}_{k}\nabla\Phi_{\rho}^{k}(\btheta(j))\right\|^{2}\right]
		\end{aligned}
	\end{equation}
	from Lemma \ref{lem:delta dynamic}. Plugging this into \eqref{eq:descent eq} in Lemma \ref{lem:descent lemma} and summing up it over $t=0, \cdots, T$, we have 
	\begin{equation}\label{eq:descent total}
		\small
		\begin{aligned}
			& \mE\left[\sum\limits_{k=1}^{K_{y}}\hat{p}_{k}\Phi_{\rho}^{k}(\btheta(T))\right] \leq \mE\left[\sum\limits_{k=1}^{K_{y}}\hat{p}_{k}\Phi_{\rho}^{k}(\btheta(0))\right] - \left(\frac{\eta_{\btheta}}{2} - \tilde{L}\eta_{\btheta}^{2}\right)\sum\limits_{t=0}^{T - 1}\mE\left[\left\|\sum\limits_{k=1}^{K_{y}}\hat{p}_{k}\nabla\Phi_{\rho}^{k}(\btheta(t))\right\|^{2}\right] \\
			& + 4M^{2}\left(\frac{\tilde{L}K_{y}L_{12}^{2}\eta_{\btheta}^{2}}{\rho} + \frac{\left(\sigma^{2} + L_{11}\right)L_{12}^{2}\eta_{\btheta}}{2\rho^{2}}\right)\sum\limits_{t=0}^{T - 1}\left(1 - \frac{\gamma}{3}\right)^{t} \\
			& + \left[\frac{T\tilde{L}\eta_{\btheta}^{2}}{2} + \left(\frac{\tilde{L}K_{y}L_{12}^{2}\eta_{\btheta}^{2}}{\rho} + \frac{\left(\sigma^{2} + L_{11}\right)L_{12}^{2}\eta_{\btheta}}{2\rho^{2}}\right)\left(2\gamma^{2} + \frac{2\eta_{\btheta}^{2}L_{11}^{2}}{\gamma}\right)\sum\limits_{t=0}^{T - 1}\sum\limits_{j=0}^{t - 1}\left(1 - \frac{\gamma}{3}\right)^{j}\right]\sigma^{2} \\
			& + \frac{4\eta_{\btheta}^{2}L_{11}^{2}}{\gamma}\left(\frac{\tilde{L}K_{y}L_{12}^{2}\eta_{\btheta}^{2}}{\rho} + \frac{\left(\sigma^{2} + L_{11}\right)L_{12}^{2}\eta_{\btheta}}{2\rho^{2}}\right)\sum\limits_{t=0}^{T- 1}\sum\limits_{j=0}^{T - 1 - t}\left(1 - \frac{\gamma}{3}\right)^{j}\mE\left[\left\|\sum\limits_{k = 1}^{K_{y}}\hat{p}_{k}\nabla\Phi_{\rho}^{k}(\btheta(t))\right\|^{2}\right]. 
		\end{aligned}
	\end{equation}
	It can be verified that for any $t$, $\sum_{j=0}^{t - 1}\left(1 - \gamma / 3\right)^{j} \leq 3/\gamma$, and plugging this into the above inequality we get 
	\begin{equation}
		\small
		\begin{aligned}
			\mE\left[\sum\limits_{k=1}^{K_{y}}\hat{p}_{k}\Phi_{\rho}^{k}(\btheta(T))\right] & \leq \mE\left[\sum\limits_{k=1}^{K_{y}}\hat{p}_{k}\Phi_{\rho}^{k}(\btheta(0))\right] - \frac{\eta_{\btheta}}{3}\sum\limits_{t=0}^{T - 1}\mE\left[\left\|\sum\limits_{k=1}^{K_{y}}\hat{p}_{k}\nabla\Phi_{\rho}^{k}(\btheta(t))\right\|^{2}\right] \\
			& + \frac{12\eta_{\btheta}M^{2}\tilde{M}_{\rho}}{\gamma} + \left[\frac{T\tilde{L}\eta_{\btheta}^{2}}{2} + \eta_{\btheta}\tilde{M}_{\rho}\left(2\gamma^{2} + \frac{2\eta_{\btheta}^{2}L_{11}}{\gamma}\right)\frac{3T}{\gamma}\right]\sigma^{2}, 
		\end{aligned}
	\end{equation}
	where $\tilde{M}_{\rho}=\frac{\tilde{L}K_{y}L_{12}^{2}}{\rho} + \frac{\left(\sigma^{2} + L_{11}\right)L_{12}^{2}}{2\rho^{2}}$ and we use the fact $\eta_{\btheta} \leq \min\left\{\frac{1}{12\tilde{L}}, \frac{\gamma}{12\sqrt{\tilde{M}_{\rho}}L_{11}}\right\}$. Finally from \eqref{eq:phi bound} and $\gamma = T^{-\frac{2}{5}}$, $\eta_{\btheta} \leq T^{-\frac{3}{5}}$ we get that
	\begin{equation}
		\small
		\begin{aligned}
			\frac{1}{T}\sum\limits_{t=0}^{T - 1}\mE & \left[\left\|\sum\limits_{k=1}^{K_{y}}\hat{p}_{k}\nabla\Phi_{\rho}^{k}(\btheta(t))\right\|^{2}\right] \leq \frac{3}{\eta_{\btheta}T}\left(M + \frac{\rho}{K_{z}e}\right) + \frac{36M^{2}\tilde{M}_{\rho}}{\gamma T} +  \frac{3\tilde{L}\eta_{\btheta}\sigma^{2}}{2} \\
			& + 18\gamma\tilde{M}_{\rho}\sigma^{2} + \frac{18\eta_{\btheta}^{2}L_{11}\tilde{M}_{\rho}\sigma^{2}}{\gamma^{2}} \\
			& \leq 3T^{-\frac{2}{5}}\left(M + \frac{\rho}{K_{z}e}\right) + 36M^{2}\tilde{M}_{\rho}T^{-\frac{3}{5}} + \frac{3\tilde{L}T^{-\frac{3}{5}}\sigma^{2}}{2} + 18T^{-\frac{2}{5}}\tilde{M}_{\rho}\sigma^{2} + 18T^{-\frac{2}{5}}L_{11}\tilde{M}_{\rho}\sigma^{2} \\
			& = \cO\left(T^{-\frac{2}{5}}\right). 
		\end{aligned}
	\end{equation}
	This completes our proof to the first conclusion. We highlight that the value of $\eta_{\btheta}$ satisfies
	\begin{equation}
		\small
		\eta_{\btheta} = \min\left\{\frac{1}{12\tilde{L}}, \frac{T^{-\frac{2}{5}}}{12\sqrt{\tilde{M}_{\rho}}L_{11}}, T^{-\frac{3}{5}}, \frac{ T^{-\frac{2}{5}}}{2\sqrt{6K_{y}}L_{11}L_{12}}\right\} = \cO\left(T^{-\frac{3}{5}}\right). 
	\end{equation}
	\par
	To see the last conclusion, similar to \eqref{eq:phi bound} we have that
	\begin{equation}
		\small
		\begin{aligned}
			\Phi_{\rho}^{k}(\btheta) - \Phi^{k}(\btheta) & = \lambda\bF^{k}(\btheta)^{\top}\left(\bu_{k}^{*}(\btheta) - \bu_{k}^{*}(\btheta, \rho)\right) - \lambda\rho\bu_{k}^{*}(\btheta, \rho)^{\top}\log{\left(K_{z}^{2}\bu_{k}^{*}(\btheta, \rho)\right)} \\
			& \leq \lambda\bF^{k}(\btheta)^{\top}\left(\bu_{k}^{*}(\btheta, \rho) - \bu_{k}^{*}(\btheta)\right) + \frac{\lambda\rho}{K_{z}e}
		\end{aligned} 
	\end{equation}
	where $\bu_{k}^{*}(\btheta) = \arg\max\{i: \Phi^{k}(\btheta)(i)\}$. Due to Theorem 1 in \citep{epasto2020optimal}, we have 
	\begin{equation}
		\small
		\bF^{k}(\btheta)^{\top}\left(\bu_{k}^{*}(\btheta) - \bu_{k}^{*}(\btheta, \rho)\right) \leq 2\rho\log{K_{z}}.
	\end{equation}
	Thus we can conclude 
	\begin{equation}
		\small
		\left|\Phi_{\rho}^{k}(\btheta) - \Phi^{k}(\btheta)\right| \leq \lambda\rho\left(\frac{1}{K_{z}e} + 2\log{K_{z}}\right)
	\end{equation}
	which implies our conclusion. 
\end{proof}
\section{More Experiments}\label{app:experiments}
In this section, we conduct more experiments on a synthetic dataset and real-world dataset to further verify the effectiveness of our proposed methods.  
\subsection{\texttt{Toy Example}}
\begin{table}[t!]
	\caption{Test accuracy (\%) of linear model on the OOD test data of \texttt{Toy example}. The OOD test data are drawn from distributions with different $\sigma_{YZ}^{\text{test}}$. The results are the mean of five independent runs.}
	\label{tbl:toy example}
	\centering
	\scalebox{1}{
		{
			\begin{tabular}{c|*{6}{c}}
				\hline
				Method / $\sigma_{YZ}^{\text{test}}$ & 0.00 & -0.20 & -0.40 & -0.60 & -0.80 & -0.99 \\
				\hline
				ERM & 88.5 & 83.0 & 68.0 & 53.5 & 33.5 & 27.2 \\
				IRM  & 96.0 & 90.5 & 91.0 & 90.5 & 91.0 & 90.5 \\
				GroupDRO & 96.5 & 94.5 & 93.0 & 91.0 & 90.5 & 88.0 \\ 
				Correlation	& 92.5 & 87.5 & 79.0 & 69.0 & 43.0 & 39.5 \\
				RCSV & \textbf{99.5} & \textbf{99.0} & \textbf{98.5} & \textbf{97.5} & \textbf{98.0} & \textbf{97.0} \\
				RCSV$_{\rm U}$ & 98.5 & 97.0 & 96.0 & 95.0 & 94.5 & 89.5 \\
				\hline
	\end{tabular}}}
\end{table}
\begin{table}[t!]
	\caption{Cosine-similarity $\langle\btheta_{2}, \bmu_{2}\rangle/(\|\btheta_{2}\|\|\bmu_{2}\|)$ of linear models trained on different methods. Model with smaller cosine-similarity theoretically exhibits better OOD generalization ability.} 
	\label{tbl:toy example cos}
	\centering
	\scalebox{1}{
		{
			\begin{tabular}{c|*{6}{c}}
				\hline
				Methods & ERM & IRM & GroupDRO & Correlation & RCSV & RCSV$_{\rm U}$ \\
				\hline
				$\frac{\langle\btheta_{2}, \bmu_{2}\rangle}{\|\btheta_{2}\|\|\bmu_{2}\|}$ & 0.95 & 0.13 & 0.18 & 0.89 & \textbf{0.03} & 0.05 
				\\
				\hline
	\end{tabular}}}
\end{table}
In this section, we apply the proposed RCSV and RCSV$_{\rm U}$ to a constructed toy example with spurious correlation. 
\paragraph{Data.} The data is constructed as the example in Appendix \ref{app:proofs in sec generalizing}. For two 5-dimensional vectors $\bmu_{1}, \bmu_{2}$, the training data $X$ follows normal distribution $\cN((Y\bmu_{1}^{\top}, Z\bmu_{2}^{\top})^{\top}, \bI_{10})$ where $\bI_{10}$ is a $10\times 10$ identity matrix. The label $Y$ and spurious attributes $Z$ take value from $\{-1, 1\}$ and are all drawn from a standard binomial distribution (i.e., $P_{Y}(Y = 1) = P_{Y}(Y = -1) = 0.5$). As in \eqref{eq:data structual}, the spurious correlation coefficient $\sigma_{YZ}^{\text{train}}$ between $Y$ and $Z$ vary on different distribution. We generate 1000 (resp. 200) training (resp. test) samples. Concretely, the 1 is fixed as 0.99 for the unique training distribution, while there are 6 constructed test distributions respectively with $\sigma_{YZ}^{\text{test}}$ in $\{0.00, -0.20, -0.40, -0.60, -0.80, -0.99\}$. As can be seen, the spurious correlations in the test sets are opposite to the one in the training set. Thus, over-fitting the spurious correlation will mislead the trained model.  
\paragraph{Setup.} We use the linear model $f_{\btheta}(\bx) = \btheta^{\top}\bx$ and its prediction on $Y$ is $\text{sign}(f_{\btheta}(\bx) )$. We compare the proposed methods RCSV and RCSV$_{\rm U}$ with the baseline methods as in the main body of this paper. The domain generalization methods can be applied with observed spurious attributes is because the data can be viewed as from two domains i.e., data drawn under conditions of $Y= Z$ and $Y\neq Z$. The loss function $\cL(\cdot,\cdot)$ is cross entropy. The hyperparameters of baseline methods follow the ones in original papers. Our methods are trained by SGD with the used hyperparemeters deferred in Appendix \ref{app:hyperparemeters}.
\paragraph{Main Results.} In Table \ref{tbl:toy example}, we report the test accuracies of trained models evaluated on OOD data to see if all the aforementioned methods can break the spurious correlation. From the results, we have the following observations.
\par
For all methods, the test accuracies are consistently improved with the decrease of the gap between $\sigma_{YZ}^{\text{train}} - \sigma_{YZ}^{\text{test}}$. This is explained as the decreased $\sigma_{YZ}^{\text{train}} - \sigma_{YZ}^{\text{test}}$ leads to smaller mismatches between training and test distributions, thus improving accuracy. 
\par
The models trained by the proposed two methods and domain generalization methods (IRM and GroupDRO) can break the spurious correlation (generalize on OOD test data), which verifies the effectiveness of our methods. On the other hand, both RCSV and RCSV$_{\rm U}$ beats the domain generalization methods which require the domain label (i.e., the spurious attributes $Z$ in the constructed dataset). Thus, the extra information required by RCSV (resp. RCSV$_{\rm U}$) is equivalent (resp. less) compared with domain generalization methods.     
\par
On the other hand, let the last 5-dimensional parameters of the linear model be $\btheta_{2}$. By $X\sim\cN((Y\bmu_{1}^{\top}, Z\bmu_{2}^{\top})^{\top}, \bI_{10})$, one can verify that the when $\btheta_{2}^{\top}\bmu_{2}\approx 0$, the output of model $\btheta^{\top}X$ does not related to $Z$ with high probability. Then the model can break the spurious correlation. 
\par
To see this, in Table \ref{tbl:toy example cos}, we present the cosine-similarity $\langle \btheta_{2},\bmu_{2}\rangle / (\|\btheta_{2}\|\|\bmu_{2}\|)$ (the cosine-similarity is used to alleviate the interference caused by scales of the two vectors) of the models trained by methods in Table \ref{tbl:toy example}. The results show that the models trained by OOD generalizable methods have smaller cosine-similarities.
 
\begin{figure*}[t!]\centering
	\subfloat[\texttt{C-MNIST-F1}]{
		\includegraphics[width=0.3\textwidth]{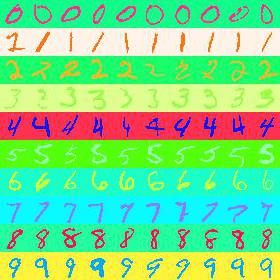}}
	\hspace{0.1in}
	\subfloat[\texttt{C-MNIST-F2}]{
		\includegraphics[width=0.3\textwidth]{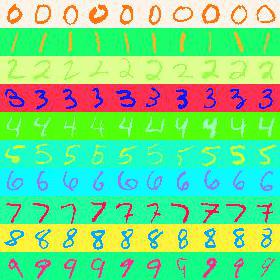}}
	\hspace{0.1in}
	\subfloat[\texttt{C-MNIST-R}]{
		\includegraphics[width=0.3\textwidth]{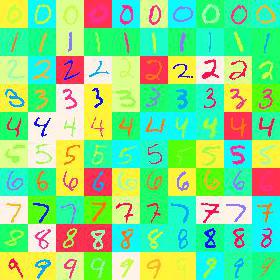}}
	\caption{Images of three \texttt{C-MNIST} datasets with different spurious correlations. The first two have fixed spurious correlation between colors and the label of digits, while the spurious correlation in the last one is random.}
	\label{fig:generated data}
\end{figure*}
\subsection{\texttt{Colored-MNIST}}
\begin{table*}[t!]
	\caption{Test accuracy (\%) of convolution neural networks trained on the different mixtures of \texttt{C-MNIST-F1} and \texttt{C-MNIST-R}. The OOD test data either \texttt{C-MNIST-R} or \texttt{C-MNIST-F2}.}
	\label{tbl:mnist_c}
	\centering
	\scalebox{1}{
		{
			\begin{tabular}{c|*{5}{c}|*{5}{c}}
				\hline
				Test set & \multicolumn{5}{c|}{\texttt{C-MNIST-R}} & \multicolumn{5}{c}{\texttt{C-MNIST-F2}} \\ 
				Method/$\alpha$ & 0.80 & 0.85 & 0.90 & 0.95 & 0.99 & 0.80 & 0.85 & 0.90 & 0.95 & 0.99 \\
				\hline
				ERM & 96.1 & 95.4 & 93.3 & 87.5 & 63.8 & 95.6 & 95.2 & 92.8 & 82.6 & 52.8 \\
				IRM  & 91.2 & 90.5 & 90.4 & 87.9 & 76.6 & 94.3 & 93.5 & 91.1 & 87.2 & 63.1 \\ 
				GroupDRO & 96.6 & 95.9 & 93.7 & 92.1 & 76.8 & 95.5 & 95.3 & 94.3 & 91.5 & 70.2 \\
				Correlation	& 82.6 & 79.6 & 79.5 & 72.6 & 38.5 & 82.6 & 78.4 & 68.2 & 35.6 & 25.2 \\
				RCSV & \textbf{98.0} & \textbf{97.6} & \textbf{96.7} & \textbf{94.3} & \textbf{81.3} & \textbf{96.9} & \textbf{96.3} & \textbf{95.1} & \textbf{90.2} & \textbf{81.3} \\ 
				RCSV$_{\rm U}$ & 97.9 & 97.5 & 96.5 & 94.1 & 77.9 & 96.8 & 95.6 & 93.6 & 86.9 & 62.3 \\
				\hline
	\end{tabular}}}
\end{table*}
In this section, we empirically verify the effectiveness of the proposed methods on a constructed real-world dataset \texttt{Colored-MNIST}. 
\paragraph{Data.} Our dataset is constructed on the $\texttt{MNIST}$ \citep{lecun1998gradient} which consists of 60,000 training data and 10,000 test data. Each data is a grey-scale hand-written digit from ten categories, i.e., 0 to 9. We construct our \texttt{Colored-MNIST} (\texttt{C-MNIST}) by inducing the spurious correlation in the training and test sets. Concretely, for each digit, we assign two colors as spurious attributes respectively for its foreground and background. The spurious correlation can be induced into such dataset by tying the relationship between the label of digits and the two colors. 
\par
We pick 20 specific colors, the first and the last 10 colors are respectively used as 10 categories of two spurious attributes, i.e., the colors of foreground and background of a digit. We consider datasets with two kinds of spurious correlations. The first is fixed spurious correlation, which means data from each \emph{specific} category of digit is assigned two \emph{specific} colors respectively for its foreground and background. The other is random spurious correlation which means that for each data, two \emph{randomly} sampled colors are respectively assigned to its foreground and background regardless of its category. We will construct two \texttt{C-MNIST} with different but fixed spurious correlations (abbrev. as \texttt{C-MNIST-F1} and \texttt{C-MNIST-F2}), and one \texttt{C-MNIST} with random spurious correlation (abbrev. as \texttt{C-MNIST-R}). 
\par
Some of the generated datasets are in Figure \ref{fig:generated data}. As can be seen, the three versions of \texttt{C-MNIST} has different spurious correlations between the label of digits and the colors of foreground and background. Besides that, the spurious correlation in \texttt{C-MNIST-F1} and \texttt{C-MNIST-F2} are fixed while \texttt{C-MNIST-R} has randomized spurious correlation. 

\paragraph{Setup.} We construct various training sets based on the original 60,000 training samples of \texttt{MNIST}. Concretely, we choose $\alpha\in\{0.8, 0.85, 0.90, 0.95, 0.99\}$, then for each $\alpha$, we construct a training set with $\lfloor 60,000 \times\alpha\rfloor$ \footnote{$\lfloor\cdot\rfloor$ is the floor of a number} samples are from \texttt{C-MNIST-F1} while the other $\lfloor 60,000 \times(1 - \alpha)\rfloor$ are constructed as \texttt{C-MNIST-R}. We use two test sets which are respectively the 10,000 test samples constructed as \texttt{C-MNIST-F2} and \texttt{C-MNIST-R}. Obviously, the data from \texttt{C-MNIST-R} in the training set alleviates the misleading signal from the training set brought by \texttt{C-MNIST-F1} due to the spurious correlation between color and digit in it, and the existence of these data meets the Assumption \ref{ass:feature space}. 
\par
Our model is a five-layer convolution neural network in \citep{arpit2019predicting}. The models are trained over the 5 aforementioned datasets with different $\alpha$ by the methods that appeared in the above section. One can verify that the training set can be viewed as a mixture of data from two domains, i.e., \texttt{C-MNIST-F1} and \texttt{C-MNIST-R}. Thus the domain generalization based methods IRM and GroupDRO can be applied here. The used loss function $\cL(\cdot, \cdot)$ is cross-entropy, and detailed hyperparameters are presented in Appendix \ref{app:hyperparemeters}.
\paragraph{Main Results.} To see if the models trained by these methods can break the induced misleading spurious correlation, we report their test accuracies on the \texttt{C-MNIST-R} and \texttt{C-MNIST-F2}. The results are summarized in Table \ref{tbl:mnist_c} with the following observations from it. 
\par
The test accuracies of all these methods increase with the decreased $\alpha$. This is a natural result since smaller $\alpha$ corresponds with more training samples from \texttt{C-MNIST-R} which alleviates the misleading signal from the data with spurious correlation in \texttt{C-MNIST-F1}. Thus models trained over the training set with smaller $\alpha$ exhibit improved generalization ability OOD data with correlation shift.  
\par
Similar to the results in Section \ref{sec:experiments}, our RCSV (resp. RCSV$_{\rm U}$) consistently improve the OOD generalization error, compared with the methods with (resp. without) observed spurious attributes. More surprisingly, RCSV$_{\rm U}$ beats the methods with observed spurious attributes methods IRM and GroupDRO for a large $\alpha$. The observations again verify the efficacy of our proposed methods. 
\par
The model trained by the most commonly used method ERM on datasets with small $\alpha$ also generalizes on OOD data. Thus a relatively large number of data without spurious correlation in the training set also breaks the spurious correlation brought by other data. 
\par
Finally, we observe that the performance of models on \texttt{C-MNIST-R} consistently better than on \texttt{C-MNIST-F2}. This is due to there exist data drawn from \texttt{C-MNIST-R} in the training set, while the data from \texttt{C-MNIST-F2} does not appear in the training set.              

\begin{table}[t!]
	\caption{Test accuracy (\%) of ResNet50 on each group of \texttt{CelebA} and \texttt{Waterbirds}. The experiments are conducted without reweighted sampling trick.}
	\label{tbl:celeba_without_reweight}
	\centering
	\scalebox{1.0}{
		{
			\begin{tabular}{l|c|*{8}{c}}
				\hline
				Dataset & Method / Group & D-F & D-M & B-F & B-M & Avg & Total & Worst & SA\\
				\hline
				\multirow{6}{*}{\texttt{CelebA}} & RCSV & 91.1	& 91.0	& 92.9	& 92.2	& \textbf{91.8}	& 91.3	& \textbf{91.0}
				 &$\surd$ \\
				& IRM  & 90.1	& 92.3	& 90.1	& 86.1	& 89.7	& 91.0	& 86.1
				 & $\surd$ \\
				& GroupDRO & 90.3	& 93.0	& 94.3	& 87.2	& 91.2	& 91.8	& 87.2
				 & $\surd$ \\
				& RCSV$_{\rm U}$ & 94.0	& 98.6	& 91.1	& 60.0	& 85.9	& \textbf{95.1}	& 60.0
				  & $\times$ \\
				& Correlation	& 94.1	& 99.3	& 82.1	& 35.7	& 77.8	& 94.0	& 35.7
				 & $\times$ \\
				& ERM & 95.4	& 99.5	& 82.8	& 41.8	& 79.9	& 94.8	& 41.8
				 & $\times$ \\
				\hline
				Dataset & Method / Group & L-L & L-W & W-L & W-W & Avg & Total & Worst & SA\\
				\hline
				\multirow{6}{*}{\texttt{Waterbirds}} & RCSV & 97.1	& 86.0	& 86.4	& 93.3 &	\textbf{90.7}	& 91.2	& \textbf{86.0}
				 & $\surd$ \\
				& IRM  & 98.6	& 90.6	& 79.4	& 90.8	& 89.9	& \textbf{92.5}	& 79.4
				 & $\surd$ \\
				& GroupDRO & 98.4	& 94.4	& 71.5	& 85.4	& 89.9	& 92.4	& 71.5
				 & $\surd$ \\
				& RCSV$_{\rm U}$ & 99.0	& 81.1	& 77.3	& 93.6	& 87.8	& 89.0	& 77.3
				 & $\times$ \\
				& Correlation	& 99.9	& 88.5	& 59.0 &	90.3	& 84.4	& 89.9	& 59.0
				 & $\times$ \\
				& ERM & 99.6	& 88.5	& 58.1	& 92.5	& 84.7	& 90.6	& 58.1
				 & $\times$ \\
				\hline
	\end{tabular}}}
\end{table}
\begin{table}[t!]
	\caption{Test accuracy (\%) of BERT on each group of \texttt{MultiNLI}. The experiments are conducted without reweighted sampling trick.}
	\label{tbl:mnli_without_reweighting}
	\centering
	\scalebox{0.85}{
		{
			\begin{tabular}{l|c|*{10}{c}}
				\hline
				Dataset & Method / Group & C-WN & C-N & E-WN & E-N & N-WN & N-N & Avg & Total & Worst & SA\\
				\hline
				\multirow{6}{*}{\texttt{MultiNLI}} & RCSV & 79.8	 & 94.4	& 83.8	& 76.5	& 79.2	& 70.6	& 80.7	& 81.6	& \textbf{70.6}
				 &$\surd$ \\
				& IRM  & 79.2	& 94.2	& 83.9	& 74.2	& 79.1	& 67.6	& 79.7	& 81.4	& 67.6
				 & $\surd$ \\
				& GroupDRO & 80.4	& 94	& 82.4	& 76.2	& 80.8	& 70.3	& 80.7	&\textbf{81.8}	& 70.3 & $\surd$ \\
				& RCSV$_{\rm U}$ & 80.1	& 94.2	& 83.6	& 80.1	& 78.2	& 67.4	& 80.6	& 81.3	& 67.4	& $\times$ \\
				& Correlation	& 73.1	& 91.2	& 76.3	& 64.5	& 77.9	& 62.4	& 74.2	& 77.8	& 62.4
				& $\times$ \\
				& ERM & 80.4	& 94.8	& 83.6	& 81.4	& 78.6	& 66.6	& \textbf{80.9}	& 81.5	& 66.6
				& $\times$ \\
				\hline
	\end{tabular}}}
\end{table}
\begin{table}[t!]
	\caption{Test accuracy (\%) of BERT on each group of \texttt{CivilComments}. The experiments are conducted without reweighted sampling trick.}
	\label{tbl:civil_without_reweight}
	\centering
	\scalebox{0.9}{
		{
			\begin{tabular}{l|c|*{8}{c}}
				\hline
				Dateset & Method / Group & N-N & N-I & T-N & T-I & Avg & Total & Worst & SA \\
				\hline
				\multirow{6}{*}{\texttt{CivilComments}} & RCSV & 93.1	& 91.8	& 72.4	& 70.6	& \textbf{82.0}	& 90.2	& \textbf{70.6}
				& $\surd$ \\
				& IRM  & 93.0	& 86.3	& 74.9	& 67.8	& 80.5	& 88.1	& 67.8
				 & $\surd$ \\
				& GroupDRO & 92.9	& 88.3	& 77.9	& 65.1	& 81.1	& 88.8	& 65.1
				& $\surd$ \\
				& RCSV$_{\rm U}$ & 97.4	& 94.3	& 69.0	& 62.9	& 80.9	& \textbf{92.7}	& 62.9
				& $\times$ \\
				& Correlation	& 97.1	& 92.3	& 66.5	& 61.7	& 79.4	& 91.6	& 61.7
				 & $\times$ \\
				& ERM & 96.6	& 92.1	& 69.4	& 56.3	& 78.6	& 91.1	& 56.3
				& $\times$ \\
				\hline
			\end{tabular}
	}}
\end{table}

\section{Ablation Study}\label{app:experiments without reweighting}
We have discussed in Section \ref{sec:experiments} that the reweighed sampling trick improves the OOD generalization. Thus, we explore the effect of such trick in this section. 
\par
We follow the settings in main part of this paper, expected for the reweighted sampling strategy is set as uniformly sampling, thus the methods ERMRS$_{\rm YZ}$ and ERMRS$_{\rm Y}$ become the ERM. The results are summarized in Tables \ref{tbl:celeba_without_reweight}, \ref{tbl:mnli_without_reweighting}, and \ref{tbl:civil_without_reweight}.    
\par
As can be seen from these tables, the OOD generalization performance of model drops for all these methods compared with the results in Section \ref{sec:experiments}, especially for \texttt{CelebA} and \texttt{Waterbirds}, see the column of ``Avg'' and ``Worst'' in each table. We speculate this is because the reweighted sampling strategy enables the data in each group are equivalently appeared during training, this operation itself can break the spurious correlation in training data. Another evidence to support the degenerated OOD generalization is the improved test accuracies on the groups with similar spurious attributes in training data, e.g., the better performances on the groups D-F, D-M of \texttt{CelebA} and L-L, W-W of \texttt{Waterbirds}.
\par
The other observation is that even without this trick, our methods improve the OOD generalization compared with other baseline methods due to their better mean and worst test accuracies.
\par
Finally, the trade-off between the robustness over spurious attributes and in-distribution test accuracies is more clearly observed in these tables. This is from the comparisons between accuracy gap of data with same spurious attributes and total accuracy, which is in-distribution test accuracy for \texttt{CelebA}, \texttt{MultiNLI}, and \texttt{CivilComments}.     
\section{Setup for Experiments}
\subsection{Implementation of Two Proposed Algorithms}\label{app:algorithms}
In this section, we present the detailed algorithm flows of the proposed RCSV and R$\widehat{\rm{CSV}}_{\rm{U}}$ in the main body of this paper. The critical part is their estimators to the $\bF^{k}(\btheta)$ defined in Section \ref{sec:regularize training with cv}. 
\begin{algorithm}[h]
	\caption{Regularize training with $\widehat{\rm{CSV}}$ (RCSV).}
	\label{alg:RCSV1}
	\textbf{Input:} Training samples $\{(\bx_{i}, y_{i})\}_{i=1}^{n}$, number of labels $K_{y}$ and spurious attributes $K_{z}$, batch size $S$, learning rate $\eta_{\btheta}$, training iterations $T$, model $f_{\btheta}(\cdot)$ parameterized by $\btheta$. Initialized $\btheta_{0}, \{\bF^{k}_{0}\}$. Positive regularization constant $\lambda$, surrogate constant $\rho$, and correction constant $\gamma$.
	
	\begin{algorithmic}[1]
		\FOR    {$t=0, \cdots ,T$}
		\STATE  {\textbf{Compute the estimator $\hat{R}_{emp}(f_{\btheta(t)}, P)$;}}
		\STATE  {$\hat{R}_{emp}(f_{\btheta(t)}, P)$ is the empirical risk over a uniformly-drawn batch (size $S$) of data.}		
		\STATE  {\textbf{Compute the estimator $\hat{\bF}^{k}(\btheta(t))$, $k\in[K_{y}]$};}
		\STATE  {Initialized $K_{z}$-dimensional vector $\hat{\cL}^{k} = \textbf{0}$, $k\in[K_{y}]$;}
		\STATE  {Reweighted sample a mini-batch of data $\{(x_{t, i}, y_{t, i}, z_{t, i})\}$ with replacement, the probability of data satisfies $y_{t,i} = k$ and $z_{t, i} = z$ is $1 / (K_{y}K_{z}n_{kz})$.}
		\STATE  {Update $\hat{\cL}^{k}(z)$ as the empirical risk over $\{(x_{t, i}, y_{t, i})\}\bigcap A_{kz}$, $k\in[K_{y}], z\in[K_{z}]$}
		\STATE  {Compute $\hat{\bF}^{k}(\btheta(t)) = K_{y}K_{z}\left(\hat{\cL}^{k}(1) - \hat{\cL}^{k}(1), \cdots, \hat{\cL}^{k}(K_{z}) - \hat{\cL}^{k}(K_{z})\right)$, $k\in [K_{y}]$}
		\STATE  {\textbf{Solve the maximization problem.}}
		\STATE  {$\bF^{k}_{t + 1} = (1 - \gamma)\bF^{k}_{t} + \gamma\hat{\bF}^{k}(\btheta(t))$;}
		\STATE  {$\bu_{k}(t + 1) = \mathrm{Softmax}(\bF^{k}_{t + 1} / \rho)$.}
		\STATE  {\textbf{Update model parameters $\btheta(t)$ via SGD.}}
		\STATE  {$\btheta(t + 1) = \btheta(t) -  \eta_{\btheta}\sum\limits_{k = 1}^{K_{y}}\hat{p}_{k}\nabla_{\btheta}(\hat{R}_{emp}(f_{\btheta(t)}, P) + \lambda\bu_{k}(t + 1)^{\top}\bF_{t + 1}^{k})$.}
		\ENDFOR
	\end{algorithmic}
\end{algorithm}
\begin{algorithm}[h]
	\caption{Regularize training with $\widehat{\rm{CSV}}_{\rm U}$ (RCSV$_{\rm U}$).}
	\label{alg:RCSV2}
	\textbf{Input:} Training samples $\{(\bx_{i}, y_{i})\}_{i=1}^{n}$, number of labels $K_{y}$ and spurious attributes $K_{z}$, batch size $S$, learning rate $\eta_{\btheta}$, training iterations $T$, model $f_{\btheta}(\cdot)$ parameterized by $\btheta$. Initialized $\btheta_{0}, \{\bF^{k}_{0}\}$. Positive regularization constant $\lambda$, surrogate constant $\rho$, and correction constant $\gamma$. 
	
	\begin{algorithmic}[1]
		\FOR    {$t=0, \cdots ,T$}
		\STATE  {\textbf{Compute the estimator $\hat{R}_{emp}(f_{\btheta(t)}, P)$;}}
		\STATE  {$\hat{R}_{emp}(f_{\btheta(t)}, P)$ is the empirical risk over a uniformly-drawn batch (size $S$) of data.}
		\STATE  {\textbf{Compute the estimator $\hat{\bF}^{k}(\btheta(t))$, $k=1,\cdots, K_{y}$;}}
		\STATE  {Initialized $|A_{k}|^{2}$-dimensional vector $\hat{\bF}^{k}(\btheta(t)) = \textbf{0}$, $k=\in[K_{y}]$;}
		\STATE  {Reweighted sample a mini-batch of data $\{(x_{t, i}, y_{t, i})\}$ with replacement, the probability of data satisfies $y_{t,i} = k$ is $1 / (K_{y}n_{k})$.}
		\STATE  {Update $\hat{\bF}^{k}(j)$ with $\cL(f_{\btheta}(x_{t, i}), y_{t, i})$ if $(x_{t, i}, y_{t, i})$ is the $j$-th data in $A_{k}$, $i\in [K_{y}], j \in [K_{z}]$.}
		\STATE  {\textbf{Solve the maximization problem.}}
		\STATE  {$\bF^{k}_{t + 1} = (1 - \gamma)\bF^{k}_{t} + \gamma\hat{\bF}^{k}(\btheta(t))$;}
		\STATE  {$\bu_{k}(t + 1) = \mathrm{Softmax}(\bF^{k}_{t + 1} / \rho)$.}
		\STATE  {\textbf{Update model parameters $\btheta(t)$ via SGD.}}
		\STATE  {$\btheta(t + 1) = \btheta(t) -  \eta_{\btheta}\sum\limits_{k = 1}^{K_{y}}\hat{p}_{k}\nabla_{\btheta}(\hat{R}_{emp}(f_{\btheta(t)}, P) + \lambda\bu_{k}(t + 1)^{\top}\bF_{t + 1}^{k})$.}
		\ENDFOR
	\end{algorithmic}
\end{algorithm}

\subsection{Dataset}\label{app:dataset}
In this section, we give more details on the datasets appeared in the main part of this paper.
\paragraph{\texttt{CelebA}.} This is a celebrity face dataset \citep{liu2015deep} with 162770 training samples and 20362 test samples. For each sample, the hair color \{Dark, Blond\} is class label, while the gender \{Female, Male\} is spurious attributes. For both training and test datasets, each of them can be divided into 4 groups, i.e., ``Dark-Female'' (D-F), ``Dark-Male'' (D-M), ``Blond-Female'' (B-F), ``Blond-Male'' (B-M). The numbers of samples in training and test dataset from the 4 groups are respectively \{71629, 9767\}, \{66874, 7535\}, \{22880, 2880\}, \{1387, 180\}. Our goal is to train a model that correctly recognizes the hair color of celebrities independent of their gender. One can verify that the most difficult group of data to be generalized on is B-M, due to its extremely small proportion in males in the training set. 
\paragraph{\texttt{Waterbirds}.} This is a synthetic real-world dataset in \citep{sagawa2019distributionally} with 4795 training samples and 6993 test samples, which is constructed based on combining photograph of bird from the Caltech-UCSD Birds-200-2011 (\texttt{CUB}) dataset \citep{wah2011caltech} with image backgrounds from the \texttt{Places} \citep{zhou2017places}. For each image, its class label is from \{Waterbird, Landbird\}, and each bird is placed on spurious attributes: background from \{Land background, Water background\}. As in \texttt{CelebA}, the datasets can be categorized into 4 groups, i.e., ``Landbird-Land background'' (L-L), ``Landbird-Water background'' (L-W), ``Waterbird-Water background'' (W-W), ``Waterbird-Land background'' (W-L). The training and test datasets are constructed with the numbers of samples in each group are respectively \{3498, 2255\} (L-L), \{184, 2255\} (L-W), \{56, 642\} (W-W), \{1057, 642\} (W-L). As can be seen, the spurious correlations in the training and test sets are quite different. In the training set, most landbirds are on the land, and most waterbirds are on the water. But in the test set, waterbirds and landbirds are uniformly assigned on the two backgrounds. Thus, we are desired to train a model that breaks the spurious correlation between bird and background. The proportion of 4 groups in the training set informs that the most difficult of them to be generalized on are L-W and W-L.  
\paragraph{\texttt{MultiNLI}.} This is a dataset for natural language inference \citep{williams2018broad} with 206175 training samples and 123712 test samples. The dataset is consists of pair of sentences, and our goal is to recognize that whether the second sentence is entailed by, neutral with, or contradicts to the first sentence. It was explored in \cite{gururangan2018annotation} that there exists spurious correlation in the dataset such as the contradiction can be related to the presence of the negation words \emph{nobody, no, never,} and \emph{nothing}. Thus we set such presence as spurious attribute and the dataset can be categorized into 6 groups, i.e., ``Contradiction-Without Negation'' (C-WN), ``Contradiction-Negation'' (C-N), ``Entailment-Without Negation'' (E-WN), ``Entailment-Negation'' (E-N), ``Neutrality-Without Negation'' (N-WN), ``Neutrality-Negation'' (N-N). Our goal is learning a model that makes prediction independent with the presence of negation. The numbers of samples in training and test dataset from the 6 groups are respectively \{57498, 34597\}, \{11158, 6655\}, \{67376, 40496\}, \{1521, 886\}, \{66630, 39930\}, \{1992, 1146\}.   
\paragraph{\texttt{CivilComents}.} This is a dataset consists of collected online comments \citep{borkan2019nuanced}. The dataset has 269038 training data and 133782 test data. Our goal is to recognize whether the comment is toxic or not. The toxicity can be spurious correlated with the annotation attributes such the presence of 8 certain demographic identities includes male, female, White, Black, LGBTQ, Muslim, Christian, and other religion. Thus we set the identity of any aforementioned words as the spurious attributes, and divided the dataset into 4 groups: ``Nontoxic-Nonidentity'' (N-N), ``Nontoxic-Identity'' (N-I), ``Toxic-Nonidentity'' (T-N), ``Toxic-Identity'' (T-I). The numbers of samples in training and test dataset from the 4 groups are respectively \{148186, 72373\}, \{90337, 46185\}, \{12731, 6063\}, \{17784, 9161\}. As can be seen, there exists a spurious correlation between the toxicity and the identity attribute in the training set due to the number of data in each group. 
\par
For all these datasets, from the number of data in each group, there exists dominated spurious correlation in \texttt{CelebA} and \texttt{Waterbirds}. But this does not happened in \texttt{MultiNLI} and \texttt{CivilComments}, especially for \texttt{MultiNLI} as the strong spurious correlation only exists in the group of ``C-WN'' v.s. ``C-N''. Thus for the \texttt{MultiNLI} and \texttt{CivilComments}, expected for the spurious feature, the model should extract other features to guarantee good performance.    
\subsection{Benchmark Algorithms}\label{app:benchmark algorithms}
Empirical Risk minimization \citep[ERM,][]{vapnik1999nature} pools together the data from all the domains and then minimizes the empirical loss to train the model. 
\par
Empirical Risk minimization with reweighted sampling \citep[ERMRS,][]{idrissi2021simple} is similar to empirical risk minimization, but it reweight the sampling probability of each sample, and the weightes on each data is pre-defined. 
\par
Invariant Risk Minimization \citep[IRM,][]{arjovsky2019invariant} learns a feature representation such that the optimal classifiers on top of the representation is the same across the domains.
\par
Group Distributionally Robust Optimization \citep[GroupDRO,][]{sagawa2019distributionally} minimizes the worst-case loss over different domains.
\par
\citep[Correlation,][]{arpit2019predicting} minimizes the intra-variance of data from the same category to break the spurious correlation. 
\subsection{Training Details}\label{app:hyperparemeters}
As clarified in Section \ref{sec:experiments}, the backbone models for image datasets (\texttt{CelebA}, \texttt{Waterbirds}) and textual datasets (\texttt{MultiNLI}, \texttt{CivilComments}) are respectively ResNet-50 \citep{he2016deep} pre-trained on \texttt{ImageNet} \citep{deng2009imagenet} and pre-trained BERT Base model\citep{devlin2019bert}. 
\par
The loss function $\cL(\cdot, \cdot)$ is cross-entropy for all of these methods. The experiments on image datasets are conducted without learning rate decay while the results on textual datasets are obtained with linearly decayed learning decay via optimizer AdamW \citep{loshchilov2018decoupled}.
\par
The hyperparameters of baseline methods follow the original one in \citep{gulrajani2020search,sagawa2019distributionally,arpit2019predicting,arjovsky2019invariant,idrissi2021simple}. The hyperparameters of the proposed RCSV and RCSV$_{\rm U}$ on \texttt{CelebA}, \texttt{Waterbirds}, \texttt{MultiNLI}, \texttt{CivilComments}, \texttt{Toy example} and \texttt{C-MNIST} respectively summarized in Table \ref{tbl:celeba_hyp}, \ref{tbl:waterbirds}, \ref{tbl:mnli_hyp}, \ref{tbl:civil_hyp}, \ref{tbl:hyper_toy}, and \ref{tbl:hyper_mnist}.
\begin{table*}[htbp]
	\begin{minipage}{0.5\linewidth}
		\caption{Hyperparameters on \texttt{CelebA}.}
		\label{tbl:celeba_hyp}
		\centering
		\begin{tabular}{c c c}
			\hline
			Algorithm & RCSV    & RCSV$_{\rm U}$ \\
			\hline
			Optimizer     &  Adam      & Adam  \\
			Learning Rate &  1e-5      & 1e-5 \\
			Batch Size    &  256       & 256   \\
			Weight Decay  &  1e-4      & 1e-4     \\
			Epoch         &  50        & 50   \\
			$\lambda$     &  5         & 1   \\
			$\gamma$      &  0.9       & 0.9  \\
			$\rho$        &  1e-4      & 1e-4  \\ 
			\hline
		\end{tabular}
	\end{minipage}
	\begin{minipage}{0.5\linewidth}
		\caption{Hyperparameters on \texttt{Waterbirds}.}
		\label{tbl:waterbirds}
		\centering
		\begin{tabular}{c c c}
			\hline
			Algorithm & RCSV    & RCSV$_{\rm U}$ \\
			\hline
			Optimizer     &  Adam      & Adam  \\
			Learning Rate &  1e-5      & 1e-5 \\
			Batch Size    &  64        & 64   \\
			Weight Decay  &  1e-4      & 1e-4     \\
			Epoch         &  300       & 300   \\
			$\lambda$     &  0.1       & 0.1   \\
			$\gamma$      &  0.9       & 0.9  \\
			$\rho$        &  1e-4      & 1e-4  \\ 
			\hline
		\end{tabular}
	\end{minipage}
\end{table*}

\begin{table*}[htbp]
			\begin{minipage}{0.5\linewidth}
		\caption{Hyperparameters on \texttt{MultiNLI}.}
		\label{tbl:mnli_hyp}
		\centering
		\begin{tabular}{c c c}
			\hline
			Algorithm & RCSV    & RCSV$_{\rm U}$ \\
			\hline
			Optimizer     &  AdamW      & AdamW  \\
			Learning Rate &  1e-5      & 1e-5 \\
			Batch Size    &  32        & 32   \\
			Weight Decay  &  0         & 0     \\
			Epoch         &  3         & 3     \\
			Drop Out      &  0.1       & 0.1   \\
			$\lambda$     &  0.1       & 0.1   \\
			$\gamma$      &  0.9       & 0.9  \\
			$\rho$        &  1e-4      & 1e-4  \\ 
			\hline
		\end{tabular}
	\end{minipage}
	\begin{minipage}{0.5\linewidth}
		\caption{Hyperparameters on \texttt{CivilComments}.}
		\label{tbl:civil_hyp}
		\centering
		\begin{tabular}{c c c}
			\hline
			Algorithm & RCSV    & RCSV$_{\rm U}$ \\
			\hline
			Optimizer     &  Adam      & Adam  \\
			Learning Rate &  1e-5      & 1e-5 \\
			Batch Size    &  16        & 16   \\
			Weight Decay  &  0         & 0     \\
			Epoch         &  3         & 3     \\
			Drop Out      &  0.1       & 0.1   \\
			$\lambda$     &  0.1       & 0.1   \\
			$\gamma$      &  0.9       & 0.9  \\
			$\rho$        &  1e-4      & 1e-4  \\ 
			\hline
		\end{tabular}
	\end{minipage}
\end{table*}

\begin{table*}[htbp]
	\begin{minipage}{0.5\linewidth}
		\caption{Hyperparameters on \texttt{Toy example}.}
		\label{tbl:hyper_toy}
		\centering
		\begin{tabular}{c c c}
			\hline
			Algorithm & RCSV    & RCSV$_{\rm U}$ \\
			\hline
			Optimizer     &  SGD       & SGD  \\
			Learning Rate &  0.01      & 0.01 \\
			Momentum      &  0.9       & 0.9  \\
			Batch Size    &  32        & 32   \\
			Weight Decay  &  0         & 0     \\
			Epoch         &  100       & 100   \\
			$\lambda$     &  1.0       & 5   \\
			$\gamma$      &  0.9       & 0.9  \\
			$\rho$        &  0.01      & 0.01  \\ 
			\hline
		\end{tabular}
	\end{minipage}
	\begin{minipage}{0.5\linewidth}
		\caption{Hyperparameters on \texttt{C-MNIST}.}
		\label{tbl:hyper_mnist}
		\centering
		\begin{tabular}{c c c}
			\hline
			Algorithm & RCSV    & RCSV$_{\rm U}$ \\
			\hline
			Optimizer     &  Adam      & Adam  \\
			Learning Rate &  0.001      & 0.001 \\
			Momentum      &  /\         & /\  \\
			Batch Size    &  128        & 128   \\
			Weight Decay  &  1e-4       & 1e-4     \\
			Epoch         &  40         & 40   \\
			$\lambda$     &  1.0        & 0.05   \\
			$\gamma$      &  0.9        & 0.9  \\
			$\rho$        &  0.01       & 0.01  \\ 
			\hline
		\end{tabular}
	\end{minipage}
\end{table*}

	\clearpage
\end{document}